\documentclass{article}
\usepackage{iclr2027_conference,times}
\iclrfinalcopy   

\usepackage{amsmath,amsfonts,bm}

\def\eqref#1{equation~\ref{#1}}

\def\1{\bm{1}}

\DeclareMathAlphabet{\mathsfit}{\encodingdefault}{\sfdefault}{m}{sl}
\SetMathAlphabet{\mathsfit}{bold}{\encodingdefault}{\sfdefault}{bx}{n}

\usepackage{hyperref}
\usepackage{url}
\usepackage{graphicx}
\usepackage{booktabs}
\usepackage{amsmath,amssymb,amsthm}
\usepackage{multirow}
\usepackage{xcolor}
\usepackage{framed}
\usepackage{tabularx,array}
\newcolumntype{Y}{>{\centering\arraybackslash}X}
\newcolumntype{L}{>{\raggedright\arraybackslash}X}
\newcolumntype{R}[1]{>{\raggedright\arraybackslash}p{#1}}

\usepackage{needspace}
\usepackage{placeins}
\graphicspath{{./}}
\newtheorem{theorem}{Theorem}
\newtheorem{proposition}{Proposition}
\newtheorem{lemma}{Lemma}
\newtheorem{corollary}{Corollary}
\newtheorem{definition}{Definition}
\newtheorem{assumption}{Assumption}
\newcommand{\Gam}{\Gamma}
\newcommand{\Ob}{\bar O}
\newcommand{\Hc}{\hat H}

\title{Minimal Recurrent Behavioral Memory\\ for Imitation under Partial Observability}

\author{Xianyao Li, Fang Xu, Rui Min, Ruitong Tian, Jing Du}

\begin{document}
\maketitle

\begin{abstract}
What is the least recurrent memory needed to reproduce a specified expert under partial observability?
The instantaneous requirement is the conditional entropy of the expert's behavioral quotient, but recurrence must also preserve distinctions that future observations will not restore before use.
We characterize this minimal recurrent behavioral memory by a compatibility relation: under transitivity its classes attain the exact minimum, while the general case is an entropy minimization over closed compatible state assignments, with exact certificates on finite instances.
A sole-carrier measurement protocol separates behavioral sufficiency, excess code rate, and information carried by observations or other memory paths; experimental bit requirements refer to the induced symbolic behavioral model under the stated occupancy.
Across manipulation tasks, learned code rates remain near zero- and two-bit requirements as hidden modes grow to $512$, and anticipatory memory follows a $2\to1\to0$ requirement despite zero instantaneous demand during waiting.
Learning this representation remains difficult: event-agnostic future-behavior supervision yields $36/40$ sufficient seeds with one frozen configuration and improves the longest-horizon pixel setting from $0/8$ to $6/8$ sufficient held-out seeds (closed-loop success from $0.08$ to $0.57$). On unmodified community benchmarks, the protocol certifies delay-independent requirements, which sufficient codes match at mid-delay.
The supervision aids commitment but can induce predictive surplus; annealing it lets imitation and rate training reduce that surplus, separating the information-theoretic target from the ability to learn it.
\end{abstract}

\section{Introduction}
Imitating a specified expert under partial observability raises three questions: \emph{Which distinctions determine its behavior now? Which must persist because future observations will not restore them before use? Can a compact recurrent learner acquire that representation?} These concern current behavior, recurrent memory, and learning, respectively.

Consider a robot that observes a grasp side and a placement slot, then waits. Its waiting action depends on neither cue, but it must retain both for later decisions. After grasping, only the slot remains necessary. If the slot will be displayed again before placement, it need not be carried across the wait. \textbf{Future observations act as side information:} the policy receives them as it acts, so memory need only bridge intervals without a new reveal.

The behavioral quotient $G_{E,t}$ groups states with the same observation and current expert action distribution. The recurrent target $\Gam_t$ additionally distinguishes histories that require different behavior after a common reachable continuation (Figure~\ref{fig:concept}). It retains distinctions only while future observations cannot restore them before use. Under transitive compatibility, we prove that its conditional entropy is the exact minimum among zero-distortion recurrent realizations. Without transitivity, closed compatible state assignments replace the class variable; small instances admit exact search, and matching bounds certify the minimum on our finite corridor instances.

Existing representations answer different questions. A recurrent architecture specifies where memory resides, while a control-state target supports reward or dynamics prediction \citep{subramanian2022ais}. Predictive states preserve future-process distributions \citep{shalizi2001computational,littman2002predictive}; system identification seeks hidden parameters \citep{kumar2021rma}. These targets can distinguish histories that the specified expert treats identically. Even future-action prediction can retain a re-displayed cue if its decoder receives no future observations. Our target instead preserves the expert's responses \emph{as those observations arrive}.

\begin{figure}[t]
\centering\includegraphics[width=\linewidth]{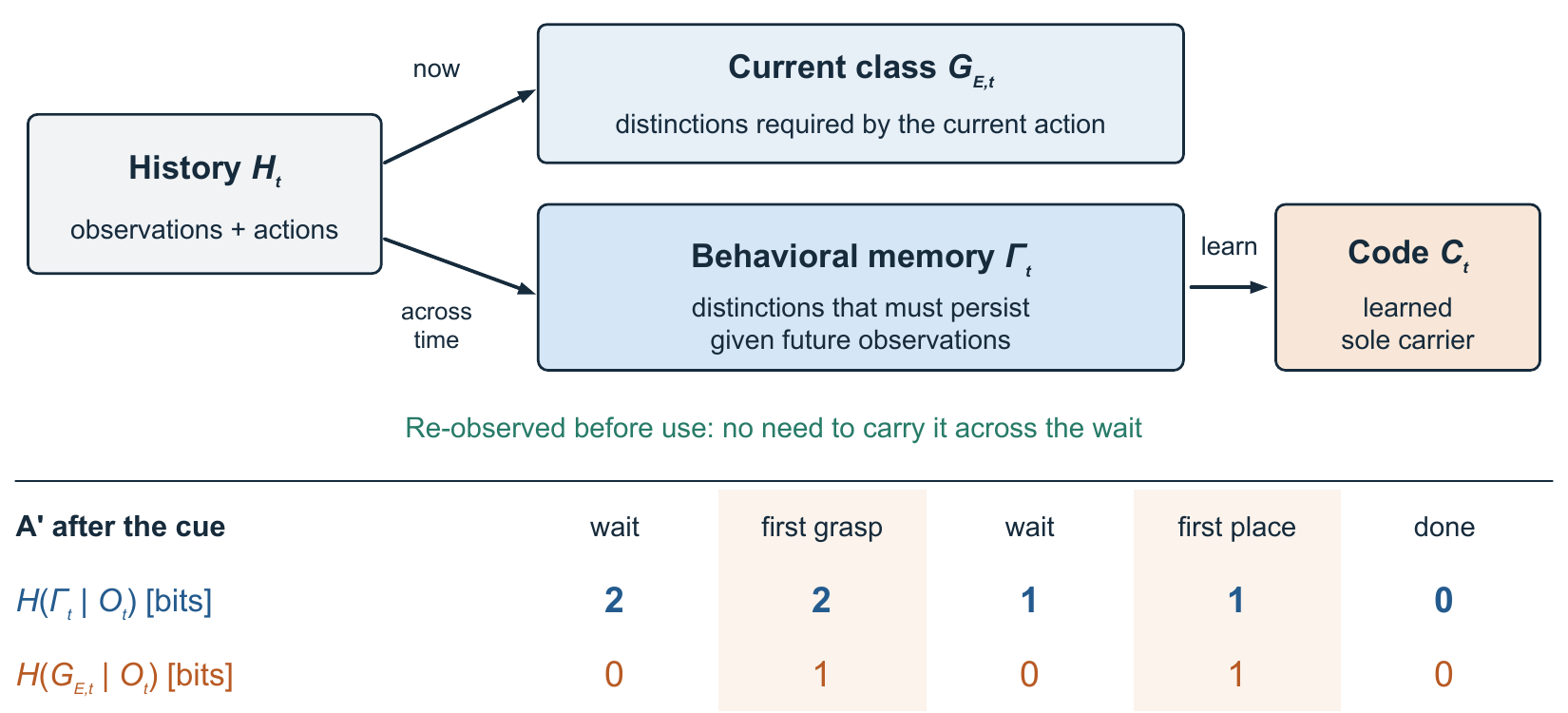}
\caption{\textbf{Current behavior, recurrent memory, and its learned realization.} In the transitive regime, history induces a current behavioral class $G_{E,t}$ and a recurrent target $\Gam_t$ that accounts for future observations as side information; $C_t$ is trained to realize this target. Bottom: A$'$ has three distinct rate levels, $2\to1\to0$. The grasp and place columns show the rates at the decision, \emph{before} the corresponding class is consumed. The green note refers to the re-reveal toy (\S\ref{sec:f3}).}
\label{fig:concept}
\end{figure}

The primary contribution is this representation target and its characterization. Measuring it requires a sole temporal carrier and explicit observation and occupancy conventions \citep[cf.][]{dann2016memory}. We check sufficiency before minimality and probe continuous bypasses, body memory, and current observations. Exact experimental bits concern the induced symbolic model, not hardware storage or exact continuous control.

Learning experiments test attainability: code rates remain near $0$ and $2$ bits as hidden-mode counts grow to $512$, while system-identification rates grow with larger codebooks. Plain training becomes less reliable with temporal distance or load \citep[cf.][]{ni2023transformers}. Event-agnostic future-behavior supervision improves acquisition from states and pixels in the tested settings, but can retain surplus information and still fails at longer delays. The architecture uses standard components \citep{tishby2000information,lee2024vqbet}; supervision is a learning surrogate, separate from the minimal-memory characterization.

\section{Minimal recurrent behavioral memory}\label{sec:theory}
\subsection{The instantaneous problem}\label{sec:quotient}
Consider a controlled process with state $S_t$, observation $O_t=\mathcal O(S_t)$, history $H_t=(O_{1:t},A_{1:t-1})$, and expert $\pi_E(\cdot\mid S_t)$. All distributions below use the expert occupancy $d_{E,t}$ on a finite horizon. We use countable reachable histories, as in the enumerated models; Appendix~\ref{app:proofs} states the conventions.
\begin{definition}[Behavioral quotient]\label{def:quotient}
States $s,s'$ are equivalent when $\mathcal O(s)=\mathcal O(s')$ and $\pi_E(\cdot\mid s)=\pi_E(\cdot\mid s')$. Their equivalence class is $G_{E,t}=g_E(S_t)$.
\end{definition}
We assume (A1) finitely many classes per observation fiber and $H(G_{E,t}\mid O_t)<\infty$; (A2) $H(G_{E,t}\mid O_t,H_t)=0$, so history determines the expert's current behavior; and (A3) a nonnegative distortion $\delta$ that vanishes exactly when action distributions agree. Assumption (A2) excludes expert-private information unavailable in history \citep{yu2026capability}.

Let $R_E(D)$ minimize $I(C;H_t\mid O_t)$ over history encoders and decoders with expected distortion at most $D$. Compressing history reduces to compressing the behavioral quotient:
\begin{equation}\label{eq:instantaneous}
R_E(D)=R_{G\mid O}(D),\qquad R_E(0)=H(G_{E,t}\mid O_t).
\end{equation}
Every zero-distortion code determines $G_{E,t}$ together with $O_t$, and encoding $G_{E,t}$ attains the bound (Appendix~\ref{app:instantaneous}). If $M$ uniform hidden modes form $R$ equal behavioral classes, $\log_2(M/R)$ bits are unnecessary for the current action. This elementary reduction answers the instantaneous question. It does not ensure a representation that can be updated recursively.

\subsection{From current behavior to an exact recurrent target}
\textbf{Why recurrence changes the requirement.} In A$'$, an early cue reveals independent binary classes $(\beta_1,\beta_2)$ for grasp and placement. Immediately afterwards $H(G_{E,t}\mid O_t)=0$, since every expert waits in the same way. Nevertheless, a recurrent policy must retain all four combinations until the first decision: $2$ bits. After the grasp, the required memory falls to $1$ bit; after placement it falls to $0$. This is the distinction that the recurrent definition must capture.

\begin{definition}[Recurrent realization]\label{def:realization}
A realization has a discrete state $C_t=F_t(C_{t-1},O_t,A_{t-1})$ with deterministic updates and fixed $C_0$, and a policy $\hat\pi(\cdot\mid O_t,C_t)$. The code is its sole internal temporal carrier. It has zero distortion if $\hat\pi(\cdot\mid O_t,C_t)=P_E(A_t\mid H_t)$ almost surely at every step.
\end{definition}
Define $R_t^{\rm mem}(0)=\inf H(C_t\mid O_t)$ over zero-distortion realizations on the \emph{full horizon}, allowing countable state spaces. The experiments use finite codebooks. Write $U_t(h)=\operatorname{supp}P_E(A_t,O_{t+1}\mid H_t=h)$ and $hu$ for a history extended by $u=(a_t,o_{t+1})$.

\begin{definition}[Behavioral memory compatibility]\label{def:gamma}
At $T$, $h\sim_T h'$ iff $(O_T,G_{E,T})(h)=(O_T,G_{E,T})(h')$. Recursively, $h\sim_t h'$ iff $(O_t,G_{E,t})(h)=(O_t,G_{E,t})(h')$ and
\[
hu\sim_{t+1}h'u\quad\text{for every }u\in U_t(h)\cap U_t(h').
\]
\end{definition}
Compatible histories demand identical current behavior and remain compatible after every continuation reachable from both. Comparing only common continuations credits future observations with the distinctions they will supply.

\begin{theorem}[Minimal recurrent behavioral memory]\label{thm:exact-trans}
Under (A1)--(A3), if $\sim_t$ is transitive at every step, its classes $\Gam_t=[H_t]_{\sim_t}$ admit a deterministic update $\Gam_{t+1}=\Phi_t(\Gam_t,O_{t+1},A_t)$ and
\begin{equation}\label{eq:memory-minimum}
R_t^{\rm mem}(0)=H(\Gam_t\mid O_t)\ge H(G_{E,t}\mid O_t).
\end{equation}
Every zero-distortion realization determines $\Gam_t$ from $(O_t,C_t)$; the realization $C_t=\Gam_t$ attains the minimum simultaneously at all steps.
\end{theorem}
\emph{Proof idea.} Histories sharing $(O_t,C_t)$ receive the same action distribution and, after a common continuation, the same updated code. Backward induction makes each such cell compatible. Under transitivity it lies in one $\Gam_t$ class, giving the lower bound. Conversely, compatibility makes the successor class independent of the representative history, so the classes themselves define a sufficient realization. Appendix~\ref{app:proofs} gives the proof.

The \emph{anticipatory memory} is $\Delta_t=H(\Gam_t\mid O_t)-H(G_{E,t}\mid O_t)\ge0$. It quantifies information that must persist despite being unnecessary for the current action. A sufficient finite codebook needs $K\ge\max_{t,o}|\Gam_t|_o$: in A$'$ the joint memory has four states although each individual decision is binary. A distinction can leave memory after its last use, or earlier if a future observation restores it before its next use.

Transitivity is directly checkable. A convenient sufficient condition, continuation-support homogeneity (A4), requires equal $U_t(h)$ whenever $(O_t,G_{E,t})(h)$ agrees. It holds in A$'$ and cue--corridor tasks; the solver certifies it there. It is not necessary: Task~A is transitive although a remembered mass changes the support of the next transport observation (Appendix~\ref{app:support}).

\subsection{Why the general case is harder}
Future observations can distinguish histories that currently share a code. If a class will be re-observed before use, histories differing in that class can have disjoint continuation supports and be compatible without storing it. But compatibility need not be transitive: $h_1\sim h_2$ and $h_2\sim h_3$ can hold on disjoint futures while $h_1,h_3$ conflict on a common future. Taking a transitive closure would then merge histories that must remain distinct.

The correct general object is a \emph{closed compatible state assignment}: a partition of histories into pairwise-compatible cells whose successors under each common continuation remain in one cell. We prove that these assignments are exactly the joint $(O_t,C_t)$ partitions of zero-distortion recurrent realizations, with
\[
R_t^{\rm mem}(0)=\inf_{\text{closed compatible assignments }(\mathcal P_s)} H(P_t\mid O_t),
\]
where $P_t$ is the cell index (Proposition~\ref{prop:cover}). This is a recurrent zero-error problem with decoder side information \citep{witsenhausen1976zero,alon1996source,koulgi2003zeroerror}, related to incompletely specified machines \citep{paull1959minimizing}; the objective here is occupancy-weighted conditional entropy.

A stronger relation $\sim_t^s$ also requires equal continuation supports and recursively applies that requirement. It always defines a realizable equivalence $\Gam_t^s$.
\begin{theorem}[General bounds]\label{thm:sandwich}
Under (A1)--(A3), every joint $(O_t,C_t)$ cell of a zero-distortion realization is compatible, and
\[
H(G_{E,t}\mid O_t)\le R_t^{\rm mem}(0)\le H(\Gam_t^s\mid O_t).
\]
Both inequalities can be strict.
\end{theorem}
Three equiprobable histories can have minimum $h_2(1/3)$ strictly inside $[0,\log_2 3]$. On Task~A, $\Gam_t^s$ charges a mass-half bit because observation supports differ, while $\Gam_t$ requires zero: the distinction changes no expert action.

\textbf{Computation and certificates.} Backward refinement computes compatibility, checks transitivity, and returns exact class entropies when the condition holds. A partition dynamic program solves small general instances. For the larger non-transitive corridor instances $W\in\{3,5,\dots,13\}$, an incompatibility-based entropy lower bound matches a certified closed compatible assignment at every step, establishing exact finite-instance rates of $2$ and $1$ bits in the two halls (Appendix~\ref{app:cover}). This is a finite-instance certificate, not a general polynomial-time algorithm. For nonzero distortion, the instantaneous rate--distortion function remains a lower bound; we do not characterize the full recurrent frontier (Appendix~\ref{app:rdprop}).

\section{When does a bit count as memory?}\label{sec:method}\label{sec:box}
Our measurement principle separates three questions.
\begin{framed}\small
\textbf{Sufficiency:} does $(C_t,\Ob_t)$ determine the required behavioral memory?\\
\textbf{Minimality:} conditional on sufficiency, how much code rate exceeds the requirement?\\
\textbf{Attribution:} is the information in $C_t$, rather than in another temporal carrier or in observations under a different occupancy?
\[\boxed{\text{sufficiency first, minimality second}}\]
\end{framed}

\textbf{What the reported bits mean.} We distinguish the theoretical \emph{behavioral memory rate} $H(\Gam_t\mid\Ob_t)$, the empirical \emph{learned code rate} $\Hc(C_t\mid\Ob_t)$, \emph{codebook capacity} $\log_2K$, and \emph{storage footprint} in bytes. The first is exact for the induced finite symbolic behavioral model under expert occupancy. The observation convention $\Ob_t=f_t(O_t)$ credits behaviorally decodable classes to the current observation. Except for a disclosed, behaviorally irrelevant Task~A sag symbol, it is a frozen coarsening, so $H(C_t\mid\Ob_t)\ge H(C_t\mid O_t)$ (Appendix~\ref{app:solver}). These rates quantify internal information burden under the convention, not exact continuous-control requirements.

\textbf{Sufficiency and excess rate.} For transitive models, define $S_\Gam=1-\Hc(\Gam_t\mid C_t,\Ob_t)/\Hc(\Gam_t\mid\Ob_t)$ where the denominator is positive, and $S_G$ analogously. A seed passes our empirical gate when $S_\Gam>0.9$ at every positive-requirement step of both waiting gaps and $S_G>0.9$ at the first grasp and first placement decisions; Task~A uses its placement gate. We call these seeds \emph{sufficient} as shorthand for passing this diagnostic, not for exact zero distortion. Non-transitive corridor comparisons use the certified rate bound rather than assume a canonical $\Gam_t$. Under exact sufficiency,
\[
H(C_t\mid\Ob_t)-H(\Gam_t\mid\Ob_t)=H(C_t\mid\Gam_t,\Ob_t).
\]
Thus matching the lower bound is evidence of absent redundancy only after sufficiency is established. Threshold sensitivity and the full-versus-relaxed gate comparison are in Appendix~\ref{app:gate}.

\textbf{Attribution controls.} A quantized readout of a persistent GRU state reports only $1.06$ bits on A$'$ despite a two-bit requirement: the unmeasured state carries history. A policy can also encode a class in its own pose; body-memory probes detect this. We therefore report expert-occupancy diagnostics separately from closed-loop success under policy occupancy, with raw-observation leak probes and matched-side-information controls (Appendix~\ref{app:sideinfo}). Low code rate or high task success alone does not establish minimal internal memory.

\subsection{A measurable recurrent realization}\label{sec:realization}
We implement a discrete sole-carrier recurrent policy (our DIACRITIC realization), $C_t=Q(F(C_{t-1},O_t,A_{t-1}))$, using a residual proposal, hard nearest-code quantization, and straight-through gradients. The action head receives $(O_t,C_t)$; no continuous hidden state persists. With a conditional prior $r_\eta$, the base objective is
\begin{equation}\label{eq:base-objective}
\mathcal L_{\rm base}=\tfrac1T\sum_t\!\left[\ell\big(\hat\pi(A_t^E\mid O_t,C_t)\big)+\beta\big(-\log r_\eta(C_t\mid O_t)\big)\right]+\mathcal L_{\rm VQ}.
\end{equation}
The expected rate penalty equals $H(C_t\mid O_t)+\mathbb E_{O_t}\mathrm{KL}(p(\cdot\mid O_t)\Vert r_\eta(\cdot\mid O_t))$. It is a variational upper bound; reported rates are plug-in entropies of hard codes. These components are standard \citep{tishby2000information,peng2019vdb,agustsson2017soft,lee2024vqbet}. Future-behavior supervision, defined in \S\ref{sec:f4}, changes the training signal while leaving the evaluated carrier unchanged. Appendix~\ref{app:design} specifies the realization and auxiliary variants.

\section{Does learned memory match the behavioral requirement?}\label{sec:findings}
\subsection{Evaluation protocol}\label{sec:protocol}
The main tasks are Task~A (physical hidden mass, zero required memory during grasp), readout-2 (nonzero memory at growing hidden-mode count), and A$'$ (two delayed decisions). The latter two use a probe channel and kinematic attachment; pixel A$'$ retains the probe and phase inputs. We certify the induced symbolic models and evaluate the same code diagnostics across tasks. Unless stated otherwise, a cell has eight seeds and each policy has $128$ closed-loop rollouts. Rate minimality is assessed on sufficient seeds; counts and all-seed control success are reported alongside it.

Two protocols answer different questions. For the plain learner's \emph{distortion-constrained} comparison, $\beta$ is selected by closed-loop fidelity before inspecting rates (A$'$: $\beta\le10^{-3}$). An \emph{attainability} control uses a continuous scaffold annealed to zero before evaluation ($N\approx4$k, $\beta=2\!\times\!10^{-3}$); it demonstrates a reachable boundary, not a point on the plain learner's frontier. The supervised learning experiments use a separately frozen configuration. Full settings, fresh-recording checks, and estimator diagnostics are in Appendices~\ref{app:solver}--\ref{app:replication}.

\textbf{Auditing observation side information.}\label{sec:audit} Injecting a placement-class leak into A$'$ raises success from $0.59$ to $0.96$ and slot accuracy from $0.68$ to $0.98$, while action error falls. A pre-specified low-rate rule flags only $1/96$ runs: codes can retain other content. An exploratory rule pairing correct placement with insufficient second-gap memory flags $11/15$ and $8/14$ policies at $2$ and $4$\,mm, versus $0/166$ correct clean unsupervised runs. This auxiliary signal is nonmonotone ($4/16$ at $8$\,mm). Direct observation probes test whether $O_t$ reveals class information uncredited to $\Ob_t$ (Appendix~\ref{app:audit}).

\subsection{World complexity: zero and nonzero requirements}\label{sec:f1}

\begin{figure}[t]
\centering\includegraphics[width=\linewidth]{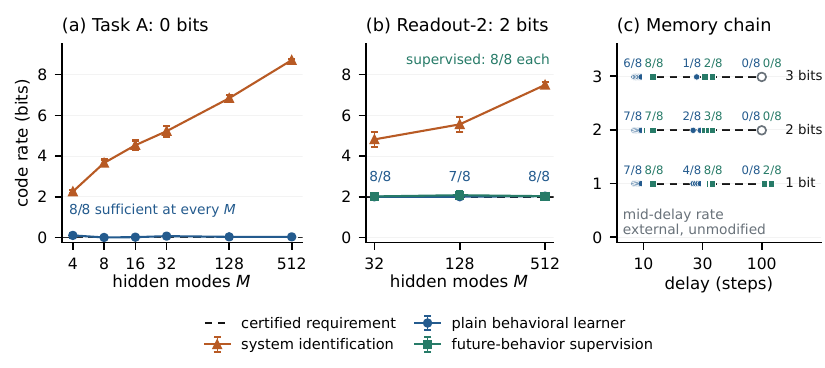}
\caption{\textbf{Learned code rate at fixed behavioral requirement.} (a,~b) World complexity: Task~A has a zero-bit grasp requirement and readout-2 a two-bit first-gap requirement as $M$ grows. Points are seed means with standard deviations over sufficient seeds (counts shown); system-identification rates use all eight seeds and the largest tested codebook. (c) \emph{Mid-delay rate} on the unmodified bsuite memory chain: every sufficient seed matches the requirement at this measurement point across the tested delays; labels count sufficient seeds of eight per learner, and hollow markers indicate none. Dashed lines are certified requirements, not fits.}
\label{fig:taskA}
\end{figure}

\textbf{Task A.} An early probe identifies one of $M=4$--$512$ hidden masses. The mass changes the arm's dynamics and later observation supports but never the expert's action choice; an independent slot is revealed shortly before placement. Compatibility is transitive, so the exact grasp-phase requirement is zero even though the strong congruence charges one bit. The plain $K=16$ carrier is sufficient on $8/8$ seeds at every $M$, carries $0.00$--$0.10$ bit during grasp with at most $0.04$ bit about mass, and attains $0.91$--$0.99$ closed-loop success. At $M=512$ its grasp rate is $0.03$ bit.

\textbf{Readout-2.} A binary display reveals $\theta\in[M]$; the expert grasps according to its quartile and places according to its half. The first-gap requirement is exactly two bits at $M=32,128,512$, while full-history information during transport grows from $6$ to $10$ bits. Plain training is sufficient on $8/8,7/8,8/8$ seeds; task-informed future-behavior supervision gives $8/8$ throughout. Learned first-gap rates remain near $2$ bits, with closed-loop success $0.85$--$0.96$. Verdicts reproduce on independent recordings. The separation therefore also holds with a nonzero behavioral memory requirement.

\textbf{What the identification comparison establishes.} A system-identification head forces the hidden parameter through the same regularized carrier. With codebook capacity increased to $K=128$--$1024$, its learned rate grows from $2.25$ to $8.71$ bits on Task~A and from $4.82$ to $7.49$ on readout-2. These are learned rates, not claims of perfect parameter recovery. A separate identification carrier restores Task~A success to $0.90$--$0.97$ while carrying an additional $4.1$--$6.0$ bits. This control separates the information burden of identification from interference when both objectives share a carrier (Appendix~\ref{app:controls}).

\textbf{External validation on community benchmarks.} On the bsuite memory chain and Passive T-maze, used unmodified \citep{cherepanov2025mikasa,ni2023transformers}, the solver certifies the induced symbolic models. The memory-chain requirement is zero during the two cue-visible steps, $b$ bits from the first cue-free step until the query, and one bit once the query index appears, independently of delay. Of $192$ runs, $87$ pass the gate, including closed-loop success $\ge0.99$, and match the certified rate \emph{at mid-delay} (Figure~\ref{fig:taskA}c).

This agreement is constrained by the task structure: at mid-delay the code is a deterministic function of the context, so exact sufficiency forces its rate to equal the context entropy. These tasks validate protocol transfer and acquisition of the required information; Task~A and readout test compression in the presence of additional world information. Reliability decreases with delay: none of $32$ runs with two or three bits at delay $100$ passes (Appendix~\ref{app:cert}).

\subsection{Anticipation, use, and re-observation}\label{sec:f2}\label{sec:f3}

\begin{figure}[t]
\centering\includegraphics[width=\linewidth]{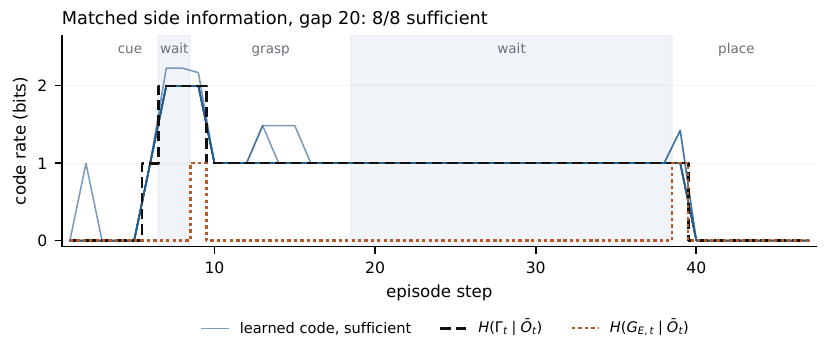}
\caption{\textbf{A recurrent realization tracks the behavioral requirement.} A$'$ at gap 20, using the matched-side-information hierarchical policy with task-informed supervision ($\beta=10^{-3}$, eight seeds). Thin blue curves show every seed's learned code rate; black and orange curves show the certified recurrent and instantaneous requirements. Both waiting gaps require memory despite zero instantaneous demand. The theory and policy share the same symbolic side information.}\label{fig:memory-profile}
\end{figure}

A$'$ isolates anticipation: its required memory follows $2\to1\to0$, while $H(G_{E,t}\mid\Ob_t)=0$ in both gaps. Sufficient plain seeds carry $2.00$ bits in the first gap. The informative result is their negligible surplus, not the lower bound already implied by exact sufficiency. A matched-side-information hierarchical control restricts the transition, prior, and behavioral head to $(C_t,\Ob_t)$; a memory-free controller executes the chosen action using $O_t$. With task-informed supervision, all $32$ runs across two horizons and two rate weights pass, carrying $2.03$--$2.09$ bits before grasp and $1.00$ afterwards, with closed-loop success $0.96$--$1.00$ (Appendix~\ref{app:sideinfo}).

\textbf{When the requirement falls.} Information expires after its last behavioral use or when future observations will restore it before use. The re-reveal toy exhibits the latter: its exact requirement is $1$ bit before the first decision and $0$ afterwards, while an open-loop future-action target retains $2.06$ and $1.29$ bits. A fully physical weighing task also has a certified $2\to0$ requirement because the next lift re-reveals mass; its offline attainability and closed-loop failure are documented in Appendix~\ref{app:weigh}.

\textbf{Reaching the lower boundary.} Structural expiration does not guarantee learned minimality. On A$'$, among seeds sufficient under policy occupancy, plain training at $\beta=10^{-3}$ achieves $0.97$ success with post-use rate $1.38$; stronger pressure reduces it to $1.26$ while success falls to $0.70$. The scaffold control instead reaches $0.99$--$1.05$ bit on five of six seeds sufficient under policy occupancy. In the signpost corridor (Appendix~\ref{app:grid}), sufficient codes in the $M=16$ profile average $1.2$--$1.8$ bits against the one-bit requirement. These observations establish attainability in specific settings and optimization slack elsewhere; the weight sweep is not an exact rate--distortion frontier (Appendix~\ref{app:rd}).

\section{Can a compact learner find the minimal memory?}\label{sec:f4}
We compare learned code rates with certified requirements, then ask how reliably training reaches a sufficient representation. The theorem specifies the target, not a learning guarantee.

\subsection{A temporal learning difficulty}
At a fixed certified requirement, code sufficiency is negatively associated with delay across $320$ sole-carrier runs on A$'$ ($-0.146$/step, CI $[-0.18,-0.12]$). This pooled association is descriptive because configurations differ. Controlled $16$-seed sweeps also give negative within-learner slopes; larger memory load reduces reliability (Appendices~\ref{app:stats}, \ref{app:compute}). Full-history attention recovers the required decisions in each tested setting under teacher forcing and reaches $0.88$--$1.00$ closed-loop success, showing attainability with unrestricted history. Its persistent history excludes its code readout from the memory-rate comparison.

\begin{samepage}
\subsection{Event-agnostic future-behavior supervision}
To help a compact code acquire the required memory, we sample $j\sim U\{1,\ldots,T-1\}$ at each step $t$. A training-only head predicts a causal forecaster's future-action estimate $\tilde a^E_{t,j}(H_t)$:
\begin{equation}\label{eq:forecast-objective}
\mathcal L_{\rm future}=\mathbb E_{t,j}\!\left[\mathbf 1_{t+j\le T}\left\|q_\psi(O_t,C_t,j)-\tilde a^E_{t,j}(H_t)\right\|^2\right].
\end{equation}
The forecaster learns from recorded histories and provides conditional predictions even before cue visibility. No latent labels or event annotations enter the objective; only offsets beyond the episode are masked. Both forecaster and auxiliary head are discarded at evaluation. We optimize $\mathcal L_{\rm base}+\lambda(k)\mathcal L_{\rm future}$: $\lambda=1$ through $60\%$ of training, falls to zero by $80\%$, and leaves imitation and rate alone for the remaining steps.
\end{samepage}

\begin{table}[htbp]\centering\small
\caption{\textbf{Learning near the certified two-bit requirement.} First-gap rates are reported below; counts give full-gate sufficient seeds out of eight, and the last row gives all-seed closed-loop success. The event-agnostic configuration is frozen; all $40$ verdicts reproduce on independent recordings. The task-informed reference uses event structure.}\label{tab:learning-main}\label{fig:barrier}
\setlength{\tabcolsep}{4pt}
\begin{tabular}{@{}lccccc@{}}\toprule
training target & gap 6 & gap 10 & gap 20 & $M=16$ & $M=32$ \\\midrule
teacher internal state & $2/8$ & $0/8$ & $0/8$ & $0/8$ & $1/8$ \\
task-informed future behavior & $8/8$ & $8/8$ & $8/8$ & $8/8$ & $5/8$ \\
event-agnostic future behavior & $8/8$ & $7/8$ & $8/8$ & $6/8$ & $7/8$ \\\midrule
event-agnostic: closed-loop success & $.99$ & $.87$ & $.90$ & $.79$ & $.93$ \\\bottomrule
\end{tabular}
\end{table}

Event-agnostic supervision yields first-gap rates of $2.00$--$2.04$ bits on sufficient seeds, against a two-bit requirement. The frozen configuration ($K=16$, $\beta=10^{-3}$, $N=448$, $10^4$ steps) yields $36/40$ such seeds (Table~\ref{tab:learning-main}); every verdict reproduces on fresh recordings. A more informed reference predicts only the first grasp and placement actions and masks them after use. It reaches $37/40$, with lower post-use rates, but requires the event structure. These are two surrogates with different information requirements, not different definitions of $\Gam_t$.

\textbf{What the controls suggest.} Matching a full-history teacher's internal state yields only $0$--$2/8$ sufficient seeds: a state used for the current action need not expose a pending class that the teacher can retrieve later. Longer training, curriculum, and code refinement leave the longest A$'$ horizon at $0$--$1/8$. Improvement from future-behavior supervision at the same capacity therefore points to the long-range training signal as an important bottleneck, without identifying a unique optimization mechanism. Direct recorded-action targets work for the task-informed reference ($7/8$); they do not establish teacher-free performance for the frozen event-agnostic configuration (Appendix~\ref{app:generic}).

\textbf{State count and learning reliability.} Readout-3 requires three bits and at least eight states, yet sufficient-seed counts are $0,1,5,15,23$ out of $24$ for $K=8,10,12,16,24$, respectively; sufficient seeds average $2.99$--$3.14$ bits (Appendix~\ref{app:ksweep}). This separates the required state count from the capacity that supports reliable learning under the tested configuration. Small continuous RNNs also fail on the external T-maze, although these controls do not isolate quantization from network size and optimization (Appendix~\ref{app:exttmaze}).

\subsection{Pixels and predictive surplus}

\begin{table}[htbp]\centering\small
\caption{\textbf{Pixel code rates versus the certified requirement.} A$'$ at gap 20 requires $2\to1$ bits. Sixteen seeds per row; rates average sufficient seeds, success all seeds. Inputs retain probe and phase channels. The $60\%$ schedule was selected on seeds 0--7 and tested unchanged on seeds 8--15.}\label{tab:pixels-main}
\setlength{\tabcolsep}{5pt}
\begin{tabular}{@{}lccc@{}}\toprule
learner & first $\to$ second gap (bits) & sufficient & closed-loop success \\\midrule
plain & $1.99\to1.52$ & $1/16$ & $0.14$ \\
event-agnostic, annealed & $2.00\to1.13$ & $13/16$ & $0.65$ \\
task-informed reference & $2.00\to1.00$ & $15/16$ & $0.61$ \\\bottomrule
\end{tabular}
\end{table}

With pixels, sufficient codes approach the certified rates (Table~\ref{tab:pixels-main}). On held-out seeds, event-agnostic supervision yields $6/8$ sufficient policies and $0.57$ closed-loop success, against $0/8$ and $0.08$ for plain training; a schedule fixed before the pixel runs gives $12/16$ overall (Appendix~\ref{app:generic}). Grasp-side accuracy is high for all learners; the difference concerns the later placement. Pixel control precision still limits success.

\textbf{Acquisition and minimality respond differently to supervision.} Without annealing, event-agnostic supervision still acquires the memory at gap 20 ($7/8$), but the second-gap code rate remains $2.14$ bits from states and $2.01$ from pixels against a one-bit requirement. The expired grasp class contributes at most $0.02$ bit with or without annealing.

Conditioned on the pending placement class, the state-input code carries $0.56$ bit about the initial $x$-position quartile, compared with $1.16$ bits of total surplus. Removing positional jitter and retraining yields $8/8$ sufficient seeds and $0.94$ closed-loop success, while reducing un-annealed surplus from $1.16$ to $0.23$ bit. This supports initial-position variability as an important driver. With the original jitter, annealing reduces information about $x$ to $0.01$--$0.03$ bit and surplus to $0.17$--$0.21$ bit (Appendix~\ref{app:generic}).

The external memory chain also separates acquisition from forgetting. After the query appears, the requirement falls to about one bit, but sufficient runs with two- and three-bit contexts retain mean rates of $1.91$ and $2.69$ bits, despite their exact mid-delay rates (Appendix~\ref{app:cert}). The query is shown only at the final step, so the rate term can reward forgetting at that step alone.

An open-loop action target can reward continuous-control details outside the symbolic target, as well as distinctions that future observations will supply. The tested A$'$ controls separate two effects: future supervision helps acquire required memory, while removing it lets the rate objective reduce predictive surplus.

\section{Related representation targets}\label{sec:related}
\textbf{Predictive states.} Causal states and predictive-state representations preserve future observation distributions \citep{shalizi2001computational,still2010optimal,littman2002predictive,gangwani2020belief,ni2024bridging}; the $\epsilon$-transducer preserves conditional input--output behavior \citep{barnett2015computational}. Our fixed-expert target compares responses on common reachable continuations, treating future observations as side information. Predicting the full process can require additional distinctions (Appendix~\ref{app:targets}).

\textbf{States for control or identification.} Bisimulation, approximate information states, and stable quotients preserve reward, dynamics, or Markov structure \citep{givan2003equivalence,ferns2004metrics,castro2020scalable,subramanian2022ais,zhang2026minimal}. Inverse representations target control-endogenous distinctions \citep{mhammedi2023representation,lamb2023guaranteed,wu2024generalizing}; identification targets hidden parameters \citep{kumar2021rma,liang2024rma}. A specified expert can ignore these. Our matched objectives compare targets, rather than full published algorithms (Appendix~\ref{app:controls}).

\textbf{Information-limited memory.} Bounded-memory control uses sequential rate--distortion objectives \citep{fox2012bounded,fox2016minimum}; decision-centric memory retains distinctions supporting good decisions \citep{zou2026remember,walsh2026supportsufficiency,yamin2026whatmust}. Our fixed-expert characterization combines recurrent realizability, future observations as side information, and entropy minimality. Memory Lens bounds action-relevant information \citep{dann2016memory}; our target includes the anticipatory gap. Complementary work studies spurious history dependence \citep{dehaan2019causal,wen2020copycat,swamy2022sequence}, memory architectures and benchmarks \citep{yue2024learningmemory,wang2026camp,shah2026halo,torne2025ptp,gao2026gatedmemory,morad2023popgym,pleines2025memorygym,chen2026rmbench,dai2026robomme}, and the separation of memory from credit-assignment length \citep{ni2023transformers} or distillation \citep{parisotto2021efficient,weinzaepfel2026compressing}. These motivate our learning controls.

\needspace{9\baselineskip}\section{Scope and conclusion}\label{sec:conclusion}
The exact quotient formula requires history-determined expert behavior and transitive compatibility; general exact computation remains limited to small instances or matching finite-instance bounds. Experimental rates concern symbolic behavior under the stated observation convention and occupancy. Reliable closed-loop evidence is strongest through two bits. On physical weighing, expert-led observation generation restores grasp-side accuracy from $0.30$--$0.32$ to $0.93$--$0.99$, supporting retention and use of this distinction while subsequent control remains imprecise (Appendix~\ref{app:weigh}).

Event-agnostic prediction need not preserve a cue whose future action depends on observations yet to arrive. It extends the compact carrier's learning horizon on the external Passive T-maze but still fails at length $100$; a $128$-dimensional GRU is more reliable (Appendix~\ref{app:exttmaze}). Compactness serves bounded, auditable memory here, without an established architectural advantage.

The central result is the minimal recurrent behavioral memory of a specified expert. Future observations determine which historical distinctions must persist; the measurement protocol makes this requirement testable; and the learning experiments show both realizations near the boundary and a gap between information sufficiency and its acquisition.
\label{sec:main-end}

\clearpage
\subsection*{Reproducibility statement}
The theoretical quantities of the enumerated benchmark models are certified by refinement or matching bounds, with closed-form validation checks reported in Appendix~\ref{app:verify}; Appendix~\ref{app:proofs} states the assumptions and complete proofs. Appendix~\ref{app:solver} specifies the benchmarks, the binned-observation convention, the leak and body-memory probes, and the acceptance checks for each dataset. Section~\ref{sec:protocol} gives the frozen configuration for each family and the two evaluation protocols. Section~\ref{sec:f4} and Appendix~\ref{app:stats} specify the row set and statistical inference used in the pooled regression, and Appendix~\ref{app:resources} reports compute. Code for the environments, training, evaluation, probes and analysis, per-run result files for the learning comparisons, and the solver as a standalone package are available at \url{https://github.com/XianyaoLi/DIACRITIC}. Given an adapter that resets an environment for each hidden value, supplies an oracle expert and maps raw observations to symbols $\Ob_t$, the solver enumerates the hidden values, builds the induced finite symbolic model, checks (A2), (A4) and transitivity at every step, and returns $H(G_{E,t}\mid\Ob_t)$, $H(\Gam_t\mid\Ob_t)$ and $H(\Gam^s_t\mid\Ob_t)$; when transitivity fails it reports only the bracket $[H(G_{E,t}\mid\Ob_t),\,H(\Gam^s_t\mid\Ob_t)]$ of Theorem~\ref{thm:sandwich}, not a certified minimum. On a CPU, a single command (\texttt{python -m certify --paper}) reproduces the $13$ memory-chain and Passive T-maze certifications of \S\ref{sec:f1} against unmodified MIKASA-Base code (commit \texttt{ac81b6f}), together with the non-transitive instance of Appendix~\ref{app:cover} and the (A4)-failing re-reveal instance of Appendix~\ref{app:verify}. The corridor certificates of Appendix~\ref{app:cover} are reproduced by two separate scripts, and the pooled regression of \S\ref{sec:f4} by one. The solver requires enumerable hidden variables and a deterministic map from hidden values to symbolic observation sequences. Recorded demonstrations, trained models, and hard-code arrays are omitted because of their size; the recording and training scripts regenerate them. Estimator and code-content diagnostics require these artifacts in addition to the bundled ledgers.

\bibliography{refs}
\bibliographystyle{iclr2027_conference}

\clearpage
\appendix
\FloatBarrier
\noindent\textbf{Guide.} Appendix~\ref{app:proofs} proves the theorems and gives the finite-instance certificates, including the matching lower and upper bounds for the non-transitive corridor (\ref{app:cover}). Appendix~\ref{app:design-measure} specifies the benchmarks, the observation convention and the estimator controls; \ref{app:gate} compares the full-trajectory gate with the relaxed one, and \ref{app:audit} reports the injected-leak audit of \S\ref{sec:audit} with both detection rules and their clean-data baseline. Appendix~\ref{app:validation} validates the target on toys (a stochastic expert in \ref{app:stochexpert}), Task~A, the readout tasks and the corridor. Appendix~\ref{app:learning} covers learning: the event-agnostic objective, its held-out pixel evaluation, its predictive surplus and the jitter intervention in \ref{app:generic}, and the codebook sweep of \S\ref{sec:f4} in \ref{app:ksweep}. Appendix~\ref{app:domains} reports additional domains: pixels, the physical weighing task with its hybrid rollout (\ref{app:weigh}), the external Passive T-maze with small-state GRUs (\ref{app:exttmaze}), and the certified community benchmarks of \S\ref{sec:f1} (\ref{app:cert}). Appendix~\ref{app:targets} compares representation targets.

\needspace{8\baselineskip}\section{Proofs and exact certificates}\label{app:proofs}
\paragraph{Conventions.} All statements are made under the expert occupancy at a fixed step $t$ and hold almost surely. For concise notation, we state them for \emph{reachable} histories, namely histories with positive probability under the expert. We therefore assume that the reachable history space is countable at each step, as in the enumerated POMDP used by the solver; in the general case, ``for all reachable histories'' is replaced by ``almost surely.'' Under (A2), $G_{E,t}=\gamma_t(H_t)$ for a measurable $\gamma_t$. By Definition~\ref{def:quotient}, the expert's action distribution is then a function of $(O_t,G_{E,t})$, which we write as $\pi_E(\cdot\mid o,g)$. A \emph{context} is a random variable $C$ generated from $H_t$ by an encoder $p(c\mid h)$, and a decoder is a map $q(\cdot\mid o,c)$ into action distributions. Zero distortion means $\mathbb E[\delta(\pi_E(\cdot\mid S_t),q(\cdot\mid O_t,C))]=0$. By (A3) and $\delta\ge0$, this condition is equivalent to $q(\cdot\mid O_t,C)=\pi_E(\cdot\mid S_t)$ almost surely. Assumption (A1) is used only to ensure that $H(G_{E,t}\mid O_t)<\infty$.

\subsection{Instantaneous minimality}\label{app:instantaneous}
\begin{theorem}[Behavioral sufficiency]\label{thm:sufficiency}
Under (A1)--(A3), zero distortion implies $H(G_{E,t}\mid O_t,C)=0$; under (A2), $G_{E,t}$ is the coarsest sufficient context.
\end{theorem}
\begin{theorem}[Rate--distortion reduction]\label{thm:reduction}
Under (A1)(A2), for all $D$, $R_E(D)=R_{G\mid O}(D)$, the conditional rate--distortion function of the source $G_{E,t}$ given $O_t$. In particular $R_E(0)=H(G_{E,t}\mid O_t)$.
\end{theorem}

\begin{proof}[Proof of Theorem~\ref{thm:sufficiency}]
Fix $(o,c)$ such that $P(O_t=o,C=c)>0$, and let $s,s'$ be states in the fiber $\mathcal S_o$ that each occur with positive probability jointly with $C=c$. Zero distortion implies $\pi_E(\cdot\mid s)=q(\cdot\mid o,c)=\pi_E(\cdot\mid s')$. Hence $s\sim_Es'$ and $g_E(s)=g_E(s')$, so $G_{E,t}$ is almost surely a function of $(O_t,C)$; equivalently, $H(G_{E,t}\mid O_t,C)=0$. This argument does not use (A2). For the second claim, (A2) makes the context $C=G_{E,t}=\gamma_t(H_t)$ admissible, and the decoder $q(\cdot\mid o,g):=\pi_E(\cdot\mid o,g)$ attains zero distortion. Thus, $G_{E,t}$ is sufficient. By the first claim, every other zero-distortion context $C'$ satisfies $H(G_{E,t}\mid O_t,C')=0$. Therefore, $G_{E,t}$ is a function of $(O_t,C')$ for every sufficient context, which establishes that it is the coarsest such context.
\end{proof}

Let $R_{G\mid O}(D)$ be the conditional rate--distortion function \citep{gray1973conditional} of the source $G_{E,t}$ with side information $O_t$ at encoder and decoder, reproduction alphabet the action distributions, and distortion $d\big((o,g),q\big)=\delta\big(\pi_E(\cdot\mid o,g),q\big)$:
\[R_{G\mid O}(D)=\inf\big\{I(C;G_{E,t}\mid O_t):\ p(c\mid g,o),\ q(\cdot\mid o,c),\ \mathbb E\big[d\big((O_t,G_{E,t}),q(\cdot\mid O_t,C)\big)\big]\le D\big\}.\]
Because $\pi_E(\cdot\mid S_t)=\pi_E(\cdot\mid O_t,G_{E,t})$, the distortion of any encoder--decoder pair in either problem depends on the joint law of $(O_t,G_{E,t},C)$ only.
\begin{proof}[Proof of Theorem~\ref{thm:reduction}]
($\ge$) Let $(p(c\mid h),q)$ be feasible for $R_E(D)$. Under (A2), $G_{E,t}=\gamma_t(H_t)$, so the chain rule gives $I(C;H_t\mid O_t)=I(C;G_{E,t},H_t\mid O_t)=I(C;G_{E,t}\mid O_t)+I(C;H_t\mid G_{E,t},O_t)\ge I(C;G_{E,t}\mid O_t)$. Define the induced encoder $p'(c\mid g,o):=P(C=c\mid G_{E,t}=g,O_t=o)$. The pair $(p',q)$ produces the same joint law of $(O_t,G_{E,t},C)$, hence the same distortion, and is feasible for $R_{G\mid O}(D)$ with rate $I(C;G_{E,t}\mid O_t)\le I(C;H_t\mid O_t)$. Taking infima, $R_{G\mid O}(D)\le R_E(D)$.
($\le$) Let $(p'(c\mid g,o),q)$ be feasible for $R_{G\mid O}(D)$ and set $p(c\mid h):=p'\big(c\mid\gamma_t(h),o(h)\big)$. The joint law of $(O_t,G_{E,t},C)$ and the distortion are unchanged, and $C$ depends on $H_t$ only through $(G_{E,t},O_t)$, so $I(C;H_t\mid G_{E,t},O_t)=0$ and $I(C;H_t\mid O_t)=I(C;G_{E,t}\mid O_t)$. Hence $R_E(D)\le R_{G\mid O}(D)$.
At $D=0$, every feasible pair has $H(G_{E,t}\mid O_t,C)=0$ (Theorem~\ref{thm:sufficiency}), so $I(C;G_{E,t}\mid O_t)=H(G_{E,t}\mid O_t)-H(G_{E,t}\mid O_t,C)=H(G_{E,t}\mid O_t)$, attained by $C=G_{E,t}$; thus $R_E(0)=R_{G\mid O}(0)=H(G_{E,t}\mid O_t)$, finite by (A1).
\end{proof}

\begin{corollary}[Harmless ambiguity and selective forgetting]\label{cor:forget}
(a) Let $U$ be a hidden variable such that $G_{E,t}$ is a function of $(O_t,U)$ (e.g.\ the hidden part of the state). Then $H(U\mid O_t)=H(G_{E,t}\mid O_t)+H(U\mid G_{E,t},O_t)$; with $M$ equiprobable hidden modes partitioned into $R$ classes of $M/R$ modes each, $H(U\mid G_{E,t},O_t)=\log_2(M/R)$. (b) If $H(G_{E,t}\mid O_t,C)=0$ and $H(C\mid O_t)=H(G_{E,t}\mid O_t)$, then $H(C\mid G_{E,t},O_t)=0$ and $I(C;V\mid G_{E,t},O_t)=0$ for every random variable $V$, in particular for any nuisance $U_{\rm nui}$. (c) For every $D$ and $\epsilon>0$ there is an encoder with distortion at most $D$ and rate at most $R_E(D)+\epsilon$ of the form $p(c\mid h)=p'(c\mid\gamma_t(h),o(h))$, and for it $I(C;V\mid G_{E,t},O_t)=0$ for every source-side variable $V$ whose joint distribution is fixed before the encoder draws its independent randomness.
\end{corollary}
\begin{proof}
(a) $H(U\mid O_t)=H(U,G_{E,t}\mid O_t)=H(G_{E,t}\mid O_t)+H(U\mid G_{E,t},O_t)$ because $H(G_{E,t}\mid U,O_t)=0$; the second statement is the entropy of a uniform variable on $M/R$ values. (b) $H(C\mid G_{E,t},O_t)=H(C\mid O_t)-I(C;G_{E,t}\mid O_t)=H(C\mid O_t)-\big(H(G_{E,t}\mid O_t)-H(G_{E,t}\mid O_t,C)\big)=0$, and $0\le I(C;V\mid G_{E,t},O_t)\le H(C\mid G_{E,t},O_t)=0$. (c) The ($\le$) direction of the proof of Theorem~\ref{thm:reduction} turns any encoder for $R_{G\mid O}(D)$ within $\epsilon$ of the infimum into an encoder of the stated form with the same rate and distortion; for it $P(C=c\mid H_t=h,V=v)=p'(c\mid\gamma_t(h),o(h))$ for every such source-side $V$, so $C$ is conditionally independent of $V$ given $(G_{E,t},O_t)$.
\end{proof}
Part (c) is an existence statement: nuisance information is never \emph{required} to reach the frontier, but not every point on the frontier must be nuisance-free.

\begin{corollary}[Closed form]\label{cor:closed}
Suppose $G_{E,t}$ is uniform on $R\ge2$ classes and independent of $O_t$, the expert is deterministic with distinct actions across classes, and $\delta$ is total variation. Then $R_E(D)=\log_2R-h_2(D)-D\log_2(R-1)$ for $0\le D\le1-1/R$ and $R_E(D)=0$ for $D\ge1-1/R$, where $h_2$ is the binary entropy.
\end{corollary}
\begin{proof}
By Theorem~\ref{thm:reduction} it suffices to compute $R_{G\mid O}(D)$. Since $G_{E,t}\perp O_t$, a scheme conditioned on $O_t$ is a family of unconditional schemes indexed by $o$, with rate $\mathbb E_o[I(C;G_{E,t}\mid O_t=o)]\ge\mathbb E_o[R_G(D_o)]\ge R_G(\mathbb E_oD_o)\ge R_G(D)$ by convexity of the unconditional rate--distortion function $R_G$, and a scheme that ignores $o$ attains $R_G(D)$; so $R_{G\mid O}=R_G$. With a deterministic expert taking action $a_g$ in class $g$, $\delta\big(\pi_E(\cdot\mid g),q\big)=1-q(a_g)$, which is linear in $q$; for a fixed encoder the decoder minimizing $\mathbb E[1-q(a_G)\mid c]$ is a point mass on the most probable class, so restricting reproductions to class estimates $\hat G$ with Hamming distortion $\mathbf 1[\hat G\ne G_{E,t}]$ loses nothing. Converse: if $P(\hat G\ne G_{E,t})=P_e\le D\le1-1/R$, Fano's inequality gives $H(G_{E,t}\mid\hat G)\le h_2(P_e)+P_e\log_2(R-1)\le h_2(D)+D\log_2(R-1)$, the right-hand side being non-decreasing on $[0,1-1/R]$, so $I(G_{E,t};\hat G)\ge\log_2R-h_2(D)-D\log_2(R-1)$. Achievability: take $\hat G$ uniform and $G_{E,t}=\hat G$ with probability $1-D$, otherwise uniform on the other $R-1$ classes; the marginal of $G_{E,t}$ is uniform, the distortion is $D$, and $I(G_{E,t};\hat G)=\log_2R-h_2(D)-D\log_2(R-1)$. For $D\ge1-1/R$ a constant $\hat G$ has rate $0$ and distortion $1-1/R$.
\end{proof}

\begin{proposition}[Deterministic frontier]\label{prop:det}
Let $R^{\det}_E(D)=\inf H(C\mid O_t)$ over deterministic encoders $C=f(H_t)$ and decoders with distortion at most $D$, and let $R^{\det}_{G\mid O}(D)$ be the same infimum over $C=f(G_{E,t},O_t)$. Under (A1)(A2), $R_E(D)\le R^{\det}_E(D)\le R^{\det}_{G\mid O}(D)$ for all $D$, with equality throughout at $D=0$.
\end{proposition}
\begin{proof}
For a deterministic encoder $I(C;H_t\mid O_t)=H(C\mid O_t)-H(C\mid H_t,O_t)=H(C\mid O_t)$, so every deterministic feasible pair is feasible for $R_E(D)$ with the same rate: $R_E\le R^{\det}_E$. Under (A2) every $f(G_{E,t},O_t)$ equals the deterministic function $f(\gamma_t(H_t),o(H_t))$ of $H_t$, with the same joint law of $(O_t,G_{E,t},C)$, hence the same distortion and the same $H(C\mid O_t)$: $R^{\det}_E\le R^{\det}_{G\mid O}$. At $D=0$, $C=G_{E,t}$ is feasible for $R^{\det}_{G\mid O}$ with rate $H(G_{E,t}\mid O_t)=R_E(0)$ (Theorem~\ref{thm:reduction}), which closes the chain.
\end{proof}

\subsection{Recurrent minimality and support conditions}\label{app:support}
\begin{assumption}[Continuation-support homogeneity, A4]
For all $t$, equality of $(O_t,G_{E,t})$ for two histories implies equality of their continuation supports $U_t(h)$.
\end{assumption}
\begin{lemma}\label{lem:a4}
Under (A4), $\sim_t$ is an equivalence for every $t$ and $\sim_t=\sim^s_t$; $\Gam_t:=[H_t]_{\sim_t}$ is a right congruence, $\Gam_{t+1}=\Phi_t(\Gam_t,O_{t+1},A_t)$, and $G_{E,t}$ is a function of $(O_t,\Gam_t)$.
\end{lemma}
\begin{theorem}[Exact regime]\label{thm:exact}
Under (A1)--(A4), $H(C_t\mid O_t)\ge H(\Gam_t\mid O_t)\ge H(G_{E,t}\mid O_t)$ for every zero-distortion realization, and $C_t=\Gam_t$, $F=\Phi$ attains it: $R^{\rm mem}_t(0)=H(\Gam_t\mid O_t)$.
\end{theorem}

Recall from \S\ref{sec:theory} the continuation support $U_t(h)=\operatorname{supp}P_E(A_t,O_{t+1}\mid H_t=h)$, the extended history $hu=(h,a_t,o_{t+1})$ for $u=(a_t,o_{t+1})$, the relation $\sim_t$ of Definition~\ref{def:gamma}, and the strong relation $\sim^s_t$: $h\sim^s_Th'$ iff $h\sim_Th'$, and for $t<T$, $h\sim^s_th'$ iff $O_t(h)=O_t(h')$, $G_{E,t}(h)=G_{E,t}(h')$, $U_t(h)=U_t(h')$ and $hu\sim^s_{t+1}h'u$ for every $u\in U_t(h)$. A recurrent realization has $C_t=F_t(C_{t-1},O_t,A_{t-1})$ with a fixed $C_0$ and deterministic $F_t$, so $C_t=c_t(H_t)$ is a deterministic function of the history and $H(C_t\mid O_t)=I(C_t;H_t\mid O_t)$; zero distortion means $\hat\pi(\cdot\mid O_t,C_t)=\pi_E(\cdot\mid S_t)$ almost surely for every $t$ under the expert occupancy. A set of histories is \emph{$\sim_t$-compatible} if its members are pairwise $\sim_t$-related. Two facts are used repeatedly: (F1) if $h\sim_th'$ or $h\sim^s_th'$ then $O_t(h)=O_t(h')$ and $G_{E,t}(h)=G_{E,t}(h')$; (F2) $u\in U_t(h)$ iff $hu$ is a reachable history at $t+1$, and then $O_{t+1}(hu)=o_{t+1}$.

\begin{lemma}[Strong congruence]\label{lem:strong}
For every $t$, $\sim^s_t$ is an equivalence relation and a right congruence: if $h\sim^s_th'$ and $u\in U_t(h)=U_t(h')$ then $hu\sim^s_{t+1}h'u$. Consequently $\Gam^s_t:=[H_t]_{\sim^s_t}$ satisfies $\Gam^s_{t+1}=\Phi^s_t(\Gam^s_t,O_{t+1},A_t)$ for a deterministic $\Phi^s_t$, and $(O_t,G_{E,t})$ is a function of $\Gam^s_t$.
\end{lemma}
\begin{proof}
Reflexivity and symmetry are immediate. Transitivity by backward induction on $t$: $\sim^s_T$ is equality of the function $h\mapsto(O_T(h),G_{E,T}(h))$. For $t<T$, let $h_1\sim^s_th_2$ and $h_2\sim^s_th_3$; then $U_t(h_1)=U_t(h_2)=U_t(h_3)=:U$ and for every $u\in U$, $h_1u\sim^s_{t+1}h_2u$ and $h_2u\sim^s_{t+1}h_3u$, so $h_1u\sim^s_{t+1}h_3u$ by the induction hypothesis; with (F1) this gives $h_1\sim^s_th_3$. The right-congruence property is the last conjunct of the definition; it says that the class of $hu$ is determined by the class of $h$ and by $u$, which defines $\Phi^s_t$ on reachable pairs. (F1) gives the last claim.
\end{proof}

\begin{proof}[Proof of Lemma~\ref{lem:a4}]
We show by backward induction on $t$ that under (A4) $\sim_t=\sim^s_t$; the remaining claims then follow from Lemma~\ref{lem:strong}. At $t=T$ the relations coincide by definition. For $t<T$, $h\sim^s_th'$ implies $h\sim_th'$, because the strong relation requires agreement on every continuation of the common support and $\sim_{t+1}=\sim^s_{t+1}$ by induction. Conversely let $h\sim_th'$. By (F1) the two histories share $(O_t,G_{E,t})$, so (A4) gives $U_t(h)=U_t(h')=:U$ and $U_t(h)\cap U_t(h')=U$; the last conjunct of Definition~\ref{def:gamma} then says $hu\sim_{t+1}h'u$, i.e.\ $hu\sim^s_{t+1}h'u$ by induction, for every $u\in U$, which is exactly $h\sim^s_th'$. Hence $\sim_t$ is an equivalence, $\Gam_t=\Gam^s_t$ is well defined, $\Gam_{t+1}=\Phi_t(\Gam_t,O_{t+1},A_t)$ with $\Phi_t=\Phi^s_t$, and $G_{E,t}$ is a function of $\Gam_t$, a fortiori of $(O_t,\Gam_t)$.
\end{proof}
The same induction directly establishes transitivity of $\sim_t$ under (A4). If $h_1\sim_th_2\sim_th_3$, all three histories share $(O_t,G_{E,t})$; (A4) makes their supports equal, and transitivity at $t+1$ then transfers to $t$. Assumption (A4) is sufficient but not necessary: the re-reveal example in the remark below violates (A4) at one step, although $\sim_t$ is transitive at every step.

\begin{proof}[Proof of Theorem~\ref{thm:sandwich}]
\emph{Compatibility of cells.} Fix a zero-distortion realization and call $\mathcal H_t(o,c)=\{h\ \text{reachable}:O_t(h)=o,\ c_t(h)=c\}$ its cells. We show by backward induction on $t$ that every cell is $\sim_t$-compatible. Let $h,h'\in\mathcal H_t(o,c)$. Zero distortion gives $\pi_E(\cdot\mid s)=\hat\pi(\cdot\mid o,c)=\pi_E(\cdot\mid s')$ for (almost) all states $s$ consistent with $h$ and $s'$ consistent with $h'$, so $G_{E,t}(h)=G_{E,t}(h')$ by Definition~\ref{def:quotient} and (A2). If $t=T$ this is $h\sim_Th'$. If $t<T$, take $u=(a,o')\in U_t(h)\cap U_t(h')$. By (F2) both $hu$ and $h'u$ are reachable and share $O_{t+1}=o'$; $F_{t+1}$ being deterministic, they also share $C_{t+1}=F_{t+1}(c,o',a)$, so they lie in one cell at $t+1$ and are $\sim_{t+1}$-related by the induction hypothesis. Hence $h\sim_th'$. Transitivity of $\sim_t$ is never used.
\emph{Lower bound.} By the previous paragraph $G_{E,t}$ is a function of $(O_t,C_t)$ (equivalently, Theorem~\ref{thm:sufficiency} applied to the context $C_t$), so $H(C_t\mid O_t)\ge I(C_t;G_{E,t}\mid O_t)=H(G_{E,t}\mid O_t)$ for every zero-distortion realization, and $R^{\rm mem}_t(0)\ge H(G_{E,t}\mid O_t)$.
\emph{Upper bound.} $C_t:=\Gam^s_t$ with $F_t:=\Phi^s_t$ is a recurrent realization (Lemma~\ref{lem:strong}; $C_0$ is the class of the empty history). The policy $\hat\pi(\cdot\mid o,\gamma):=\pi_E(\cdot\mid o,g)$, with $g$ the common value of $G_{E,t}$ on the class $\gamma$, is well defined by (F1) and has zero distortion. Hence $R^{\rm mem}_t(0)\le H(\Gam^s_t\mid O_t)$. (This realization has as many states as $\sim^s_t$ has classes; in the countable setting adopted here the infimum defining $R^{\rm mem}_t(0)$ ranges over countable state spaces, and on the enumerated POMDPs of the experiments all class counts are finite.) Strictness of both sides is shown in the Remark below.
\end{proof}
\begin{proof}[Proof of Theorem~\ref{thm:exact}]
Under (A4), $\sim_t$ is an equivalence (Lemma~\ref{lem:a4}), so a $\sim_t$-compatible set is contained in a single class: $\Gam_t$ is a function of $(O_t,C_t)$ for every zero-distortion realization, and $H(C_t\mid O_t)\ge I(C_t;\Gam_t\mid O_t)=H(\Gam_t\mid O_t)\ge H(G_{E,t}\mid O_t)$, the last step because $G_{E,t}$ is a function of $\Gam_t$. The realization $C_t=\Gam_t=\Gam^s_t$, $F=\Phi$ of the previous paragraph attains $H(\Gam_t\mid O_t)$, so $R^{\rm mem}_t(0)=H(\Gam_t\mid O_t)$.
\end{proof}

\paragraph{Remark (both bounds can be strict).} (i) Let three equiprobable reachable histories $h_1,h_2,h_3$ at step $t=T-1$ share $(O_t,G_{E,t})$, with $U_t(h_1)=\{u,v_1\}$, $U_t(h_3)=\{u,v_3\}$, $U_t(h_2)=\{w\}$ for distinct $u,v_1,v_3,w$, and let the expert require different actions after $h_1u$ and $h_3u$ (distinct classes at $T$). Then $h_1\sim_th_2$ and $h_2\sim_th_3$ hold vacuously while $h_1\not\sim_th_3$: $\sim_t$ is not transitive and (A4) fails. Here $H(G_{E,t}\mid O_t)=0$; the three supports are distinct, so $\Gam^s_t$ has three classes and $H(\Gam^s_t\mid O_t)=\log_23$; a zero-distortion realization may merge $h_2$ with either neighbour but never $h_1$ with $h_3$, so $R^{\rm mem}_t(0)=h_2(1/3)$, strictly inside $[0,\log_23]$. (ii) If a behavioral class was revealed before $t$, is absent from $O_t$, is re-revealed at $t+1$, and is first used after $t+1$, then histories that differ only in that class have disjoint continuation supports at $t$ (their next observations differ), are vacuously $\sim_t$-related and are merged by $\Gam_t$, whereas $\Gam^s_t$ separates them because their supports differ: $\Gam^s$ charges a bit that the future will re-provide. In this instance $\sim_t$ is transitive at every step although (A4) fails at $t$. Without transitivity, the zero-distortion memories are exactly the closed compatible state assignments of Appendix~\ref{app:cover}, whose minimum entropy we compute by dynamic programming.

\begin{corollary}[Codebook]\label{cor:codebook}Under (A4), any sufficient realization has $K\ge\max_{t,o}|\Gam_t|_o\ge R$; the left side can exceed $R$ (A$'$: the joint class $(\beta_1,\beta_2)$ has four values while $R_1=R_2=2$ during gap$_1$).\end{corollary}
\begin{proof}
Here $|\Gam_t|_o$ is the number of classes of $\Gam_t$ with $O_t=o$ and $R=\max_{t,o}|G_{E,t}|_o$ the largest number of behavioral classes in a fiber. Under (A4) the exact-regime paragraph of the proof of Theorem~\ref{thm:exact} shows that, for every sufficient realization and every $(t,o)$, $\Gam_t$ is a function of $(O_t,C_t)$; the map $c\mapsto\Gam_t$ on the codes that occur together with $O_t=o$ is therefore onto the $|\Gam_t|_o$ classes with $O_t=o$, and $K\ge|\Gam_t|_o$. Since $G_{E,t}$ is a function of $\Gam_t$, $|\Gam_t|_o\ge|G_{E,t}|_o$, whose maximum is $R$. In A$'$ during gap$_1$, both $\beta_1$ (used at the grasp) and $\beta_2$ (used at the place) must be carried although each pending decision has only two behavioral classes, so $|\Gam_t|_o=|\beta_1\vee\beta_2|=4$.
\end{proof}

\begin{proof}[Proof of Theorem~\ref{thm:exact-trans}]
\emph{Well-definedness.} If $\sim_t$ is transitive it is an equivalence (reflexive and symmetric by definition), so $\Gam_t$ is the partition into classes. For a class $\gamma$ at $t$ and $u=(a,o')$ with $u\in U_t(h)$ for some $h\in\gamma$, set $\Phi_t(\gamma,o',a):=[hu]_{\sim_{t+1}}$. This does not depend on the representative: if $h,h''\in\gamma$ both have $u\in U_t(h)\cap U_t(h'')$, the last conjunct of Definition~\ref{def:gamma} gives $hu\sim_{t+1}h''u$, so both extensions lie in one class. Hence $C_t:=\Gam_t$, $F_t:=\Phi_t$ is a recurrent realization on reachable pairs; by (F1) $G_{E,t}$ is constant on each class, so $\hat\pi(\cdot\mid o,\gamma):=\pi_E(\cdot\mid o,g)$ has zero distortion and $R^{\rm mem}_t(0)\le H(\Gam_t\mid O_t)$.
\emph{Lower bound.} By the compatibility paragraph of the proof of Theorem~\ref{thm:sandwich}, which does not use transitivity, every $(O_t,C_t)$-cell of a zero-distortion realization is $\sim_t$-compatible; a compatible set is contained in a single class of an equivalence, so $\Gam_t$ is a function of $(O_t,C_t)$ and $H(C_t\mid O_t)\ge I(C_t;\Gam_t\mid O_t)=H(\Gam_t\mid O_t)$. Under (A4), $\sim_t=\sim^s_t$ (Lemma~\ref{lem:a4}), which recovers Theorem~\ref{thm:exact}.
\end{proof}

\subsection{The general case: closed compatible state assignments}\label{app:cover}
A deterministic recurrent realization assigns every reachable history to one state. For rate accounting it is without loss of generality to refine that state by the current observation, $C'_t=(O_t,C_t)$: this leaves $H(C'_t\mid O_t)=H(C_t\mid O_t)$ unchanged and yields a partition into joint $(O_t,C_t)$-cells. We optimize over these partitions rather than overlapping covers; an overlapping cover, as in incompletely specified machines \citep{paull1959minimizing}, additionally requires a selection map whose induced partition determines the entropy.
\begin{definition}[Closed compatible state assignment]\label{def:cover}
A sequence of partitions $(\mathcal P_t)_{t\le T}$ of the reachable histories is a closed compatible state assignment if (i) every cell of $\mathcal P_t$ is $\sim_t$-compatible, and (ii) it is closed: if $h,h'$ lie in one cell of $\mathcal P_{t-1}$ and $u\in U_{t-1}(h)\cap U_{t-1}(h')$, then $hu$ and $h'u$ lie in one cell of $\mathcal P_t$. Its rate at $t$ is $H(P_t\mid O_t)$ under the expert occupancy, where $P_t$ is the cell index. Because compatibility requires equal observations, every cell lies within one observation fiber.
\end{definition}
\begin{proposition}\label{prop:cover}
Under (A1)--(A3), the joint $(O_t,C_t)$-cell partitions of zero-distortion recurrent realizations are exactly the closed compatible state assignments, up to the rate-preserving refinement above, and for every $t$, $R^{\rm mem}_t(0)=\inf H(P_t\mid O_t)$ over closed compatible state assignments. On a finite POMDP the infimum is a minimum.
\end{proposition}
\begin{proof}
($\Rightarrow$) Refine a realization to $C'_t=(O_t,C_t)$. Its cells are compatible by the compatibility paragraph of the proof of Theorem~\ref{thm:sandwich}; closure follows from the deterministic update because histories in one cell share both $O_{t-1}$ and $C_{t-1}$. ($\Leftarrow$) Given a closed compatible assignment, let $C_t$ be its cell index. Closure makes $F_t(\mathrm{cell}(h),o',a):=\mathrm{cell}(hu)$ well defined on reachable pairs, and compatibility gives, by (F1), a single value of $G_{E,t}$ per cell, so $\hat\pi(\cdot\mid o,c):=\pi_E(\cdot\mid o,g)$ has zero distortion. Finally, $H(C'_t\mid O_t)=H(C_t\mid O_t)$ for the forward construction, while the reverse construction has $C_t=P_t$, establishing the rate identity.
\end{proof}
\paragraph{Exact computation.} On an enumerated finite POMDP, a backward dynamic program enumerates partition sequences satisfying compatibility and closure; concentrating its objective on step $t$ gives $R^{\rm mem}_t(0)$. Two exact reductions keep small instances tractable: histories with isomorphic futures are interchangeable, and future-independent components factorize. The solver returns $h_2(1/3)$ on the three-history instance of Remark~(i), strictly inside $[0,\log_23]$, and $H(\Gam_t\mid\Ob_t)$ at every step of Task~A at $M=4$ (at most 109 memoized partition types per level, 0.9\,s) and of A$'$. The latter instances are also certified directly by Theorem~\ref{thm:exact-trans}. The unpruned search grows rapidly with the number of mutually compatible histories with non-isomorphic futures (more than $10^7$ partition types per level at Task~A $M=16$), so it is a certificate for small instances rather than a general algorithm. The drifting corridor of Appendix~\ref{app:grid}, our non-transitive family, already has up to 384 histories per level at $W=3$ with maximal compatible sets of up to 96 histories and hundreds of thousands of non-transitive triples; exact search is impractical (no termination within 25 minutes); there the minimum is instead pinned by the matching bounds below. The general problem is a recurrent analogue of zero-error source coding with decoder side information \citep{witsenhausen1976zero,alon1996source}; we do not characterize its computational complexity.

\paragraph{A certified upper bound beyond the reach of the exact DP.} On the drifting corridor ($W\ge3$), any closed compatible assignment is realisable, so a constructed one certifies an upper bound. A greedy construction (cells merged only when pairwise compatible, every merge propagated to the successors that share a continuation, chains rolled back when an implied merge is incompatible; verified from scratch for compatibility and successor consistency) attains exactly the $W=1$ ladder, $2$ bits in the first hall and $1$ bit in the second, for $W=3$, $5$ and $7$ ($384$ to $896$ histories per level), and for $W=9,11,13$ with largest-cells-first ordering. The hall-wise brackets therefore tighten from $[0,\,2+\log_2 W]$ to $[0,2]$ and $[0,1]$: the drift, re-revealed by every hall observation, need not be stored. Sufficient learned codes match the $2$-bit construction in the first hall and retain a modest $0.2$--$0.4$-bit surplus in the second.
\begin{proposition}[Incompatibility entropy lower bound]\label{prop:coloring}
Within each observation fiber $o$ at step $t$, connect two reachable histories when they are incompatible, and weight vertices by $P(H_t=h\mid O_t=o)$. Let $H_\chi(t,o)$ be the minimum entropy of a proper coloring of this weighted graph. Under (A1)--(A3),
\[
R_t^{\rm mem}(0)\ge\sum_o P(O_t=o)H_\chi(t,o).
\]
Equality is certified whenever a closed compatible state assignment attains this lower bound.
\end{proposition}
\begin{proof}
Every joint $(O_t,C_t)$ cell of a zero-distortion realization is compatible (Theorem~\ref{thm:sandwich}), so its state labels form a proper coloring within each observation fiber. Minimizing over all proper colorings relaxes successor consistency and hence lower-bounds the conditional entropy of every realization. A closed compatible assignment supplies the matching realizable upper bound (Proposition~\ref{prop:cover}).
\end{proof}
\paragraph{Finite-instance equality on the corridor.} Dropping closure leaves a necessary condition: histories with the same $O_t$ that are pairwise \emph{incompatible} must occupy different memory states, so within each observation cell a zero-distortion memory is a proper coloring of the incompatibility graph, and the minimum entropy over proper colorings, weighted by the cell probabilities, lower-bounds $R^{\rm mem}_t(0)$. Histories with identical incompatibility neighbourhoods can be merged without loss (moving a vertex from the smaller color class to the larger preserves proper coloring and majorizes the mass vector), after which every cell has at most $12$ types and the minimum-entropy coloring is computed exactly by a subset dynamic program, cross-checked by enumerating all proper colorings. For $W=3,5,7$ the bound equals the certified upper bound at every step ($0$, $1$, $2$ and $1$ bits along the episode), so the minima of these finite instances are pinned exactly at the $W=1$ ladder. A one-line certificate suffices: no pairwise-compatible set of histories has conditional mass above $\alpha=1/4$ in the first hall or $1/2$ in the second (the type graph is a disjoint union of $K_4$'s, respectively $K_2$'s), hence $H\ge\log_2(1/\alpha)$; the plain clique bound is loose (at most $1.87$ of the $2$ bits). The same agreement holds at every step for $W=9$, $11$ and $13$ (up to $1664$ histories per level; exact coloring on every cell), although these instances belong to one structural family (identical type graphs) and only the largest-cells-first construction attains the bound. The lower bound is solved exactly; the upper bound is an independently verified constructive certificate. On the three-history instance of Remark~(i) the coloring bound equals the exact dynamic program, $h_2(1/3)=0.918$ bit, whereas the largest-compatible-set certificate gives only $\log_2(3/2)=0.585$, so the coloring bound is the one to use in general. These are computations on finite instances, not a result for the general non-transitive case, whose complexity we do not characterize. State-count minimization of incompletely specified machines is NP-hard \citep{pfleeger1973state}; this related result is not a complexity proof for our entropy objective.

\subsection{Nonzero distortion}\label{app:rdprop}
Let $R^{\rm mem}_t(D)$ be the infimum of $H(C_t\mid O_t)$ over deterministic recurrent realizations whose expected distortion is at most $D$ at every step. The instantaneous function $R_E(D)$ is defined in \S\ref{sec:theory} for the source at step $t$.
\begin{proposition}\label{prop:rd}
Under (A1)--(A3), $R^{\rm mem}_t(D)$ is non-increasing in $D$ and
\[R^{\rm mem}_t(D)\ge R_E(D)=R_{G\mid O}(D)\qquad\text{for every }D\ge0.\]
The inequality can be strict, including at $D=0$.
\end{proposition}
\begin{proof}
The feasible sets are nested in $D$, which gives monotonicity. Any deterministic recurrent realization induces at step $t$ an admissible instantaneous encoder $C_t=c_t(H_t)$ and decoder $\hat\pi(\cdot\mid O_t,C_t)$. Because $C_t$ is a function of $H_t$, $H(C_t\mid O_t)=I(C_t;H_t\mid O_t)$; taking the infimum over the more restricted recurrent class therefore gives $R^{\rm mem}_t(D)\ge R_E(D)$. Theorem~\ref{thm:reduction} gives the equality on the right. At zero distortion the lower bound is $H(G_{E,t}\mid O_t)$, whereas Theorem~\ref{thm:exact-trans} gives $H(\Gam_t\mid O_t)$ in the transitive regime; A$'$ is strict in both gaps.
\end{proof}
This proposition supplies only a lower bound for $D>0$. The recurrent constraint couples encoders across steps, and deterministic finite realizations need not admit time-sharing, so we claim neither convexity nor equality with the instantaneous frontier.

\subsection{Behavioral future sufficiency}
For $J\ge0$ define $\sim^{(J)}_t$ by unrolling Definition~\ref{def:gamma} $J$ steps: $h\sim^{(0)}_th'$ iff $(O_t,G_{E,t})(h)=(O_t,G_{E,t})(h')$; for $J\ge1$ and $t<T$, $h\sim^{(J)}_th'$ iff $h\sim^{(0)}_th'$ and $hu\sim^{(J-1)}_{t+1}h'u$ for all $u\in U_t(h)\cap U_t(h')$; and $\sim^{(J)}_T:=\sim^{(0)}_T$. Then $\sim^{(T-t)}_t=\sim_t$ and $\sim^{(J+1)}_t\subseteq\sim^{(J)}_t$, since each unrolling adds conjuncts. Write $\xi_{t:j}=(O_{t+1:t+j},A_{t:t+j-1})$ for a continuation of length $j$, so that $(H_t,\xi_{t:j})=H_{t+j}$. A continuation is reachable from $h$ iff each of its steps lies in the support of the history built so far, so the continuations reachable from both $h$ and $h'$ are exactly those built step by step from common supports, and unrolling the definition gives
\begin{multline}\label{eq:simJ}
h\sim^{(J)}_th'\iff O_t(h)=O_t(h')\ \text{and}\ P_E(A_{t+j}\mid h,\xi)=P_E(A_{t+j}\mid h',\xi)\\
\text{for all }0\le j\le J\text{ and all }\xi=\xi_{t:j}\text{ reachable from both},
\end{multline}
because under (A2) $P_E(A_{t+j}\mid H_{t+j})=\pi_E(\cdot\mid O_{t+j},G_{E,t+j})$ determines $G_{E,t+j}$ within the fiber of $O_{t+j}$, and $O_{t+j}$ is part of $\xi$ for $j\ge1$.
\begin{lemma}\label{lem:gammaJ}
Under (A4) every $\sim^{(J)}_t$ is an equivalence; $\Gam^{(J)}_t:=[H_t]_{\sim^{(J)}_t}$ satisfies $\Gam^{(0)}_t\preceq\Gam^{(1)}_t\preceq\dots\preceq\Gam^{(T-t)}_t=\Gam_t$, each partition refining the previous one, and $H(\Gam^{(J)}_t\mid O_t)$ is non-decreasing in $J$.
\end{lemma}
\begin{proof}
Induction on $J$, for all $t$ simultaneously: $\sim^{(0)}_t$ is equality of a function. If $\sim^{(J-1)}_{t+1}$ is an equivalence and $h_1\sim^{(J)}_th_2\sim^{(J)}_th_3$, the three share $(O_t,G_{E,t})$, (A4) equalizes their supports, and transitivity at $(J-1,t+1)$ gives $h_1u\sim^{(J-1)}_{t+1}h_3u$ for all common $u$, i.e.\ $h_1\sim^{(J)}_th_3$. Refinement is $\sim^{(J+1)}_t\subseteq\sim^{(J)}_t$; since the coarser partition is then a function of the finer one, $H(\Gam^{(J)}_t\mid O_t)\le H(\Gam^{(J+1)}_t\mid O_t)$.
\end{proof}

The population BFS distortion at horizon $J$ of a code $C_t=c_t(H_t)$ with decoders $q_j(\cdot\mid c,o,\xi)$, $j=0,\dots,J$, is
\[D_{\rm BFS}=\sum_{j=0}^{J}w_j\,\mathbb E\Big[\delta\big(P_E(A_{t+j}\mid H_t,\xi_{t:j}),\ q_j(\cdot\mid C_t,O_t,\xi_{t:j})\big)\Big],\qquad w_j>0,\]
with the expectation over $(H_t,\xi_{t:j})$ under the expert occupancy, so that every continuation reachable from a history receives positive weight (common reachable coverage, (ii)); $\delta$ is distributional, $\delta=0$ iff the distributions coincide ((iii); for a stochastic expert this is zero KL or zero excess log-loss, not zero log-loss); all prefixes $j=0,\dots,J$ are predicted ((iv)); the full horizon is $J=T-t$ ((i)). The $j=0$ term is the imitation loss; the target $A_{t+j}$ is never an input.
\begin{proposition}[BFS minimizers]\label{prop:bfs}Under (A1)--(A3), full-horizon prediction over all common reachable continuations with a distributional loss, $D_{\rm BFS}=0$ iff every $(O_t,C_t)$-cell is $\sim^{(J)}_t$-compatible; under (A4), iff $(O_t,C_t)$ refines $\Gam^{(J)}_t$, and with $J=T-t$ the rate-minimal zero-BFS code has $H(C_t\mid O_t)=H(\Gam_t\mid O_t)$.\end{proposition}
\begin{proof}
($\Rightarrow$) Let $D_{\rm BFS}=0$ and let $h,h'$ lie in the same cell $(o,c)$. For any $0\le j\le J$ and any $\xi=\xi_{t:j}$ reachable from both, the pairs $(h,\xi)$ and $(h',\xi)$ have positive probability, so by (iii) $q_j(\cdot\mid c,o,\xi)=P_E(A_{t+j}\mid h,\xi)$ and $q_j(\cdot\mid c,o,\xi)=P_E(A_{t+j}\mid h',\xi)$; the two expert conditionals agree, and \eqref{eq:simJ} gives $h\sim^{(J)}_th'$. ($\Leftarrow$) If every cell is $\sim^{(J)}_t$-compatible, set $q_j(\cdot\mid c,o,\xi):=P_E(A_{t+j}\mid h,\xi)$ for any $h$ in the cell $(o,c)$ from which $\xi$ is reachable; by \eqref{eq:simJ} any two such $h$ give the same value (pairwise compatibility suffices; no transitivity is used), and these decoders have $D_{\rm BFS}=0$. Under (A4), $\sim^{(J)}_t$ is an equivalence (Lemma~\ref{lem:gammaJ}), so a cell is compatible iff it is contained in a class, i.e.\ iff $(O_t,C_t)$ refines $\Gam^{(J)}_t$. For $J=T-t$, $\Gam^{(J)}_t=\Gam_t$; a code refining $\Gam_t$ has $\Gam_t$ as a function of $(O_t,C_t)$ and therefore $H(C_t\mid O_t)\ge H(\Gam_t\mid O_t)$, with equality for $C_t=\Gam_t$ (for stochastic codes the same holds with $I(C_t;H_t\mid O_t)$ in place of $H(C_t\mid O_t)$).
\end{proof}
\paragraph{Why every prefix is predicted (first divergence).} Suppose that the objective included only the $j=J$ term. Let $h,h'$ share a cell but be $\sim^{(J)}_t$-incompatible, and let $j^\star\le J$ be the smallest $j$ for which some common continuation $\xi_{t:j^\star}$ yields different expert conditionals. By the minimality of $j^\star$, the action prefix $A_{t:t+j^\star-1}$ supplied along $\xi_{t:j^\star}$ is common to $h$ and $h'$, so the $j^\star$ term is positive. Every longer continuation, however, contains $A_{t+j^\star}$, whose supports may already differ between $h$ and $h'$, as they can for a deterministic expert. In that case, no $\xi_{t:J}$ is common, and the $J$ term alone can vanish on an incompatible cell. Predicting every prefix, as required by (iv), assigns a cost at the first step at which the supplied actions have not already revealed the divergence. Supplying intermediate expert actions under teacher forcing is therefore not label leakage: these actions instantiate the same continuation consumed by the recurrent transition, and the prediction target is never provided as input. With finite demonstrations, finite $J$, and sampled continuations, the trained objective is an empirical surrogate for $D_{\rm BFS}$; when $J$ is finite, its target is $\Gam^{(J)}_t$ rather than $\Gam_t$.

\subsection{Closed-form verification of the theoretical quantities}\label{app:verify}
Before running any learning experiment, we validated every partition-refinement quantity in Appendix~\ref{app:solver} against closed-form results on the finite toy problems (the T0 gate). First, Blahut--Arimoto on the conditional source $(G_{E,t},O_t)$ agrees with $H(G\mid O)$ at $D=0$ to $7.8\!\times\!10^{-16}$. Second, for $M\in\{4,8,16\}$, Blahut--Arimoto on the full history source $(H_t,O_t)$ agrees with the reduced source to $1.3\!\times\!10^{-15}$, as predicted by Theorem~\ref{thm:reduction}. Third, the solver recovers the deterministic frontier $R^{\det}_{G\mid O}(D)$ shown in Figure~\ref{fig:toy}. Fourth, at every $t$, $H(\Gam_t\mid O_t)$ matches the hand-derived toy staircases: $[0,0,1,2]$ for the reveal toy and $\log_2(R_1R_2)\to\log_2R_2\to0$ for the gap toy. Fifth, partition refinement returns the strict bracket on a re-reveal instance ($\Gam^s$ requires 2 bits where $\Gam$ requires 1), refuses to construct $\Gam_t$ on a purpose-built non-transitive instance, and produces a $\Gam^{(J)}_t$ table monotone in $J$ (Lemma~\ref{lem:gammaJ}). This refinement pass takes below one second for every dataset, with at most 768 histories per level in the corridor; the general exact dynamic program has the more limited scope described in Appendix~\ref{app:cover}.

\FloatBarrier
\needspace{8\baselineskip}\section{Experimental design and measurement}\label{app:design-measure}
These sections fix the information structure before comparing learned rates. They distinguish symbolic behavioral requirements, code entropy, codebook capacity, and physical storage, and document the empirical sufficiency gate.

\subsection{Benchmarks and the symbolic observation convention}\label{app:solver}

\begin{table}[htbp]\centering\footnotesize\setlength{\tabcolsep}{3pt}\renewcommand{\arraystretch}{1.15}
\caption{Benchmark families, regimes and frozen primary configurations ($N$ = training episodes; $K=16$ codes throughout; the Franka runs use batch 512).}
\label{tab:bench}
\begin{tabularx}{\linewidth}{@{}>{\raggedright\arraybackslash}p{0.62in}>{\hsize=.90\hsize}L>{\hsize=.97\hsize}L>{\hsize=1.14\hsize}L>{\hsize=.93\hsize}L>{\hsize=1.06\hsize}L@{}}\toprule
 & finite toys & signpost corridor & A$'$ (distortion-constrained) & A$'$ (attainability) & Task A \\\midrule
domain & finite POMDPs & grid navigation & Franka, kinematic attach & Franka, kinematic attach & Franka, physics-simulated $2$\,kg grasp \\\addlinespace[3pt]
regime (solver) & exact; (A2), (A4), transitive & exact at $W{=}1$; certified at $W{=}3,5,\dots,13$ & exact & exact & exact: $\sim_t$ transitive, (A4) fails at one step \\\addlinespace[3pt]
learner & raw VQ, lr $3\!\times\!10^{-4}$ & raw VQ, tanh-bounded state & plain $-$R & scaffold-RF & plain $-$R \\\addlinespace[3pt]
$\beta$ & $0.03$ & $0.01$ & $\le10^{-3}$ & $2\!\times\!10^{-3}$ & $10^{-3}$ \\\addlinespace[3pt]
data, steps & 6k & 6k & $N{=}1152/3584$; 10k & $N\approx4$k & $N{=}448$; 10k \\\addlinespace[3pt]
role & exact validation & cross-domain replication & distortion-constrained & attainability only & matched performance \\
\bottomrule\end{tabularx}
\end{table}

\textbf{Observation convention and frozen binner.} For behaviorally decodable classes, define $\Ob_t:=f_t(O_t)=(\text{phase},\ \text{probe symbol},\ f^1_t(O_t),\ f^2_t(O_t))$. Here, $f^i_t$ is a frozen nearest-class-mean binner of the raw end-effector position at step $t$, with class means fitted once on half of the recorded episodes. The binner is applied only when the classes are separable, defined as class-conditional means more than $5$\,cm apart, or more than 25 standard deviations of the sensor noise; otherwise, it returns a null symbol. On every recorded dataset used in the paper, including A$'$ at gaps 6, 10, and 20, the 2048-episode set, and Task A, the frozen binner reproduces the class on $100\%$ of held-out episodes at each separable step. Specifically, there are $0$ disagreements among $2{,}816$--$7{,}040$ held-out (episode, step) pairs per dataset, both with and without conditioning on the other class. Except for the Task~A sag convention disclosed below, the solver labels are therefore exactly those produced by $f_t(O_t)$, and this binned observation is a coarsening of $O_t$. Under this convention, the data-processing inequality gives $H(C_t\mid\Ob_t)\ge H(C_t\mid O_t)$, so the reported rates upper-bound the raw-observation rate. Table~\ref{tab:leak} reports a complementary held-out probe: a classifier on raw-observation windows recovers a class only when the convention marks it as separable ($0.99$--$1.00$) and otherwise performs at chance.

\begin{table}[htbp]\centering\small
\caption{Design choices, premises and evidence (\S\ref{sec:findings}).}
\label{tab:design}
\footnotesize\begin{tabularx}{\linewidth}{@{}R{1.75in}R{1.45in}L@{}}\toprule
choice & premise & evidence \\\midrule
$C_t$ is the only carrier across time & Definition~\ref{def:realization} & Appendix~\ref{app:design}: bypass under-counts in three domains \\
persistence (residual proposal) & recurrence & non-persistent code: $0/8$ sufficient (trivial) \\
discrete codebook & directly measurable accounting & continuous: KL upper bound $6.9$ bit, not comparable (Table~\ref{tab:quadrant}) \\
prior conditioned on $O_t$ & matches the conditional-rate objective & no empirical advantage over an unconditional prior in these families \\
forecaster used at training only & sole carrier at evaluation & forecaster removed before every reported rate and rollout; $C_t$ is the only cross-time variable (\S\ref{sec:f4}) \\
\bottomrule\end{tabularx}
\end{table}

\textbf{The sag symbol.} On Task~A the symbolic observation carries a mass-half symbol during transport and placing. The online binner of Appendix~\ref{app:sideinfo} reproduces it on $99.5\%$ of transport steps but only $80\%$ of placing steps, so the placing convention credits the observation with a partly decodable symbol. This symbol is behaviorally irrelevant and no reported quantity at the place step involves mass; we retain the convention and state the discrepancy explicitly.

\textbf{Benchmark hygiene} \citep[cf.][]{tao2025pobax,agarwal2026longcontext}\textbf{.} Without sensor noise, sub-millimeter contact artifacts made the mass decodable at $79$--$93\%$ (2\,mm noise and a non-contact attach removed it); a capture offset in the first pixel collection produced apparently memory-free solutions that the sufficiency metrics flagged (data recollected); and a deterministic expert makes the inverse-dynamics objective degenerate in the nuisance, which motivated the corridor's behaviorally equivalent action randomness.

\textbf{Coverage of the empirical POMDP.} The solver is exact on the finite POMDP induced by the recorded symbolic histories. On 13 of the 14 datasets every latent configuration of the generating program is recorded (Good--Turing missing mass $0$, no unseen key), and the reachable histories per level number at most 8 (A$'$), 32 (tier 16), 64 (P-I) and 252 (Task~A). The exception is Task~A at $M=32$, where 2 of 128 mass--slot pairs are unrecorded (Good--Turing missing mass $0.02$). Their symbolic sequences are reconstructed from the generating program's binary mass code and mass-half sag rule; the same reconstruction reproduces every observed key when applied from a different donor. Under this reconstruction $H(\Gam^s\mid\Ob)=1.00$ bit, as at every other $M$, and no recorded label is modified.

\textbf{Estimator bias.} All entropies are plug-in estimates over the recorded episodes with exact occupancy weights. We compute the Miller--Madow correction to $H(C\mid\Ob)$ from the number of \emph{occupied} $(c,\bar o)$ cells at each step. The correction is at most $0.009$ bit for the 2048-episode set across 16 runs and all steps, at most $0.035$ bit for the 512-episode sets, and at most $0.017$ bit for the 128 own-occupancy episodes of the reference policy. The worst-case correction of $4\cdot15/(2\cdot512\ln2)=0.085$ bit is never approached because each symbol occupies at most 2--4 codes. These corrections assess finite-sample bias in the reported rates; matching $2.00$ bits is interpreted jointly with sufficiency and the replication checks. We do \emph{not} remove self-controlled observation components, such as the end-effector pose, from $O_t$; instead, the body-memory probe detects when they carry memory, as discussed in \S\ref{sec:box}.

\begin{table}[htbp]\centering\small
\caption{Leak probe on A$'$ (tier 4, gap 6): held-out accuracy ($n=128$ episodes) of a classifier on raw observation windows, per phase and hidden variable; chance in parentheses. Bold entries are the steps at which the visibility convention marks the class visible.}
\label{tab:leak}
\begin{tabular}{lccccc}\toprule
phase & $\beta_1$ (0.50) & $\beta_2$ (0.50) & $\theta$ (0.03) & mass bin (0.25) & distractor $z$ (0.25)\\\midrule
identification (scan) & \textbf{1.00} & \textbf{1.00} & 0.10 & 0.26 & --- \\
gap$_1$ & 0.53 & 0.51 & 0.04 & 0.23 & --- \\
grasp & \textbf{0.99} & 0.47 & 0.05 & 0.24 & --- \\
gap$_2$ & 0.51 & 0.46 & 0.03 & 0.18 & 0.99 \\
place & 0.49 & \textbf{1.00} & 0.04 & 0.26 & --- \\
\bottomrule\end{tabular}
\end{table}

The solver applies backward partition refinement from Definition~\ref{def:gamma} to an enumerated finite POMDP constructed from the recorded data. We verify that the mapping from latent state to symbolic-observation and action sequences is deterministic. The solver computes $\sim_t$ and $\sim^s_t$, performs per-step checks of (A2), (A4), and transitivity, refuses to construct $\Gam_t$ when the relation is non-transitive, and returns exact conditional entropies. The symbolic observation is $(\text{phase},\text{probe bit},\text{visible classes})$. A class is defined as visible at $t$ when it can be decoded from the raw observation, using class-conditional end-effector separation $>5$\,cm after conditioning on the other factor. Without this convention, information already present in the raw observation would be incorrectly counted as memory. For Task A, the symbolic observation also contains a mass-half symbol during the sag phases.

\subsection{Realization and training variants}\label{app:design}
The residual proposal is $\tilde e_t=e_{t-1}+f(e_{t-1},g_o(O_t),g_a(A_{t-1}))$, followed by deterministic nearest-code quantization. Quantization is hard at training and evaluation; gradients use straight-through and soft-to-hard rate estimators. Unless stated otherwise, $K=16$, the code vector has dimension $32$, and the hidden width is $128$. The code is the only learned recurrent state transmitted across steps. In the matched-side-information controls, the transition and prior instead receive the frozen per-step binner $\Ob_t=f_t(O_t)$.

The tables use compact variant labels: $-$R is the plain imitation-and-rate objective; $-$RF adds behavioral future sufficiency (BFS); scaffold uses an auxiliary continuous recurrent path during training, annealed to zero before evaluation. A bypass retains that continuous path and is therefore excluded from sole-carrier minimality claims. DIACRITIC names this implementation, not a separate architectural contribution.

BFS predicts every prefix of a future continuation from a frozen code and the intervening observations and actions. Its ideal full-horizon minimizers have the compatibility property of Proposition~\ref{prop:bfs}; the finite sampled objective is only a surrogate. Event-agnostic future-behavior supervision instead predicts a random-offset action without future observations (\S\ref{sec:f4}). These objectives differ precisely in decoder side information. Conditioning the prior matches the conditional-rate objective, although the conditional and unconditional priors are empirically indistinguishable in the reported families.

\paragraph{Additional attribution evidence.} With a persistent GRU bypass, readout-code rates are $0.9$--$1.9$ bits on the toy, $1.06$ on A$'$, and $1.52$ in the corridor despite a two-bit behavioral requirement. One closed-loop policy instead externalizes the placement class in its pose: end-effector separation grows from $0.04$ to $2.07$\,cm during transport ($p=5\!\times\!10^{-14}$) at roughly $2$\,cm imitation error. Expert- and policy-occupancy sufficiency agree on $31/32$ reference policies, but the additional domains demonstrate that this agreement is not universal. These controls motivate the attribution requirement in \S\ref{sec:box}.

\subsection{Rate accounting}\label{app:accounting}
The quantity $\sum_tH(C_t\mid C_{t-1},O_t)$ assigns a cost of 0 both to a code that stores only behavioral information and to one that additionally stores nuisance information, because past observations can pass through $C_{t-1}$ without cost. Conversely, $\sum_t H(C_t\mid C_{t-1})$ over-penalizes information that remains permanently visible. At zero distortion, the per-step quantity $H(C_t\mid O_t)$ is characterized by Theorems~\ref{thm:exact-trans} and~\ref{thm:sandwich} and assigns a cost when information must persist beyond its visibility. Throughout the paper, ``pay'' refers to representational burden rather than communication cost.

\subsection{Side-information and estimator controls}\label{app:sideinfo}
\paragraph{Resolution sweep (N1).} The reported rates condition on $\Ob_t=f_t(O_t)$. To test whether the code carries raw-observation detail that $\Ob_t$ omits, we refine the conditioning by the raw end-effector position on a grid of decreasing cell size and recompute $\Hc(C_t\mid f_r(O_t))$ for the 16 plain $-$R runs of Figure~\ref{fig:fig3} (11 sufficient). In the first gap the sufficient seeds give $2.010$ bits at $\Ob_t$ and $2.004$, $2.004$, $1.983$, $1.906$, $1.550$ bits with cells of $10$, $5$, $2$, $1$, $0.5$\,cm. The decrease is fragmentation: permuting the codes within each $\Ob_t$ cell, which preserves $H(C\mid\Ob)$ exactly and destroys any dependence on position, gives the same decrease to within $0.008$ bit at every resolution (at $0.5$\,cm there are 237 occupied cells with 5.4 episodes each). The Miller--Madow values range from $2.012$ to $1.762$, and a two-fold cross-validated logistic model of the code from the raw 27-dimensional observation and $\Ob_t$ gives $2.067$ bits, an upper bound on $H(C_t\mid O_t)$ that does not fragment. Thus neither diagnostic provides evidence that raw observation predicts the first-gap code beyond $\Ob_t$. In the second gap the sufficient seeds carry $1.52$ bits, of which $0.04$--$0.08$ bit exceeds the permutation null and is explained by the end-effector position of four low-$\beta$ seeds; the cross-validated bound is $1.48$ bits, which we quote wherever a leak-resistant post-use rate is needed.
\paragraph{Estimator bias (N2).} Recomputing $S_\Gam$ with Miller--Madow corrections applied consistently to the joint and marginal entropies changes it by at most $0.014$ across 160 runs, and the number of sufficient seeds is identical at thresholds $0.8$, $0.9$ and $0.95$ in every cell: the scores are bimodal, with every sufficient run at $S_\Gam=1.000$ on every second-gap step and the best insufficient run at $0.37$.
\paragraph{Matched side information (E1).} The pre-specified control feeds the recurrent transition and the conditional prior a one-hot of $\Ob_t=f_t(O_t)$ computed online by a frozen per-step binner (phase from the observation, probe symbol from the probe channels, visible classes by nearest per-step class mean of the end-effector position, mass half by a per-step height threshold on Task~A); the binner agrees with the symbolic labels on every recorded step of A$'$ and on every Task~A step except the mass-half symbol, which it reproduces on $99.5\%$ of transport steps but only $80\%$ of placing steps, where the end-effector height no longer separates the masses (the symbol is behaviorally irrelevant, and no reported quantity at the place step involves the mass); no unseen symbol occurred in closed loop. Results (8 seeds per cell; 128 closed-loop episodes per seed): on A$'$ (2048-episode set, $N=1152$) $6/8$ seeds are sufficient at $\beta=0$ and $5/8$ at $10^{-3}$, with first-gap rate $2.00$ on the sufficient seeds and mean closed-loop success $0.84$ and $0.92$ (raw input: $6/8$, $5/8$, $0.96$, $0.97$); on Task~A all four values of $M$ give $8/8$, grasp rate $0.00$--$0.10$ bit, and success $0.97$--$1.00$.

Two stronger variants also restrict the behavioral head: a hierarchical policy whose symbolic-action head reads $(C_t,\Ob_t)$ and whose memory-free controller executes from $O_t$, and an intent policy whose head regresses the action from $(C_t,\Ob_t)$ and whose controller refines it from $O_t$. The hierarchical variant is sufficient on $8/8$ Task~A seeds at $M=4,8,16,32$ (behavioral error $0.000$; closed-loop success $0.73$--$0.79$), and the intent variant on $8/8$ at the tested $M=4,32$ (success $0.93$--$0.98$).

Both give $0/8$ on A$'$: the hierarchical head fails at cross-entropy weights $1$, $10$, and $30$, and the intent head at both values of $\beta$; they retain $\beta_1$ (used three steps after the reveal) but lose $\beta_2$ (nineteen steps) before the place step, whose error is at chance. A zero-distortion recurrent realization using only $(C_t,\Ob_t)$ nevertheless exists under Theorem~\ref{thm:exact-trans}, so these failures concern optimization under the stricter read-side architecture rather than the theoretical rate target.

\paragraph{Strict read-side architectures with forecast supervision (E1$\times$E4).} Table~\ref{tab:e1e4} reports the two read-side-restricted architectures of the previous paragraph trained with the behavioral-forecast target of Appendix~\ref{app:distill}; the training-time head also reads only $(\Ob_t,C_t)$. Across all 64 runs, $S_\Gam>0.9$ at every first- and second-gap step and $S_G>0.9$ at every grasp and place step (the minimum of either score over all checked runs and steps is $1.000$), and all eight cells pass the pre-specified numerical gate. The hierarchical head directly exposes the learned behavioral decision. For the intent architecture, nearest-class decoding of the intent vector is only a diagnostic proxy and is not the action executed by the controller; we therefore omit it from the behavioral-error column. Sufficiency and rate are evaluated from the code in both architectures.

\begin{table}[htbp]\centering\scriptsize\setlength{\tabcolsep}{3pt}
\caption{Strict read-side architectures on A$'$ with behavioral-forecast supervision (8 seeds per cell; theory $2.00\to1.00$; forecaster removed at evaluation). Without the forecast target the same architectures are sufficient on $0/8$ seeds in every cell.}
\label{tab:e1e4}
\begin{tabular}{llcccccc}\toprule
gap & architecture & $\beta$ & sufficient & first-gap rate & post-use rate & behav. error & closed-loop success \\\midrule
6 & hierarchical & 0 & 8/8 & $2.03\pm0.08$ & 1.00 & 0.000 & 1.00 \\
6 & hierarchical & $10^{-3}$ & 8/8 & $2.09\pm0.17$ & 1.00 & 0.000 & 1.00 \\
6 & intent & 0 & 8/8 & $2.02\pm0.05$ & 1.09 & -- & 0.96 \\
6 & intent & $10^{-3}$ & 8/8 & $2.03\pm0.08$ & 1.00 & -- & 0.94 \\
20 & hierarchical & 0 & 8/8 & $2.04\pm0.10$ & 1.00 & 0.027 & 0.96 \\
20 & hierarchical & $10^{-3}$ & 8/8 & $2.03\pm0.08$ & 1.00 & 0.000 & 1.00 \\
20 & intent & 0 & 8/8 & $2.00\pm0.00$ & 1.00 & -- & 0.88 \\
20 & intent & $10^{-3}$ & 8/8 & $2.00\pm0.00$ & 1.00 & -- & 0.84 \\
\bottomrule\end{tabular}
\end{table}

\subsection{Sufficiency-gate sensitivity}\label{app:gate}
The main text uses one gate: $S_\Gam>0.9$ at every step of both gaps with a positive exact requirement and $S_G>0.9$ at the first grasp and first place step. A relaxed gate that inspects only the second gap (mean $S_\Gam>0.9$) and placement (mean $S_G>0.9$) tests what survives the grasp but does not check the first-gap join. Of $2211$ A$'$-family runs, $928$ pass the relaxed gate and $879$ the full gate; no run passes the full gate only, except on the weighing task, where the relaxed gate is undefined. The task-informed forecast cells are unchanged ($8/8$ at gaps 6, 10, 20 and $M=16$) except $M=32$ ($6\to5$ of $8$), and the readout tasks are unchanged. Counts that change (relaxed $\to$ full): plain $-$R on A$'$ at $\beta=2\!\times\!10^{-3}$ and $3\!\times\!10^{-3}$, $4\to2$ and $5\to3$; $-$RF at $3\!\times\!10^{-3}$, $2\to1$; GRU bypass, $6\to1$; privileged join head at gap 20, $8\to5$, and at $K=128$ on the 4-bit tier, $2\to1$; forecast supervision on the 4-bit tier, $3\to0$ ($K=16$) and $5\to2$ ($K=64$); P-I with forecast supervision at $M=128$ and $512$, $5\to4$ and $2\to0$; label stride 4, $8\to7$. The bypass row is the instructive one: a continuous state can restore the placement class late, so a gate on the second gap alone mistakes late recovery for a correct recurrent state throughout. Every sufficient seed under the full gate carries $2.00$--$2.15$ bits in the first gap. Under the full gate, $479$ of the $615$ passes among the $1688$ runs with stored per-step requirements have every score above $0.999$, and thresholds of $0.8$ and $0.95$ give $650$ and $576$ passes.

\subsection{Observation side information: an injected-leak audit}\label{app:audit}
Memory diagnostics and direct observation probes exposed three benchmark flaws (the pixel frame captured one step late, Appendix~\ref{app:pixels}; the resting pose that remembered the mass, Appendix~\ref{app:weigh}; sub-millimetre contact artefacts, Appendix~\ref{app:solver}). These cases motivate a controlled audit. On A$'$ (gap 6, $N=1152$) we \emph{inject} a side channel of controlled magnitude: during the second gap the observed height of the object in the hand is shifted by $\pm\delta$ according to the pending place class, on top of the $2$\,mm sensor noise. The injection leaves the symbolic observation unchanged, so the requirement under that convention remains one bit in the second gap.

For each $\delta$ we record a dataset, train the plain sole-carrier policy ($\beta=10^{-3}$) on $16$ seeds, and evaluate every policy in closed loop with the same leak. The low-rate prediction and flag A were fixed before the runs; flags B and B$'$ were formulated afterwards.

\begin{table}[htbp]\centering\footnotesize\setlength{\tabcolsep}{4pt}
\caption{Injected-leak audit. Top: $16$ seeds per magnitude $\delta$. \emph{Correct} means held-out class error below $0.05$ at the first place step. Flag A, fixed before the runs, pairs correctness with second-gap rate below $0.9$ bit. The exploratory flags pair correctness with failure of the full gate (B) or second-gap memory alone (B$'$). Bottom: false alarms on separate clean recordings, conditioned on correctness.}
\label{tab:audit}
\begin{tabular*}{0.78\linewidth}{@{\extracolsep{\fill}}lcccccc@{}}\toprule
$\delta$ (mm) & 0 & 0.5 & 1 & 2 & 4 & 8 \\\midrule
closed-loop success & $0.59$ & $0.52$ & $0.61$ & $0.83$ & $0.84$ & $0.96$ \\
slot accuracy & $0.68$ & $0.68$ & $0.79$ & $0.93$ & $0.95$ & $0.98$ \\
held-out action error (MSE) & $0.107$ & $0.116$ & $0.053$ & $0.015$ & $0.022$ & $0.007$ \\
behaviorally correct seeds & $5$ & $5$ & $7$ & $15$ & $14$ & $16$ \\
seeds passing the sufficiency gate & $5$ & $4$ & $7$ & $4$ & $6$ & $3$ \\
flag A: correct, rate $<0.9$ bit & $0/5$ & $0/5$ & $0/7$ & $0/15$ & $1/14$ & $0/16$ \\
flag B: correct, full gate fails & $0/5$ & $1/5$ & $0/7$ & $11/15$ & $8/14$ & $13/16$ \\
flag B$'$: correct, gap 2 fails & $0/5$ & $0/5$ & $0/7$ & $11/15$ & $8/14$ & $4/16$ \\
\bottomrule\end{tabular*}
\par\smallskip
\begin{tabular*}{0.78\linewidth}{@{\extracolsep{\fill}}lccc@{}}\toprule
clean baseline & correct runs & flag B & flag B$'$ \\\midrule
unsupervised learners & $166$ & $9$ ($5.4\%$) & $0$ ($0.0\%$) \\
all learner configurations & $627$ & $51$ ($8.1\%$) & $18$ ($2.9\%$) \\
\bottomrule\end{tabular*}
\end{table}

At $\delta=8$\,mm, success and slot accuracy are higher and action error is lower than in the clean condition. Conventional performance metrics can therefore improve on a contaminated benchmark. Flag A fails as a sensitive detector: only one of the $96$ runs is flagged. Codes in behaviorally correct leaked policies can retain other content, with rates of $1.0$--$2.7$ bits in the $2$\,mm condition, so a code rate below the requirement is not a necessary signature of leakage.

\needspace{6\baselineskip}
\paragraph{Exploratory behavioral--representation disagreement.} After observing these data, we defined flag B: a correct place decision paired with failure of the full sufficiency gate. It flags $11/15$, $8/14$, and $13/16$ correct policies at $\delta=2,4,8$\,mm. Applied unchanged to $1432$ separate clean runs recorded before the rule existed, it flags $9/166$ correct unsupervised runs and $51/627$ correct runs overall (Table~\ref{tab:audit}). These are empirical alarm rates, not a leakage certificate.

Two mismatches prevent a theorem-based inference. Behavioral sufficiency (Theorem~\ref{thm:sufficiency}, Appendix~\ref{app:instantaneous}) assumes zero distortion, whereas ``correct'' here permits $5\%$ classification error. A uniform binary target with symmetric $4\%$ error can have $H(G\mid C)=h_2(0.04)\approx0.242$ bit and $S_G\approx0.758$, satisfying our correctness criterion while failing the gate without any leak. Moreover, Theorem~\ref{thm:exact-trans} concerns reproduction over the full horizon; flag B pairs one place decision with a gate that also tests the first gap and grasp. Most clean firings involve a correctly remembered place class but a lost grasp class. They are genuine insufficiencies and do not imply a side channel.

Flag B$'$ instead pairs the place decision only with second-gap memory. It flags $0/166$ correct clean unsupervised runs and $18/627$ overall. Its injected-leak counts are $11/15$, $8/14$, and $4/16$ at $2,4,8$\,mm. At $8$\,mm the code can copy the visible leak during the second gap, restoring $S_\Gam$ there; thus even this matched signal is not monotone in leak magnitude. Both rules are exploratory diagnostics requiring follow-up.

\paragraph{Direct observation probes.} The raw-observation probes in Appendix~\ref{app:solver} directly test whether $O_t$ reveals a behavioral class not credited to $\Ob_t$, using held-out prediction and the corresponding symbolic-observation baseline. Such probes address the side-information discrepancy itself. Behavioral--representation disagreement can motivate that check, but neither it nor the absence of a low code rate establishes whether a leak is present.

\subsection{Independent-recording replication}\label{app:replication}
To check that the exact agreement between learned and theoretical rates is not a property of the recordings on which the solver was built, we recorded fresh expert datasets with new seeds (A$'$ gap 6: 1024 episodes; Task~A $M=4$ and $32$: 512 each, $M=512$: 1024, all stratified over the masses) and re-evaluated 96 saved models under teacher forcing, rebuilding the solver from the new recordings. On every dataset, the rebuilt solver returns the same $H(\Gam_t\mid\Ob_t)$, $H(G_{E,t}\mid\Ob_t)$ and $H(\Gam^s_t\mid\Ob_t)$ at every step (maximum difference $0.000$), together with the same (A4)-failure steps and visibility convention. Every A$'$ configuration preserves its seedwise sufficiency verdict; the first-gap rate of sufficient seeds is $2.00$--$2.02$ bits and post-use rates move by at most $0.03$ bit. DIACRITIC also remains sufficient on every Task~A seed ($8/8$, $8/8$, and $16/16$ at $M=4,32,512$), with grasp rate $0.02$--$0.10$ bit and at most $0.05$ bit about mass. System identification remains at $1/8$ and $0/8$ for $M=4$ and $32$. At $M=512$, its placement-sufficiency count falls from $7/8$ on the training recordings to $4/8$ on fresh recordings; averaged over all seeds, its fresh-data grasp rate is $8.70$ bits, $I(C;\mathrm{mass}\mid\Ob)=1.26$ bits, and mean placement $S_G$ changes from $0.923$ to $0.902$. This threshold sensitivity does not alter the rate separation, but it precludes a claim of seedwise replication for that baseline (released code).

We also re-evaluated 72 models from the large-$M$ P-I and coverage-controlled Task~A experiments on independently recorded, stratified datasets. For forecast-distilled P-I, the training-to-fresh all-seed rates are $2.36\to2.48$ at $M=128$, $2.63\to2.72$ at $M=512,N=448$, and $2.31\to2.38$ at $M=512,N=1000$; the corresponding late-gate counts are $5\to3$, $3\to3$, and $2\to0$ of 8. The capacity-relieved P-I system-identification rates replicate as $6.45\to6.49$ and $8.72\to8.69$ bits. In the fully covered Task~A cell, DIACRITIC and system identification both remain place-sufficient on $8/8$ seeds; their training-to-fresh rates are $0.01\to0.01$ and $8.17\to8.19$ bits, and their mass information is $0.01\to0.01$ and $1.43\to1.36$ bits. The same policies obtain closed-loop success $0.99$ and $0.23$, respectively. Thus full mode coverage does not remove the rate or closed-loop separation on Task~A, whereas large-$M$ P-I sufficiency is not stable across recordings.

\subsection{Inference latency and storage footprint}\label{app:cost}
\paragraph{Measurement.} All three carriers use the paper's widths ($d=32$, hidden width $128$; the Transformer has two layers and four heads). One policy step at history length $T$ is timed as the median of $200$ steps after $20$ warm-up steps, on one CPU core at batch $1$ and on a GPU at batch $64$ (Table~\ref{tab:cost}). The Transformer re-encodes its full token history at every step, so its per-step cost grows with $T$ and its carried state is the token history itself ($64(T-1)$ floats: $8.2$\,kB at $T=33$, $262$\,kB at $T=1024$). A key--value cache would remove the re-encoding but still stores $2LTh$ floats per episode, about $2$\,kB per step at these widths, so the footprint still grows with the episode. The recurrent carriers keep a fixed state of $128$ bytes (GRU, $32$ floats) or $129$ bytes (DIACRITIC, the code vector plus one code index) at a constant per-step cost ($0.55$\,ms on the CPU, $1.1$--$1.2$\,ms on the GPU). The comparison isolates the cost of committing information in advance; it is not a recommendation of one deployment architecture.

\begin{table}[htbp]\centering\footnotesize\setlength{\tabcolsep}{4pt}
\caption{Per-step inference latency (median of 200 steps) as a function of the history length $T$, with the widths of the paper's models. The Transformer re-encodes its token history at every step (no key--value cache), so these timings describe that implementation; its carried history is $64(T-1)$ floats (8.2\,kB at $T=33$, 262\,kB at $T=1024$), compared with 128 bytes for the GRU and 129 bytes for DIACRITIC (state plus an 8-bit code index).}
\label{tab:cost}
\begin{tabular}{lcccccc}\toprule
& \multicolumn{3}{c}{CPU, batch 1 (ms)} & \multicolumn{3}{c}{GPU, batch 64 (ms)} \\
$T$ & Transformer & GRU & DIACRITIC & Transformer & GRU & DIACRITIC \\\midrule
33 & 1.13 & 0.55 & 0.54 & 1.77 & 1.12 & 1.18 \\
128 & 1.78 & 0.55 & 0.55 & 1.83 & 1.18 & 1.17 \\
256 & 3.51 & 0.56 & 0.53 & 2.12 & 1.15 & 1.18 \\
512 & 9.03 & 0.55 & 0.54 & 5.45 & 1.10 & 1.13 \\
1024 & 29.3 & 0.56 & 0.55 & 16.8 & 1.13 & 1.17 \\
\bottomrule\end{tabular}
\end{table}

\subsection{Training-time cost of forecast supervision}\label{app:traincost}

\begin{table}[htbp]\centering\footnotesize\setlength{\tabcolsep}{4pt}
\caption{Cost bookkeeping on A$'$ gap 20 (one B200 GPU; medians over runs). The forecaster and the auxiliary head exist only during training.}
\label{tab:traincost}
\begin{tabular}{@{}R{0.34\linewidth}ccR{0.30\linewidth}@{}}\toprule
 & parameters & training time & at deployment: carried state / per-step latency \\\midrule
forecaster, 3000 steps (default) & $0.49$--$0.51$\,M & $11$--$12$\,s & removed \\\addlinespace
random-offset forecaster, 100 steps (timing only) & $0.51$\,M & $\approx1$\,s & removed \\\addlinespace
compact student, 10\,000 steps, with or without the auxiliary head & $0.19$\,M & $\approx800$\,s & $129$ bytes / $0.55$\,ms \\\addlinespace
full-history Transformer baseline, 10\,000 steps & $0.49$\,M & $\approx650$\,s & $64(T-1)$ floats / $1.1$--$29$\,ms for $T=33$--$1024$ \\
\bottomrule\end{tabular}
\end{table}

Fitting the default forecaster adds about $1.5\%$ to the student's training time; the measured 100-step random-offset forecaster fit takes under $0.2\%$. Separate \emph{task-informed} controls at A$'$ gap 20 yield $8/8$ sufficient seeds with a 100-step forecaster and $7/8$ with recorded-action targets (Appendix~\ref{app:generic}). These learning results do not establish either performance claim for the frozen event-agnostic configuration. Deployment costs remain those of Appendix~\ref{app:cost}.

\subsection{Compute and data collection}\label{app:resources}
We collected demonstrations on one RTX~4090 using 64 environments, requiring $\approx$8\,s per 64 episodes. Offline training used one GPU per run on an NVIDIA B200, with an RTX PRO 6000 as fallback; a 10k-step run requires $\approx$12--25\,min. Closed-loop evaluation was performed in Isaac Sim on an RTX PRO 6000, requiring $\approx$65\,s per 128 episodes. Metered scheduler usage over all cluster allocations for this work, including development, failed, and unreported runs, comprises $3{,}033$ training tasks: 403.1 GPU-hours on B200 across $1{,}925$ tasks and 178.0 GPU-hours on RTX PRO 6000 across $1{,}108$ tasks. The $2{,}993$ closed-loop evaluation tasks used 66.9 GPU-hours and the 52 independent re-recording tasks 2.5 GPU-hours on RTX PRO 6000, for a total of $\approx$651 GPU-hours. Corridor training, the external benchmarks, and teacher-forced re-evaluation on independent recordings ran in 464 CPU-only tasks. Local development, data collection, and the pixel evaluation used one RTX~4090 and were not metered. The two training GPU types were used interchangeably for identical jobs, and we do not report a timing comparison between them.

\FloatBarrier
\needspace{8\baselineskip}\section{Validation of the behavioral memory target}\label{app:validation}
\subsection{Finite toy validation}\label{app:toy}

\begin{figure}[htbp]\centering
\includegraphics[width=\linewidth]{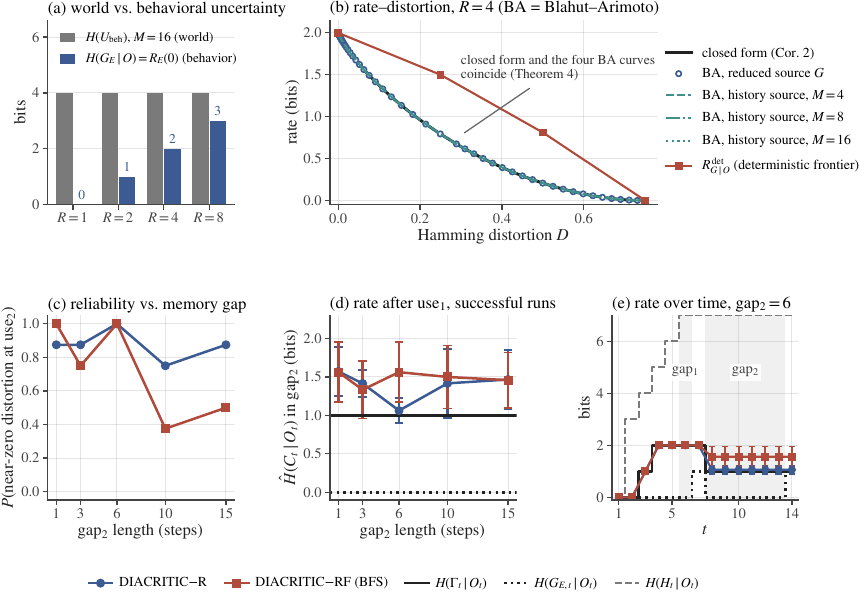}
\caption{\textbf{Toy validation.} (a) $R_E(0)=\log_2R$ is independent of $M$: the world uncertainty $H(U_{\rm beh})$ is fixed at $\log_2 16$ while the behavioral rate is $\log_2 R$. (b) Blahut--Arimoto on the full history source agrees with the closed form for the reduced source (Theorem~\ref{thm:reduction}) for $M=4,8,16$; the deterministic frontier lies above the stochastic curve. (c)--(e) The gap toy (unified configuration, 8 seeds): reliability as a function of the gap$_2$ length for $-$R and $-$RF, the learned post-use rate of successful runs relative to $H(\Gam_t\mid O_t)=1$, and the per-step rate at gap$_2=6$ against the three-level staircase. Panels (c)--(e) use near-zero distortion at the second decision as their diagnostic.}
\label{fig:toy}
\end{figure}

\paragraph{Toy results.}
All toy results can be reproduced from the released result files; these experiments use CPUs and the stated numbers of seeds.

\emph{Gate (T0).} On the reveal toy, the solver returns $H(\Gamma_t|\bar O_t)=[0,0,1,2]$ bits. On the gap toy, it returns $[0,0,1,2,2,2,2,1,1,1,1]$, corresponding to 2 bits at grasp, 1 bit in the gap, and 0 after use. It reports bounds without constructing $\Gam_t$ on a hand-constructed non-transitive instance and separately flags a violation of (A2). On a re-reveal instance, it returns a $\Gamma$ staircase of $2\to1\to0$, while $\Gamma^s$ remains at 2, demonstrating that the upper bound in Theorem~\ref{thm:sandwich} can be strict.

\emph{Exact rate--distortion (Figure~\ref{fig:toy}).} Blahut--Arimoto on the reduced source agrees with the closed form in Corollary~\ref{cor:closed} to $7.8\times10^{-16}$. For $M\in\{4,8,16\}$, the history source gives the same curve, confirming both Theorem~\ref{thm:reduction} and independence from $M$ to $1.3\times10^{-15}$. The deterministic frontier is $(2.0,0),(1.5,.25),(0.81,.5),(0,.75)$, compared with $2.0/0.79/0.21/0$ for the stochastic frontier (Proposition~\ref{prop:det}).

\emph{Enumeration regime.} When the population conditional is available ($\beta\in[0,0.1]$, 3 seeds), both -R and -RF recover $\hat H(C|\bar O)=[0,0,1,2]$ exactly, with $S_\Gamma=1$, $\hat I(C;z|\bar O)=0$, and $D_{TV}\le0.002$. At $\beta=0.3$, the code collapses (0.67/1.67 bits, $D=0.042$).

\emph{Finite-sample selective forgetting.} With $n=48$, $p=0.7$, and 4 seeds, the nuisance information $\hat I(C;z|\bar O)$ decreases as $0.81\to0.55\to0.05$, and the surplus $\hat H(C|\Gamma,\bar O)$ decreases as $1.54\to0.98\to0.10$, when $\beta$ changes from $0\to0.1\to0.3$. At the decision step, $S_\Gamma$ changes as $0.54\to0.94\to0.56$, with the final value reflecting collapse. On the gap toy ($n=512$, gap 6), the gap-to-gap rate changes from $3.27\to2.94$ at $\beta=0.01$ to $1.90\to1.40$ at $\beta=0.1$, while $\hat I(C;z|\bar O)$ decreases from $0.36\to0.01$. These measurements show reduced surplus under stronger rate pressure, together with a loss of sufficiency when the pressure becomes too large.

\emph{Factorized ablation.} At gap 6 with 4 seeds and $\beta=0.03$, DIACRITIC is sufficient on 3/4 seeds at a gap-1 rate of 2.00; the \emph{nonpersistent} variant reaches 0/4, and the \emph{uncond} variant reaches 4/4 and is indistinguishable on this toy. The \emph{continuous} carrier succeeds only at $\beta=0.003$, with a 6.9-bit KL bound and 0.32 bit of nuisance information. This bound does not provide a comparably tight minimality measurement and is not an exact rate. The \emph{bypass} variant, in which a continuous GRU state persists across time, succeeds on 20/24 runs across 12 seeds$\times$2 learning rates but reports $\hat H(C|\bar O)=0.9$--$1.9<2.00$ at $D=0$. Because memory bypasses the code, this readout rate underestimates the memory rate, motivating the sole-carrier requirement.

\emph{Violation of (A2).} Under a Bernoulli(0.3) reveal loss, the solver flags (A2). Replacing the privileged quotient by the history-conditional action law $P_E(A_t\mid H_t)$ defines an observational relation $\Gamma^{\rm obs}$ with $H(\Gamma^{\rm obs}\mid\bar O)=[0,0,1.58,2.58]$, where ``lost'' forms a third class. Both objectives recover this observational rate, with $S_{\Gamma^{\rm obs}}=1$, $D_{TV}\le0.007$, and zero nuisance information. This does not restore reproducibility of the demonstrator's privileged coupling, which lies outside the theorem's assumptions (cf.\ Remark~4.7 of \citet{yu2026capability}).

\emph{Discovery aids on the toy.} Across gaps $\{6,10,15\}$ and 4 seeds, future-sufficiency-triggered code refinement does not change the number of sufficient seeds: -R gives 3,3,4 with and without refinement, while -RF gives 3,2,2 versus 3,2,1. The annealed continuous scaffold is modestly more reliable, with 4,4,3 for scaffold-R and 3,4,4 for scaffold-RF.

\subsection{A stochastic expert: the quotient is defined on action laws}\label{app:stochexpert}
Definition~\ref{def:quotient} groups hidden states by the expert's action \emph{distribution}, whereas the other experiments use deterministic experts. In a finite check three equiprobable hidden states are displayed at $t=1$ and hidden afterwards; at $t=3$ the expert draws one of two actions with $P(L\mid s)=p,p,1-p$. States 0 and 1 induce the same non-degenerate law, and the solver merges them: $H(G_{E,t}\mid O_t)=H(\Gam_t\mid O_t)=h_2(1/3)=0.918$ bit rather than $\log_23=1.585$ (transitive; (A4) holds). For the learner (cross-entropy on sampled demonstrations, $K=8$) the prediction that the rate would be within $0.1$ bit of $0.918$ on every seed failed: at $p=0.7$ the quotient is found on $4/8$ seeds without rate pressure and on none with $\beta\ge0.01$ (the code collapses to zero bits, KL $0.2$ bit; distinguishing the third state is worth only $0.1$ bit of likelihood at one step); at $p=0.9$ it is found on $6/8$ seeds at $\beta=0$ and on $1$--$4$ of $8$ with rate pressure. In all $64$ runs the two equal-law states share a code and no run uses three codes: whenever a memory is learned it is the distributional quotient, and the observed failures collapse the code rather than separate the equal-law states (cf.\ \S\ref{sec:f4}).

\subsection{Task A and readout controls}\label{app:controls}

\begin{figure}[htbp]\centering\includegraphics[width=\linewidth]{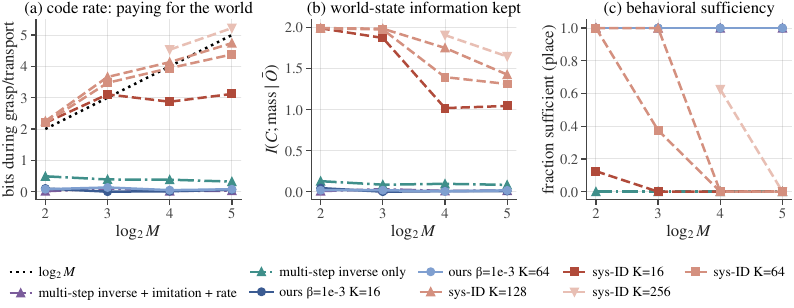}
\caption{\textbf{Task A controls} (8 seeds per point). Left: code rate during grasp and transport as a function of $\log_2 M$. The capacity-relieved system-identification rate rises approximately linearly with $\log_2 M$, whereas DIACRITIC with $K=16$ or $64$ remains at $\le0.13$ bit. Middle: retained mass information. Right: fraction of seeds sufficient at the place step; the multi-step-inverse-only memory never reaches sufficiency.}
\label{fig:sysK}
\end{figure}

\textbf{Codebook capacity for system identification.} At $K=16$, a code can represent at most $4$ bits, whereas identifying $M=32$ masses requires $5$ bits. The $0/8$ sufficiency of the $K=16$ system-identification baseline at large $M$ may therefore reflect limited capacity. Table~\ref{tab:sysK} repeats this baseline with $K=64$, $128$, and $256$, holding the data, training steps, and $\beta$ fixed. Greater capacity is necessary but not sufficient: system identification becomes sufficient at $M=4$ with $K=64$, at $M\le8$ with $K=128$, and on $5/8$ seeds at $M=16$ with $K=256$, but remains insufficient at $M=32$. With the largest codebook tested at each $M$, its rate increases from $2.25$ to $5.22$ bits; closed-loop success ranges from $0.15$ to $0.79$. DIACRITIC with $K=64$ is unchanged at every $M\le32$: it is sufficient on $8/8$ seeds, carries $0.05$--$0.13$ bit with $I(C;\mathrm{mass}\mid\Ob)\le0.02$, and reaches success of $0.99$--$1.00$. System identification is therefore capable of solving the smaller instances, but its learned rate grows with world complexity and does not deliver comparable control accuracy through a shared carrier. In settings where $K$ is not fixed by a theoretical fiber bound, a practical held-out selection rule is to increase $K$ until the sufficiency gate is first met and the measured rate stabilizes across successive values, then freeze $K$ before final evaluation.

\textbf{Multi-step inverse objectives.} Multi-step inverse kinematics and agent-centric state discovery \citep{mhammedi2023representation,lamb2023guaranteed} train an encoder to predict $A_t$ from the representation at $t$ and a future observation $O_{t+j}$. We add a recurrent analogue to our model. This is a MusIK/ACSD-\emph{style} baseline rather than a direct implementation of either published algorithm, and uses a head $q(A_t\mid C_t,O_t,O_{t+j})$ with $j\sim U\{1,\dots,8\}$. We evaluate two variants. In \emph{inverse only}, the inverse head and VQ losses shape the memory, the policy head receives a detached $(O_t,C_t)$, and no rate term is used. In \emph{inverse + imitation + rate}, the inverse head is added to our objective.

Table~\ref{tab:sysK} shows that inverse-only memory never reaches sufficiency in these tests: it obtains $0/8$ at every $M$, the recorded class-error metric is $0.30$--$0.33$, and closed-loop success of $0.03$--$0.06$. Instead, it retains $0.3$--$0.6$ bit about the mass. A future observation can reveal where the object was placed, allowing the inverse head to predict an earlier action without retaining the slot across the gap. This explains why the inverse target need not induce the anticipatory representation; the learned nuisance information is consistent with its emphasis on dynamics. When added to imitation and the rate term, the inverse objective has no measurable effect: the model reaches $8/8$, matches the behavioral code rate and $I(C;\mathrm{mass}\mid\Ob)$, and obtains closed-loop success of $0.99$--$1.00$. For demonstrations from a deterministic expert, the inverse problem does not depend on the nuisance, allowing the rate term to remove it. Empirically separating the objectives for a layer-2 nuisance, an agent-centric variable ignored by the expert, requires behaviorally equivalent action randomness whose observable effect depends on that nuisance; Appendix~\ref{app:grid} provides such a construction.

\begin{table}[htbp]\centering\small
\caption{Task A controls, 8 seeds per cell. Top block: seeds sufficient at the place step / closed-loop success (128 episodes per seed). Bottom block: code rate during grasp (bits) / $I(C;\mathrm{mass}\mid\Ob)$ (bits).}
\label{tab:sysK}
\begin{tabular}{lcccc}\toprule
& $M=4$ & $M=8$ & $M=16$ & $M=32$\\\midrule
\multicolumn{5}{@{}l}{\emph{sufficient seeds / closed-loop success}}\\
ours, $K=16$ & 8/8 / 0.91 & 8/8 / 0.96 & 8/8 / 0.98 & 8/8 / 0.98\\
ours, $K=64$ & 8/8 / 1.00 & 8/8 / 1.00 & 8/8 / 0.99 & 8/8 / 1.00\\
sys-ID, $K=16$ & 1/8 / 0.11 & 0/8 / 0.08 & 0/8 / 0.06 & 0/8 / 0.07\\
sys-ID, $K=64$ & 8/8 / 0.71 & 3/8 / 0.18 & 0/8 / 0.08 & 0/8 / 0.09\\
sys-ID, $K=128$ & 8/8 / 0.79 & 8/8 / 0.34 & 0/8 / 0.12 & 0/8 / 0.16\\
sys-ID, $K=256$ & -- & -- & 5/8 / 0.15 & 0/8 / 0.16\\
multi-step inverse only & 0/8 / 0.03 & 0/8 / 0.04 & 0/8 / 0.05 & 0/8 / 0.06\\
inverse + imitation + rate & 8/8 / 1.00 & 8/8 / 0.99 & 8/8 / 0.99 & 8/8 / 0.99\\\midrule
\multicolumn{5}{@{}l}{\emph{code rate during grasp / $I(C;\mathrm{mass}\mid\Ob)$ (bits)}}\\
ours, $K=16$ & 0.10 / 0.04 & 0.00 / 0.00 & 0.01 / 0.01 & 0.06 / 0.01\\
ours, $K=64$ & 0.08 / 0.02 & 0.13 / 0.02 & 0.05 / 0.00 & 0.08 / 0.02\\
sys-ID, $K=16$ & 2.20 / 1.99 & 3.11 / 1.87 & 2.87 / 1.02 & 3.12 / 1.04\\
sys-ID, $K=64$ & 2.20 / 1.99 & 3.47 / 1.98 & 3.94 / 1.40 & 4.38 / 1.31\\
sys-ID, $K=128$ & 2.25 / 1.99 & 3.66 / 1.99 & 4.13 / 1.75 & 4.74 / 1.43\\
sys-ID, $K=256$ & -- & -- & 4.53 / 1.91 & 5.22 / 1.65\\
multi-step inverse only & 0.49 / 0.13 & 0.39 / 0.09 & 0.38 / 0.10 & 0.32 / 0.08\\
inverse + imitation + rate & 0.01 / 0.01 & 0.10 / 0.03 & 0.02 / 0.01 & 0.03 / 0.01\\
\bottomrule\end{tabular}
\end{table}

\paragraph{World complexity beyond five bits.} Table~\ref{tab:bigM} extends Task~A to $M=128$ and $512$ with datasets stratified over the masses (each mass recorded 4 and 2 times). The predictions were fixed before the runs: DIACRITIC rate $\le0.2$ bit during grasp and transport, $I(C;\mathrm{mass}\mid\Ob)\le0.05$ bit, closed-loop success $\ge0.9$. System identification at $K=16$ is a capacity-stress diagnostic ($\log_2K<\log_2M$ by construction); the capacity-relieved rows use $K=256$ and $K=1024$. Training uses $N=448$ episodes at every $M$, which covers all 128 masses at $M=128$ but only 352 of 512 at $M=512$; the $N=896$ row is a data-scale sensitivity control, and the $N=1000$ rows, in which every mass appears in training, are the coverage-controlled comparison. With full coverage, system identification at $K=1024$ becomes place-sufficient on all seeds (also on fresh recordings, Appendix~\ref{app:replication}), yet still stores $8.2$ bits and succeeds in closed loop on $0.23$ of episodes against $0.99$ for DIACRITIC: the separation is not a coverage artifact.

\begin{table}[htbp]\centering\footnotesize\setlength{\tabcolsep}{4pt}
\caption{Task A at $M=128$ and $512$ (8 seeds per row; 128 closed-loop episodes per seed).}
\label{tab:bigM}
\begin{tabularx}{\linewidth}{@{}l>{\raggedright\arraybackslash}Xcccc@{}}\toprule
$M$ & objective & sufficient (place) & grasp rate (bits) & $I(C;\mathrm{mass}\mid\Ob)$ & success \\\midrule
128 & ours, $K=16$, $\beta=10^{-3}$ & 8/8 & 0.03 & 0.01 & 0.99 \\
128 & ours, $K=16$, $\beta=0$ & 8/8 & 0.18 & 0.02 & 0.99 \\
128 & sys-ID, $K=16$ & 0/8 & 3.73 & 0.62 & 0.04 \\
128 & sys-ID, $K=256$ & 3/8 & 6.84 & 1.40 & 0.17 \\
512 & ours, $K=16$, $\beta=10^{-3}$ & 8/8 & 0.03 & 0.01 & 0.99 \\
512 & ours, $K=16$, $\beta=0$ & 8/8 & 0.10 & 0.02 & 0.98 \\
512 & sys-ID, $K=16$ & 0/8 & 3.70 & 0.05 & 0.06 \\
512 & sys-ID, $K=1024$ & 7/8 & 8.71 & 1.32 & 0.19 \\
512 & ours, $K=16$, $\beta=10^{-3}$, $N=896$ (data-scale control) & 8/8 & 0.01 & 0.00 & 1.00 \\
512 & ours, $K=16$, $\beta=10^{-3}$, $N=1000$ (every mass in training) & 8/8 & 0.01 & 0.01 & 0.99 \\
512 & sys-ID, $K=1024$, $N=1000$ (every mass in training) & 8/8 & 8.17 & 1.43 & 0.23 \\
\bottomrule\end{tabularx}
\end{table}

\paragraph{Non-zero behavioural memory at growing world complexity (readout task).}
Task~A has exact minimum $0$ during grasp because mass does not affect the expert's action choice. To test the non-zero counterpart of the world-complexity comparison, we retain the Franka grasp--transport--place task, attach the object kinematically so that mass is not re-revealed, and display the hidden class $\theta\in[M]$ on a 9-channel binary readout during three scan steps. The expert grasps according to the quartile of $\theta$ ($2$ bits, used $3$ steps after the readout) and places according to its half (\emph{readout-2}: $1$ bit, used $\approx20$ steps later) or octile (\emph{readout-3}: $8$ slots separated by at least $21$\,cm). These nested threshold classes admit the exact gap-to-transport profiles $2\to1\to0$ and $3\to3\to0$ at every $M\in\{32,128,512\}$; the solver certifies A2, A4 and transitivity, while $H(H_t\mid\Ob_t)=\log_2M+1$ bits during transport. During the first three side-approach steps, the end-effector pose reveals the grasp quartile, temporarily making the pending readout-2 slot conditionally deterministic; its conditional rate returns to $1$ bit once this side information disappears. This is temporary redundancy with the current observation, not expiration: the recurrent state must preserve the distinction until placement. The visibility convention accordingly uses unconditional class means for the nested factor. Every dataset contains $1024$ stratified episodes ($N=448$ for training), and an independent recording provides the replication set. Table~\ref{tab:readout} reports teacher-forced re-evaluation on both recordings and closed-loop evaluation over $128$ episodes per policy.

An analog single-channel display was tried first; with eight threshold classes even the full-history forecaster could not decode the slot (held-out error $0.28$--$0.51$), so we use the binary display.
On readout-2 the plain rate-penalised learner ($-$R, $K=16$, $\beta=10^{-3}$) is sufficient on $8/8$, $7/8$ and $8/8$ seeds at $M=32,128,512$, with first-gap rate $2.00$--$2.03$ bits (theory $2.00$) and transport rate $1.06$--$1.23$ (theory $1.00$). Forecast supervision reaches $8/8$ at every $M$ with rates $2.02$--$2.08$ and $1.09$--$1.27$ bits, and closed-loop success $0.95$, $0.90$ and $0.96$ ($-$R: $0.94$, $0.85$, $0.87$). The behavioral learners reproduce every verdict on the fresh recordings, with rate changes of at most $0.01$ bit. System identification stores $3.3$ bits at $K=16$ and $4.8$, $5.6$ and $7.5$ bits with the largest codebook tested ($K=256,256,1024$; $3.7$--$7.0$ bits about $\theta$), yet reaches only $0.19$--$0.26$ closed-loop success with slot accuracy $0.60$--$0.73$. Thus the behavioral rate remains fixed at a non-zero value while the system-identification rate grows with slope $0.67$ per bit of $\log_2M$.

Readout-3 probes the next representational scale. Forecast supervision reaches sufficiency on $6/8$, $5/8$ and $4/8$ seeds, identically on the fresh recordings, and the sufficient codes carry $3.03$--$3.17$ bits against the $3.00$-bit prediction. Closed-loop evaluation reveals the remaining occupancy gap: across all seeds, success is $0.15$--$0.28$ and slot accuracy $0.47$--$0.59$; among teacher-forced sufficient seeds, success is $0.30$--$0.48$ despite $0.96$ teacher-forced place accuracy. A $K=8$ carrier, which meets the cardinality lower bound $\max_{t,o}|\Gam_t|_o=8$, trains to sufficiency on $0/24$ seeds, whereas $K=16$ realizes the three-bit code on $15/24$ seeds. The unsupervised learner reaches that code on $6/24$ seeds; at $M=128$ and $512$, those sufficient seeds attain $0.84$ closed-loop success. These results establish three-bit offline attainability while locating reliable closed-loop learning at two bits.

\begin{table}[htbp]\centering\footnotesize\setlength{\tabcolsep}{4pt}
\caption{Readout task, 8 seeds per cell, $K=16$ unless noted (sys-ID largest $K$: 256, 256, 1024); each cell lists $M=32$ / $128$ / $512$. Rates are all-seed means except for the $K=16$ forecast row on readout-3, which averages sufficient seeds; closed-loop success is always averaged over all seeds. Behavioral-learner verdicts replicate on the fresh recordings; the largest-$K$ readout-2 sys-ID count at $M=512$ changes from $1/8$ to $0/8$. Theory is in brackets.}\label{tab:readout}
\begin{tabularx}{\linewidth}{@{}l>{\raggedright\arraybackslash}Xcccc@{}}\toprule
task & learner & sufficient & first-gap rate & transport rate & closed loop \\ \midrule
readout-2 & $-$R & 8/8, 7/8, 8/8 & 2.00, 2.02, 2.03 [2.00] & 1.10, 1.23, 1.06 [1.00] & 0.94, 0.85, 0.87 \\
readout-2 & forecast & 8/8, 8/8, 8/8 & 2.02, 2.08, 2.03 [2.00] & 1.11, 1.27, 1.09 [1.00] & 0.95, 0.90, 0.96 \\
readout-2 & sys-ID, largest $K$ & 4/8, 6/8, 1/8 & 4.82, 5.56, 7.49 & 4.49, 4.74, 7.93 & 0.25, 0.26, 0.19 \\
readout-3 & forecast & 6/8, 5/8, 4/8 & 3.03, 3.09, 3.17 [3.00] & 3.06, 3.01, 3.09 [3.00] & 0.25, 0.28, 0.15 \\
readout-3 & forecast, $K=8$ & 0/8, 0/8, 0/8 & 2.73, 2.92, 2.87 & 2.49, 2.46, 2.76 & 0.05, 0.05, 0.05 \\
readout-3 & sys-ID, largest $K$ & 6/8, 0/8, 0/8 & 4.83, 5.78, 7.94 & 4.78, 5.78, 8.31 & 0.32, 0.07, 0.03 \\ \bottomrule
\end{tabularx}
\end{table}

\paragraph{Separating the information and control costs of system identification.}
The sys-ID baseline decodes $\theta$ from $(O_t,C_t)$, forcing world and behavioral information through the same rate-penalized codebook. We separate these roles on Task~A at $M=32$ (all masses observed in training) and at the full-coverage $M=512$ cell ($N=1000$). A \emph{dual carrier} adds a recurrent hard-VQ code $C^{\rm id}_t$ ($K=256$ / $1024$) read only by the $\theta$ head, while the policy reads only $(O_t,C_t)$ with $K=16$; both codes are rate-penalized and share the observation and action encoders. We also test a shared carrier with the $\theta$ loss down-weighted to $0.1$, and a dual carrier that stops the identification gradient at the shared encoders. The dual models reach $0.90$--$0.97$ closed-loop success, compared with $0.07$--$0.52$ across the displayed shared-carrier controls. Their behavioral codes carry $0.28$--$0.49$ bit with at most $0.03$ bit about mass, while the identification codes carry $4.1$--$6.0$ bits, including $2.5$--$5.2$ bits about $\theta$. Down-weighting alone leaves $4.5$ / $6.3$ bits in the shared carrier and reaches $0.38$ / $0.52$ success; stop-gradient changes the dual result little. Independent recordings reproduce every behavioral sufficiency verdict and rate to within $0.04$ bit. Thus system identification retains a distinct, additive rate burden, whereas its control penalty is not a consequence of that rate alone: in these controls it appears when both objectives share one regularized carrier.

\begin{table}[htbp]\centering\footnotesize\setlength{\tabcolsep}{4pt}
\caption{Task A, sys-ID controls (8 seeds): transport rates of the behavioural code $C$ and of the identification code $C^{\rm id}$ (bits), mass information carried by $C$, $\theta$ information carried by $C^{\rm id}$, and closed-loop success (mean, min). ``shared'' = the $\theta$ head reads $C$ itself.}\label{tab:dual}
\begin{tabular}{@{}lccccc@{}}\toprule
configuration & $H(C\mid\Ob)$ & $I(C;\mathrm{mass}\mid\Ob)$ & $H(C^{\rm id}\mid\Ob)$ & $I(C^{\rm id};\theta\mid\Ob)$ & success \\ \midrule
\multicolumn{6}{@{}l}{\emph{$M=32$, $N=448$ (every mass in training)}} \\
DIACRITIC ($K{=}16$) & 0.34 & 0.01 & -- & -- & 0.99, 0.88 \\
shared, weight 1 ($K{=}16$) & 2.83 & 0.73 & -- & -- & 0.07, 0.05 \\
shared, weight 0.1 ($K{=}256$) & 4.52 & 0.80 & -- & -- & 0.38, 0.23 \\
dual carrier ($K^{\rm id}{=}256$) & 0.46 & 0.03 & 4.12 & 2.73 & 0.90, 0.84 \\
dual, stop-gradient & 0.29 & 0.00 & 4.35 & 2.45 & 0.93, 0.88 \\ \midrule
\multicolumn{6}{@{}l}{\emph{$M=512$, $N=1000$ (every mass in training)}} \\
DIACRITIC ($K{=}16$) & 0.25 & 0.01 & -- & -- & 0.99, 0.97 \\
shared, weight 1 ($K{=}1024$) & 6.67 & 0.61 & -- & -- & 0.23, 0.09 \\
shared, weight 0.1 ($K{=}1024$) & 6.25 & 0.70 & -- & -- & 0.52, 0.29 \\
dual carrier ($K^{\rm id}{=}1024$) & 0.49 & 0.02 & 5.98 & 5.18 & 0.96, 0.93 \\
dual, stop-gradient & 0.28 & 0.02 & 5.94 & 5.14 & 0.97, 0.92 \\ \bottomrule
\end{tabular}
\end{table}

Across $M=4$ to $512$ the capacity-relieved system-identification rate is $2.25, 3.66, 4.53, 5.22, 6.84, 8.71$ bits, with a fitted slope of $0.88$ in $\log_2 M$ and intercept $0.80$. At $M=128$ and $512$, mean closed-loop success is respectively $0.17$ and $0.19$, despite sufficiency on $3/8$ and $7/8$ seeds. The $M=512$ teacher-forced result is not robust to the evaluation sample: on fresh expert recordings the count falls to $4/8$, while mean placement $S_G$ changes from $0.923$ to $0.902$ (Appendix~\ref{app:replication}). More generally, the gate tests whether the code contains the discrete placement class under expert occupancy; it does not certify continuous-action accuracy or stability under policy-induced observations, and the present evaluations do not separate these two possible sources of the low closed-loop success.

\subsection{Signpost corridor: exact requirements and learned rates}\label{app:grid}
The A$'$ manipulation task satisfies (A4), whereas Task~A violates (A4) but remains transitive. To test whether the empirical phenomena depend on this design, we introduce a partially observed grid corridor that shares only the solver and learner with A$'$. Its $W=1$ setting is in the exact regime; the drift variants $W>1$ are non-transitive, and matching lower and upper bounds certify the rates of the recorded $W=3,5,7$ instances (Proposition~\ref{prop:coloring}). The agent moves from left to right through $\log_2 M$ signpost cells, each of which reveals one bit of a hidden goal code $\theta\in[M]$. It then traverses a wide hall of gap$_1$ cells, reaches a junction where the expert turns according to $\beta_1(\theta)$, traverses a second hall of gap$_2$ cells, and reaches a second junction governed by $\beta_2(\theta)$. Observations are symbolic and indicate the start texture, sign bit, hall identity and lateral cell, or junction identity; actions are forward, sidestep, and turn. In the first cell of each hall, the expert sidesteps left or right uniformly at random, providing behaviorally equivalent action randomness.

A hidden lateral \emph{drift} $w\in[W]$, displayed on the start sign, shifts the agent sideways at every hall step. The drift changes the transitions experienced by the agent and is revealed again by its observed lateral position, but never affects the expert's action. It therefore serves as an agent-centric, behaviorally irrelevant nuisance, analogous to a current rather than a mass. At $W=1$, the solver certifies (A2), (A4), and transitivity at every step and returns the staircase $H(\Gam_t\mid\Ob_t)=2\to2\to1$ across hall 1, junction 1, and hall 2, compared with $H(H_t\mid\Ob_t)=3$--$7$ bits. At $W>1$, (A4) fails in the halls. The strong-congruence bracket is $[0,\,2+\log_2 W]$; the finite-instance certificates tighten the hall requirements to exactly $2$ and $1$ bits. We train with the toy trainer using raw VQ, $K=16$, 6k steps, $\beta=0.01$, and a $\tanh$-bounded residual state to prevent codebook collapse. Each cell contains 8 seeds, and the multi-step-inverse and system-identification heads match those in Appendix~\ref{app:controls}.

\begin{figure}[htbp]\centering\includegraphics[width=\linewidth]{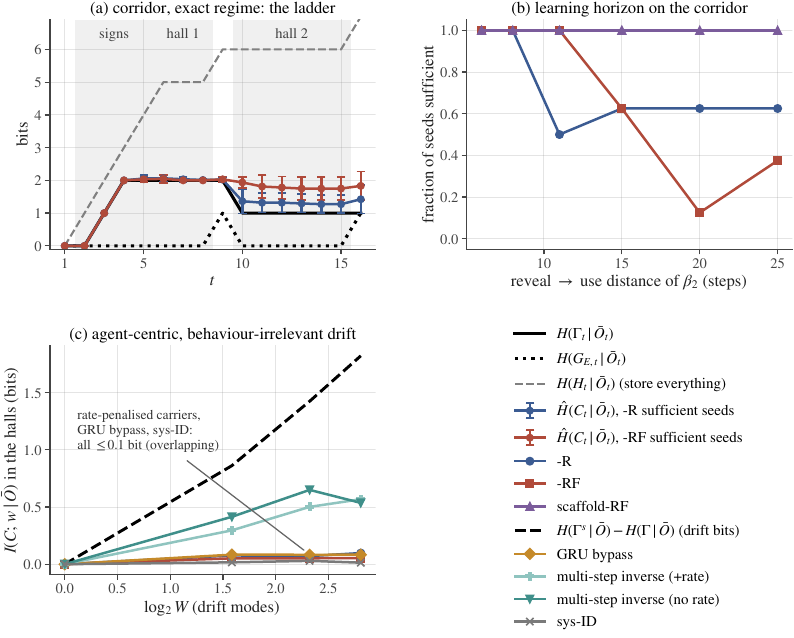}
\caption{\textbf{Signpost corridor.} (a) In the exact regime at $M=16$, sufficient seeds carry $2.00$ bits through hall 1 while $H(G_{E,t}\mid\Ob_t)=0$, then reduce their rate after junction 1; the history contains 5 bits. Expiration is partial (1.2--1.8 compared with 1.0). (b) Fraction of sufficient seeds as a function of the reveal$\to$use distance for $\beta_2$ (hall-2 length 1--20). The plain realizations decay with distance (logit slope $-0.11$/step, $p=0.01$ with configuration-clustered errors), whereas the annealed scaffold does not. (c) Drift information retained in the halls as a function of $\log_2 W$. The multi-step-inverse memory carries $0.3$--$0.7$ bit about the drift, with or without the rate term, whereas DIACRITIC and system identification carry $\le0.08$ bit. The dashed line denotes the additional rate assigned by the strong congruence.}
\label{fig:grid}
\end{figure}

\textbf{Staircase and world complexity.} Every sufficient seed carries $2.00$--$2.07$ bits in hall 1 at $M=4$, $16$, and $64$, while the history grows from 3 to 7 bits. The sufficient-seed counts are $7/8$, $4/8$, and $6/8$ for $-$R; $8/8$, $8/8$, and $6/8$ for $-$RF; and $8/8$, $7/8$, and $6/8$ for $\beta=0$. System identification is sufficient at $M=4$, where $\theta$ \emph{is} the join, but reaches $0/8$ at $M=16$ and $64$ while retaining $2.2$--$2.5$ bits about $\theta$. The multi-step-inverse memory reaches $0/8$ at every $M$, with junction error $0.5$. Expiration after junction 1 is partial in this domain: $-$R retains $1.17$--$1.36$ bits relative to the $1.00$ boundary. Unlike on A$'$, increasing $\beta$ within the admissible range does not reduce this surplus ($\beta=0$: $1.28$--$1.30$). We therefore treat the discrepancy as a limitation of the realization rather than of the target.

\textbf{Learning horizon.} For hall-2 lengths $1,3,6,10,15,20$, corresponding to distances 6--25, $-$R reaches sufficiency on $8,8,4,5,5,5$ of 8 seeds, and $-$RF reaches $8,8,8,5,1,3$. Increasing the length of hall 1 to 8 and 15 cells reduces $-$R to $2/8$ and $0/8$. We pool the 208 runs at $W=1$, comprising 416 run$\times$class rows and 24 configurations. For the plain realizations, the distance coefficient is $-0.11$/step ($p=0.01$, configuration bootstrap CI $[-0.20,-0.05]$), matching the sign and order of the A$'$ coefficient ($-0.15$). The annealed scaffold, which only mitigates the barrier on A$'$, reaches $8/8$ at every corridor distance up to 25, with slope $0$. The recurrence of the distance effect outside manipulation supports a temporal optimization interpretation, while the differing benefit of the scaffold shows that its mechanism is domain-dependent.

\textbf{Layer-2 nuisance.} At $W=3,5,7$ drift modes, DIACRITIC retains $\le0.08$ bit about the drift and $2.00$--$2.08$ bits in hall 1, with no dependence on $W$; the strong congruence would instead assign $2.86$--$3.82$ bits. The multi-step-inverse memory retains $0.30/0.50/0.47$ bit with the rate term and $0.54/0.71/0.48$ without it, because inferring a sidestep from the observed lateral displacement requires the drift. This memory never reaches sufficiency. System identification retains $2.4$--$2.6$ bits about $\theta$ and $\le0.04$ bit about $w$. These results separate the targets empirically: a control-endogenous objective retains an agent-centric quantity ignored by the expert, whereas the behavioral objective need not. The objectives become distinguishable only when behaviorally equivalent action randomness makes the inverse problem depend on the nuisance, which is why Task A with a deterministic expert cannot exhibit this separation (Appendix~\ref{app:controls}).

\textbf{Matched-distortion controls.} The inverse-only memory fails in the corridor even with $K=64$ codes, reaching $0/8$ at both $W=1$ and $W=5$. At $W=5$, it retains $\approx1.0$ bit about the drift and $2.07$ bits in hall 1, but its junction error remains at chance. Its failure is therefore not caused by codebook capacity. When trained \emph{jointly} with imitation and the behavioral code rate term, the same head has no measurable effect: sufficiency is $4/8$ at $W=1$ and $6/8$ at $W=5$, matching $-$R in the corresponding cells, and drift information is $0.09$--$0.12$ bit, compared with $0.30$--$0.70$ for the inverse objective alone. Thus, the inverse objective retains drift only in the absence of pressure to remove it; it does not enforce minimality, and the rate term removes this information at negligible cost to the inverse loss.

With $K=64$, system identification reaches sufficiency on $4/8$ seeds, compared with $0/8$ at $K=16$, at a cost of $3.9$ bits when $\theta$ contains 4 bits. DIACRITIC remains unchanged at $K=64$, reaching $8/8$ at $2.03$ bits. Thus capacity contributes to the $K=16$ system-identification failure but does not fully resolve it; the inverse-only objective remains at $0/8$ even with $K=64$ (preceding paragraph). The continuous-carrier ablation, implemented as a GRU bypass, requires a larger learning rate in this domain: it reaches $0$--$2/8$ at $3\!\times\!10^{-4}$ and $7/8$ at $3\!\times\!10^{-3}$. When successful, its readout code carries $1.5$ bits where $\Gam$ requires $2.0$, reproducing the undercount observed on the toy and A$'$. A code that is not the sole memory carrier therefore does not yield a valid memory rate.

\subsection{Plain learned code trajectories}\label{app:plain-trajectories}

\begin{figure}[htbp]\centering\includegraphics[width=\linewidth]{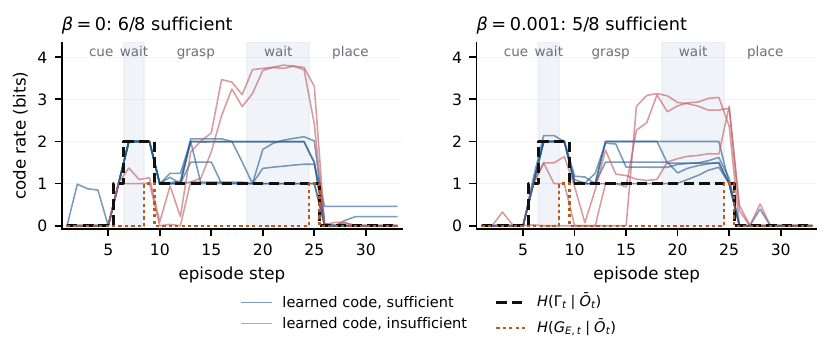}
\caption{\textbf{Every plain seed, with sufficiency checked first.} A$'$ at gap 6, $N=1152$, $\beta=0$ and $10^{-3}$, eight seeds per cell. Blue trajectories pass the full gate of \S\ref{sec:method}; red trajectories do not. The black curve is the certified behavioral memory rate and the orange curve the instantaneous requirement. Rates use expert occupancy. A low code rate on a failed seed is not evidence of compression at preserved behavior.}\label{fig:fig3}
\end{figure}

This view complements Figure~\ref{fig:memory-profile}: it retains the variation across all plain seeds, while the main figure tests a realization whose behavioral decoder has the same side information as the theoretical model.

\FloatBarrier
\needspace{8\baselineskip}\section{Learning compact behavioral memory}\label{app:learning}
The following experiments concern acquisition of the independently defined target. Counts use the full gate unless a relaxed or class-specific diagnostic is explicitly identified. Rates conditional on passing and all-seed control outcomes are kept separate.

\subsection{Event-agnostic future-behavior supervision}\label{app:generic}
\paragraph{Objective.} For every training episode and step $t$ we draw one offset $j\sim U\{1,\dots,T-1\}$ and a training-only head regresses, from $(g_o(O_t),C_t)$ and an embedding of $j$, the forecast of $A_{t+j}$ produced by a causal Transformer that was trained from scratch on the same recorded episodes to predict all future actions from the history prefix (offsets beyond the end of the episode are masked; nothing else is). No reveal step, decision step, divergent action or latent label enters the construction; before the class is revealed the forecaster outputs the conditional mean. The auxiliary weight is $1$ until $60\%$ of the $10^4$ training steps and decreases linearly to $0$ at $80\%$; the last $20\%$ optimize imitation and rate only. Everything else is the frozen configuration of Appendix~\ref{app:distill} ($K=16$, $\beta=10^{-3}$, $N=448$). The \emph{task-informed} forecast of Appendix~\ref{app:distill} instead predicts only the first grasp and the first place action, supervised from the last reveal step and masked once used.

\begin{table}[htbp]\centering\scriptsize\setlength{\tabcolsep}{2.5pt}
\caption{Forecast supervision on A$'$ (state observations), 8 seeds per cell. Top: sufficient seeds under the full-trajectory gate (in parentheses: the same models re-evaluated on an independent recording). Bottom: first-gap $\to$ second-gap rate of the sufficient seeds (bits; theory $2.00\to1.00$) / closed-loop success over all seeds (128 episodes per policy). EA = event-agnostic random-offset forecast; TI = task-informed forecast.}
\label{tab:generic}
\begin{tabular}{@{}lccccc@{}}\toprule
sufficient seeds & gap 6 & gap 10 & gap 20 & $M=16$ & $M=32$ \\\midrule
EA, anneal $60\%$ (frozen) & $8/8$ ($8/8$) & $7/8$ ($7/8$) & $8/8$ ($8/8$) & $6/8$ ($6/8$) & $7/8$ ($7/8$) \\
EA, anneal $40\%$ & --- & --- & $7/8$ ($7/8$) & --- & $7/8$ ($7/8$) \\
EA, not annealed & $6/8$ ($6/8$) & $8/8$ ($7/8$) & $7/8$ ($7/8$) & $8/8$ ($8/8$) & $6/8$ ($6/8$) \\
TI & $8/8$ (---) & $8/8$ ($8/8$) & $8/8$ ($8/8$) & $8/8$ ($8/8$) & $5/8$ ($5/8$) \\
teacher state & $2/8$ & $0/8$ & $0/8$ & $0/8$ & $1/8$ \\\midrule
rate / closed loop & gap 6 & gap 10 & gap 20 & $M=16$ & $M=32$ \\\midrule
EA, anneal $60\%$ (frozen) & $2.00\to1.15$ / $.99$ & $2.00\to1.26$ / $.87$ & $2.00\to1.17$ / $.90$ & $2.00\to1.09$ / $.79$ & $2.04\to1.21$ / $.93$ \\
EA, anneal $40\%$ & --- & --- & $2.00\to1.21$ / $.89$ & --- & $2.03\to1.23$ / $.82$ \\
EA, not annealed & $2.00\to1.51$ / $.89$ & $2.00\to1.92$ / $.91$ & $2.02\to2.14$ / $.90$ & $2.03\to1.29$ / $.98$ & $2.03\to1.63$ / $.92$ \\
TI & $2.01\to1.07$ / $.95$ & $2.05\to1.06$ / $.89$ & $2.09\to1.10$ / $.92$ & $2.08\to1.09$ / $.84$ & $2.13\to1.11$ / $.60$ \\
\bottomrule\end{tabular}
\end{table}

On the independent recordings every seed-wise verdict of the frozen configuration is reproduced ($36$ pass$\to$pass, $4$ fail$\to$fail, no switches); for the task-informed forecast $30$ of the $32$ re-evaluated verdicts agree, with one switch in each direction, and over all $520$ models re-evaluated in this appendix and Appendices~\ref{app:distill} and~\ref{app:weigh} the agreement is $96\%$. Overall sufficiency is similar for the two objectives ($36/40$ and $37/40$) with setting-dependent differences ($M=16$: $6/8$ against $8/8$; $M=32$: $7/8$ against $5/8$, closed loop $0.93$ against $0.60$), while the task-informed target is more rate-efficient after first use ($1.06$--$1.11$ against $1.09$--$1.26$ bits). Across these cells, policies that pass the full-trajectory gate also exhibit high closed-loop success ($0.90$--$1.00$; e.g.\ $M=16$: $0.79$ over all seeds, $1.00$ over sufficient seeds); we report this as an association, not a causal claim. A variant that sums the loss over all future offsets is not robust ($3/8$ at gap 20 at weight $1$, $1/8$ on the independent recording; $7/8$ only with weight $0.1$ and annealing) and is not used.

\paragraph{Pixels.} Table~\ref{tab:pxgeneric} repeats the comparison from pixels; the forecaster reads the same images. The $60\%$ schedule was the best of three tried on seeds 0--7 at gap~20; seeds 8--15 were run afterwards with nothing changed. On those held-out seeds plain training is sufficient on $0/8$ (closed-loop success $0.08$, slot accuracy $0.55$), the event-agnostic objective on $6/8$ ($0.57$, $0.84$) and the task-informed forecast on $7/8$ ($0.49$, $0.82$); the $40\%$ schedule, fixed before any pixel run, gives $6/8$ ($0.65$, $0.88$). Grasp-side accuracy, whose class is used two steps after the reveal, is $0.97$--$1.00$ for every learner: the differences arise only at the slot chosen twenty steps later. Among offline-sufficient seeds the slot accuracy is $0.97$ (event-agnostic) and $0.93$ (task-informed), against $0.53$--$0.59$ for insufficient seeds, so the offline gate predicts the closed-loop memory decision; task success is further limited by pixel control precision (action error $2.6$--$6.0$\,cm).

\begin{table}[htbp]\centering\footnotesize\setlength{\tabcolsep}{4pt}
\caption{Pixel A$'$: sufficient seeds (full-trajectory gate), rates of the sufficient seeds (theory $2.00\to1.00$), and closed loop over all seeds (128 episodes per policy).}
\label{tab:pxgeneric}
\begin{tabular}{@{}llcccc@{}}\toprule
horizon & learner & sufficient & rate (bits) & success & slot accuracy \\\midrule
gap 6 & plain $-$R & $2/8$ & $1.99\to1.00$ ($n{=}2$) & $0.17$ & $0.60$ \\
gap 6 & task-informed forecast & $7/8$ & $1.99\to1.12$ & $0.25$ & $0.94$ \\\midrule
gap 20 & plain $-$R & $1/16$ & $1.99\to1.52$ ($n{=}1$) & $0.14$ & $0.56$ \\
gap 20 & event-agnostic, not annealed & $7/8$ & $2.15\to2.01$ & $0.40$ & $0.80$ \\
gap 20 & event-agnostic, annealed from $40\%$ & $12/16$ & $2.01\to1.13$ & $0.51$ & $0.82$ \\
gap 20 & event-agnostic, annealed from $60\%$ & $13/16$ & $2.00\to1.13$ & $0.65$ & $0.89$ \\
gap 20 & task-informed forecast & $15/16$ & $2.00\to1.00$ & $0.61$ & $0.91$ \\
\bottomrule\end{tabular}
\end{table}

\paragraph{What the un-annealed objective keeps.} In the second gap only $\beta_2$ is required and the symbolic observation is uninformative, so $I(C_t;\beta_1\mid\beta_2)$ measures the expired class still carried by the code. It is at most $0.02$ bit for every learner, annealed or not, from states and from pixels, while $I(C_t;\beta_2\mid\beta_1)=0.98$--$1.00$ bit. The surplus $\Hc(C_t\mid\Gam_t,\Ob_t)$ of the un-annealed objective is $1.02$ bits from pixels and $1.16$ from states, compared with $0.13$ and $0.17$ after annealing and $0.00$ for the task-informed forecast.

The scripted expert moves along linear segments whose per-step displacement is fixed by the starting pose. During the second gap this is the grasp pose, set by the object's initial position (uniform within $\pm2$\,cm), which is no longer visible and changes the displacement by up to $40\%$. Conditioned on $\beta_2$, the un-annealed code carries $0.56$ bit about the initial $x$-position quartile ($0.01$--$0.03$ after annealing) and at most $0.03$ bit about the $y$ position, hover nuisance, distractor, or mass. This information can improve an open-loop action forecast even though the symbolic behavioral target does not require it.

At the midpoint of the second gap ($t=28$), a separate five-fold cross-validated analysis gives conservative lower bounds on the information explained by these features. Relative to the $1.47$-bit surplus at that step, initial position explains at least $0.38$ bit ($26\%$), increasing to $0.51$ bit ($35\%$) when noisy observation history is included. These are lower bounds, not a limit on the explainable share; they use a single-step measurement rather than the phase average reported above.

\paragraph{Intervening on initial-position variability.} We re-recorded A$'$ gap 20 without the $\pm2$\,cm object-position jitter and retrained the un-annealed objective (8 seeds; prediction fixed beforehand: surplus below $0.6$ bit). All $8/8$ seeds are sufficient, with $2.00\to1.23$ bits, second-gap surplus $0.23$ bit, and closed-loop success $0.94$. The original jittered task gives $7/8$ sufficient seeds and $1.16$ bits of surplus. Removing this source of task variation therefore substantially reduces predictive surplus under retraining, supporting initial-position variability as an important driver. The $0.93$-bit difference compares separately trained policies; it is not a decomposition of the position information stored by the original code. These results help explain excess rate under open-loop prediction (Appendix~\ref{app:e5}) while leaving its full content unresolved.

\paragraph{Dependence on the forecaster (A$'$ gap 20, task-informed target).} Regressing the recorded future actions directly, without any forecaster, is sufficient on $7/8$ seeds (all-seed closed-loop success $0.82$). Forecasters trained for $300$, $100$ and $30$ steps instead of $3000$ (held-out class error of their pending-action forecast $0.00$, $0.08$, $0.25$) give $7/8$, $8/8$ and $4/8$ on the independent recording (all-seed closed loop $0.75$, $0.78$, $0.67$). Forecasts that are \emph{consistently wrong}---for $5$, $10$ or $20\%$ of the training episodes the target is the forecast of an episode of a different behavioral class---give $2/8$, $0/8$ and $0/8$ on the clean independent recording, with closed-loop slot accuracy $0.91$, $0.77$ and $0.62$: the student tolerates an imprecise forecaster but distils its systematic errors, and the entropy gate fails once about $5\%$ of the episodes receive a wrong code. With recorded actions as targets the all-offset event-agnostic objective fails ($0/8$): before the reveal the recorded future actions are unpredictable, whereas the forecaster supplies their conditional mean.

\subsection{Task-informed supervision and teacher-state controls}\label{app:distill}
This appendix documents the \emph{task-informed} forecast and the teacher-state control; the event-agnostic objective of the main text is documented in Appendix~\ref{app:generic}. The forecaster is a causal Transformer (two layers, width 128, $d=32$) trained from scratch on the recorded training episodes for 3000 steps to output, at every step $t$ from the last reveal step onward, the expert's action at the first grasp step and at the first place step, each masked once that step is reached, so that its target at $t$ is exactly the pending class-dependent behavior; it uses only recorded observations and actions. Its held-out error is $0.000$ (per-dimension MSE of standardized actions) on every dataset. The student is the plain $K=16$ DIACRITIC ($\beta=10^{-3}$) with one additional training-time head that regresses the forecaster's output from $(g_o(O_t),C_t)$ under the same mask (weight 1); the head and the forecaster are discarded after training, and all rates and rollouts are those of the student alone. The teacher-state control replaces the target by the frozen full-history Transformer baseline's pre-quantization state at $t$ (the full-history baseline of \S\ref{sec:f4}, same seed and data).

\begin{table}[htbp]\centering\footnotesize\setlength{\tabcolsep}{3.5pt}
\caption{Distillation controls on A$'$ and P-I (8 seeds per cell; full-trajectory gate of \S\ref{sec:protocol}; the event-agnostic objective is in Table~\ref{tab:generic}). Top block: seeds sufficient. Bottom block: rates and closed-loop success of the sufficient seeds under the forecast target.}
\label{tab:distill}
\begin{tabular}{lccccc}\toprule
target & gap 6 & gap 10 & gap 20 & $M=16$ & $M=32$ \\\midrule
task-informed forecast of pending behavior & 8/8 & 8/8 & 8/8 & 8/8 & 5/8 \\
teacher's internal state & 2/8 & 0/8 & 0/8 & 0/8 & 1/8 \\
unsupervised (best variant) & 3/8 & 2/8 & 1/8 & 0/8 & 0/8 \\\midrule
forecast: rate before use (theory 2.00) & 2.01 & 2.05 & 2.09 & 2.08 & 2.13 \\
forecast: rate after use (theory 1.00) & 1.07 & 1.06 & 1.10 & 1.09 & 1.11 \\
forecast: closed-loop success & 0.95 & 0.89 & 0.92 & 0.84 & 0.73 \\
\bottomrule\end{tabular}
\end{table}

\paragraph{The 4-bit join.} On tier~16 the complete anticipatory join requires $\log_2(4\cdot4)=4$ bits under the uniform model; its empirical test-split entropy is $3.97$ bits. Under the gate of \S\ref{sec:box} task-informed forecast supervision is sufficient on $0/8$ seeds at $K=16$ and $2/8$ at $K=64$; a relaxed gate that inspects only the second gap and the place step is passed by $3/8$ and $5/8$, but it tests only the two-bit remainder after the grasp (Appendix~\ref{app:gate}). The two successful $K=64$ seeds carry $4.09$ and $4.21$ bits in the first gap and attain closed-loop success $0.45$ and $0.46$. Forecast supervision thus yields two realizations of the four-bit representation when the codebook has slack, but not reliable learning at this scale; the privileged-head control likewise reaches only $1/8$ at $K=128$ (Appendix~\ref{app:controls}).

\subsection{Future-behavior targets without observation side information}\label{app:e5}
The decision-centric targets of Table~\ref{tab:targets} keep a history distinction when some future decision depends on it. Their fixed-demonstrator form is an open-loop decoder that must predict the future action sequence from $(C_t,O_t)$ alone, without the future observations and actions that the BFS decoder of \S\ref{sec:method} receives; a world-predictive form additionally predicts the future observations. We train both on the gap toy with $\beta_2$ re-revealed one step before use (where $\Gam$ needs $1$ bit in the first gap and $0$ afterwards while the strong congruence keeps $0.67$ bit), on the corridor with drift $W\in\{1,5\}$, and, for the action-only form, on Task~A and A$'$. On the re-reveal toy the action-only target carries $2.06$ bits in the first gap and $1.29$ after use ($\Gam$: $1.00$ and $0$): it pays for the re-revealed class because a future action depends on it, which is precisely the side-information term of Theorem~\ref{thm:exact-trans}; the observation-and-action target carries $1.75$ and $0.85$. On the same toy, the side-information decoder of \S\ref{sec:method} ($-$RF) lands at $1.12$ and $0.08$ bits. On the corridor with $W=5$ drift modes the two forms separate as the theory predicts: the action-only target retains $0.01$ bit about the drift, as DIACRITIC does ($0.03$), whereas the observation-predictive target retains $0.85$ bit, since the drift changes future observations but not future actions. On the robot the action-only target is sufficient on $8/8$ seeds of every Task~A dataset and of A$'$, but carries $1.61$--$1.86$ bits during grasp and transport on Task~A (DIACRITIC: $0.03$), with only $I(C;\mathrm{mass}\mid\Ob)=0.03$ bit about mass. It carries $2.16$ bits after use on A$'$ ($\Gam$: $1.00$), with closed-loop success $0.75$--$0.87$. Thus, an open-loop future decoder can charge the code for trajectory detail that future observations would supply.

\paragraph{Label density.} The forecast target need not be dense. On A$'$ gap~20 (reveal$\to$place distance $33$), applying the distillation loss only every second or fourth step after the last reveal ($12$ and $6$ supervised (step, block) pairs per episode instead of $23$) gives $8/8$ and $7/8$ full-gate sufficiency, respectively, with full-gate first-gap rates $2.05$ and $2.05$ bits and all-seed closed-loop success $0.88$ and $0.84$ ($0.92$ with every step). Behavior-aligned long-range supervision is therefore effective without per-step labels.

\subsection{Other training interventions}\label{app:negative}

\begin{table}[htbp]\centering\footnotesize\setlength{\tabcolsep}{3pt}\renewcommand{\arraystretch}{1.15}
\caption{Sufficiency and minimality on A$'$ (gap 6, $N=1152$ unless noted). Minimality is the post-use rate against the $1.00$-bit boundary. The variational KL bound of the continuous model is not comparable to hard-code entropy and is therefore not read as a minimality result. Appendices~\ref{app:negative} and~\ref{app:controls} report additional variants.}
\label{tab:quadrant}
\begin{tabularx}{\linewidth}{@{}>{\raggedright\arraybackslash}p{1.55in}>{\centering\arraybackslash}p{0.62in}LL>{\centering\arraybackslash}p{0.55in}@{}}\toprule
variant & sufficient & post-use rate (bits) & rate measurement & success \\\midrule
DIACRITIC ($-$R, $\beta=10^{-3}$) & \checkmark & $1.38$ (near) & \checkmark & $0.97$ \\
scaffold $\to$ DIACRITIC ($N\approx4$k) & \checkmark & $0.99$--$1.05$ on $5/6$ (strongest) & \checkmark\ ($\alpha=0$ at evaluation) & $0.87$--$1.00$ \\
$\beta=0$ & \checkmark & $1.53$ (redundant) & \checkmark & $0.96$ \\
continuous bottleneck & possible & KL bound $6.9$; nuisance $0.32$ & upper bound only & --- \\
GRU bypass & \checkmark & $1.06$ (apparent under-rate) & $\times$\ (not the sole carrier) & $0.97$--$1.00$ \\
system identification & budget-dep. & $2.2$--$3.2$ (learned rate) & \checkmark & $0.06$--$0.79$ \\
\bottomrule\end{tabularx}
\end{table}

We report training interventions that do not remove the barrier in \S\ref{sec:f4}. Unless noted otherwise, all experiments use A$'$ with 8 seeds. \emph{Future-sufficiency-driven code refinement} applies the solver's partition-refinement principle as a learning rule by splitting the code whose members have the least consistent futures. It does not change the number of sufficient seeds at any gap on the toy or on A$'$ (Fisher exact $p=1.0$ in every cell; pooled coefficient $+0.35$, $p=0.28$). A \emph{gap curriculum} warm-starts gap 10/20 and $M=16/32$ from the gap-6 scaffold model with the same seed. It yields $2/8, 1/8, 2/8, 0/8$ and a pooled coefficient of $-1.6$ ($p=3\!\times\!10^{-3}$), indicating worse performance. \emph{Doubling the training budget} to 20k steps leaves gap 20 at $0/8$ and $M=16$ at $0/8$. On the toy, a \emph{continuous Gaussian bottleneck} reaches sufficiency only at $\beta=0.003$, with a KL bound of $6.9$ bits and $0.32$ bit of retained nuisance information; this bound is not comparable to hard-code entropy (Table~\ref{tab:quadrant}). A \emph{non-persistent code} reaches $0/8$. An \emph{unconditional prior} is indistinguishable from the conditional prior in every family: $5/8$ versus $5/8$ on A$'$, and $8/8, 6/8, 4/8$ versus $7/8, 4/8, 6/8$ in the corridor at $M=4,16,64$. Finally, \emph{BFS} has a modest positive coefficient in the pooled regression ($+0.64$, $p=0.02$) but does not remove the barrier.

\subsection{Codebook size between the cardinality bound and learnability}\label{app:ksweep}
On readout-3 the exact requirement is three bits and Corollary~\ref{cor:codebook} requires $K\ge8$. With task-informed forecast supervision, sufficient-seed counts increase monotonically across the tested codebooks: $0$, $1$, $5$, $15$, and $23$ of $24$ runs (three values of $M$, eight seeds each) at $K=8$, $10$, $12$, $16$, and $24$. Sufficient seeds average $2.99$--$3.14$ bits; all-seed closed-loop success is $0.07$, $0.10$, and $0.60$ at $K=10$, $12$, and $24$. Some runs succeed with modest slack, and $K=24$ yields the highest observed reliability, $23/24$. The sweep distinguishes the representational state-count bound from reliable acquisition under this optimizer and training configuration (\S\ref{sec:f4}); it does not establish that smaller feasible codebooks are unreachable.

\subsection{Observed rate--fidelity trade-offs}\label{app:rd}

\begin{figure}[htbp]\centering\includegraphics[width=\linewidth]{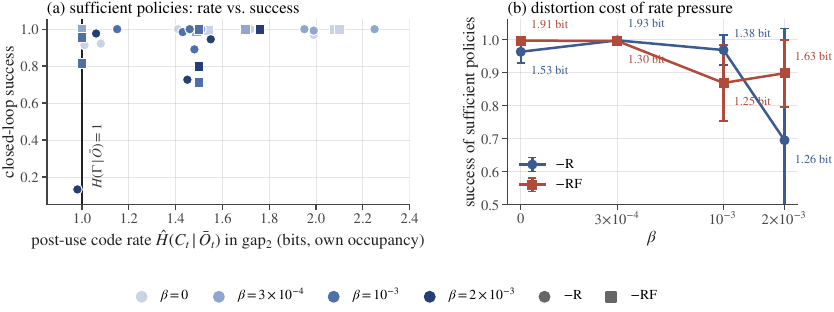}
\caption{\textbf{Empirical rate--distortion in the robot environment} (80 policies, 128 closed-loop episodes each). Left: post-use code rate as a function of closed-loop success among policies that reach sufficiency; many lie exactly at $H(\Gam\mid\Ob)=1$. Right: success of sufficient policies as a function of $\beta$.}
\label{fig:rd}
\end{figure}

As the post-use rate approaches $H(\Gam\mid\Ob)=1$, closed-loop success for sufficient plain $-$R policies decreases from $0.96$--$1.00$ at $\beta\le3\!\times\!10^{-4}$ to $0.70$--$0.90$ at $\beta=2\!\times\!10^{-3}$. This trade-off motivates fixing $\beta$ under a $D\approx0$ constraint before inspecting any rate (\S\ref{sec:protocol}). Across 448 models evaluated in closed loop, teacher-forced and own-occupancy sufficiency agree on $31/32$ members of the reference set. Sufficient policies succeed on $0.70$--$1.00$ of episodes, compared with $0.12$--$0.45$ for insufficient policies. The full pooled regression from \S\ref{sec:f4}, including three specifications, run- and configuration-clustered errors, and cluster bootstraps, and the corridor regression are reproduced by the released code.

\subsection{Sufficiency counts and temporal-distance statistics}\label{app:stats}
Sufficiency counts are reported per 8- or 16-seed cell. Wilson 95\% intervals for 8 seeds are $[0,0.32]$ for $0/8$, $[0.22,0.79]$ for $4/8$ and $[0.68,1]$ for $8/8$. Among 45 pairwise contrasts in the refinement, curriculum, coverage, $\beta$, and BFS sweeps, only two reach Fisher $p<0.05$, both involving the scaffold at gap 6 ($5/8$ vs $0/8$). The remaining differences in those sweeps are not established; this comparison set does not include the future-behavior supervision experiments.

The controlled distance sweeps of \S\ref{sec:f4} use 16 seeds per cell for gaps $6/10/20$ (distances 19/23/33 for the place class). Their outcome is class-specific sufficiency at use: placement counts are $1,3,1$ of 16 for $-$R, $5,2,0$ for $-$RF and $10,7,0$ for scaffold-RF, with grasp-class counts of $42$--$45$ of 48 at distance 3. These are not the full-trajectory gate counts used for the frozen supervision comparison. The gap$_1$ sweeps (distances 11 and 21) have 8 seeds per cell. Per-learner logistic slopes, their cluster-robust intervals, the interaction test and threshold sensitivity ($0.8$ and $0.95$ leave every count unchanged except one grasp cell) are in the released analysis code.

\subsection{Large hidden-mode counts on P-I}\label{app:compute}
\textbf{P-I on A$'$.} Figure~\ref{fig:pI} presents the exact-regime counterpart of the world-complexity comparison. The probe channel reveals the full $\theta\in[M]$, while each phase contains only two behavioral classes. For $M$ up to 32, every seed retains a learned rate of $1$--$2$ bits; successful seeds carry exactly $2.00$ bits, and the slope with respect to $\log_2 M$ is $0$. On the same data, the system-identification objective increases to $3$ bits, with a slope of $+0.89$ per bit of $\log_2 M$, and loses the behavioral memory. The unsupervised variants reach $0/8$ sufficient seeds at $M\ge16$, illustrating the optimization barrier in \S\ref{sec:f4} when many bits are revealed but few are behaviorally relevant; forecast supervision is the exception reported there.

\textbf{P-I at $M=128$ and $512$ with forecast supervision.} We freeze the forecast-distilled configuration selected at $M\le32$ ($-$R, $\beta=10^{-3}$, $K=16$, $N=448$; stratified data). On the training recordings, the all-seed first-gap means are $2.36$ and $2.63$ bits at $M=128$ and $512$; among late-gate sufficient seeds they are $2.17$ and $2.24$ bits (theory $2.00$). The corresponding all-seed means on fresh recordings are $2.48$ and $2.72$.

Same-task system identification stores $3.4$--$3.5$ bits at $K=16$ and $6.45$ and $8.72$ bits at $K=256$ and $1024$ ($4.89$ and $7.72$ bits about $\theta$), is sufficient on $0/8$ seeds, and succeeds in closed loop on at most $0.18$ of episodes.

Reliability nevertheless degrades with $M$: late-phase sufficiency is $8/8$, $6/8$, $5/8$ and $3/8$ at $M=16$, $32$, $128$ and $512$ (full-trajectory $4/8$ and $3/8$ at the last two). Some verdicts change on fresh recordings, including $5\to3$ at $M=128$. Training on all 512 modes ($N=1000$) yields only $2/8$ late-gate and $0/8$ full-trajectory seeds on the training recordings, and $0/8$ late-gate seeds on fresh recordings. Thus, conditional on successful realization, the behavioral rate remains near two bits as $M$ grows, but the all-seed mean rises modestly and reliability deteriorates. We interpret this as a rate separation from world identification and as an optimization scale boundary, not as reliable flat scaling.

\begin{figure}[htbp]\centering\includegraphics[width=\linewidth]{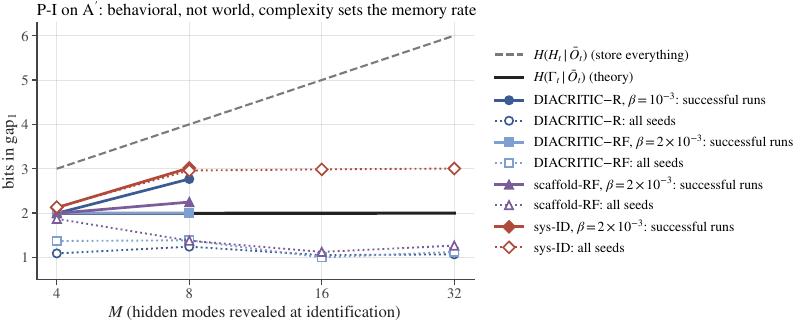}\caption{A$'$ with the full $\theta$ revealed. For $M\le32$, successful seeds remain at 2.00 bits, the all-seed average is flat over the plotted range, and the system-identification rate increases with $M$. Large-$M$ results and their reliability boundary are reported in the text.}\label{fig:pI}\end{figure}

\FloatBarrier
\needspace{8\baselineskip}\section{Additional domains and limits of transfer}\label{app:domains}
These domains test parts of the information structure and the learning surrogate beyond the main comparisons. Offline sufficiency, memory-dependent decisions, and complete closed-loop task success are reported separately.

\subsection{Pixel observations and observation-timing checks}\label{app:pixels}
\paragraph{Setup.} We recollect A$'$ at gap 6 and $N=512$ using a separate $128^2$ camera for each environment. The trainer crops the table region, downsamples it to $64^2$, and replaces $g_o$ with a three-layer CNN. Its output is concatenated with the low-dimensional gripper opening, probe tag, and phase. Because the tag is not visible in the image, the identification channel remains instrumented. All other components, including the codes, prior, rate, metrics, $\beta=2\times10^{-3}$, and 10k training steps, remain unchanged; BFS is not used.

\paragraph{Results.} (Forecast supervision from pixels, at gaps 6 and 20, is reported in Appendix~\ref{app:generic}, Table~\ref{tab:pxgeneric}; this paragraph concerns plain training.) DIACRITIC-R recovers the exact staircase from pixels on 2 of 8 seeds. Its gap-1 rate is $1.99$ relative to $H(\Gam|\bar O)=1.99$, and its gap-2 rate is $1.00$ relative to $1.00$. Both gaps have $S_\Gam=1.00$; the place phase has $S_G=1.00$; $\hat I(C;\text{mass}|\bar O)=0.01$; and the test MSE is $0.003$. The remaining seeds fail to recover the place class, with error $0.38$--$0.53$ at the first place-divergent step. This outcome matches the discovery-limited behavior observed with state inputs at the same $N$, where 0--3 of 8 seeds succeed. The annealed scaffold reaches 0 of 8 seeds from pixels. Anticipatory memory can therefore emerge from pixel observations at exactly the predicted rate when optimization succeeds; the supervised long-horizon pixel results are reported in Appendix~\ref{app:generic}.

\paragraph{Rate predictions as benchmark diagnostics.} In the initial pixel collection, each frame was captured \emph{after} the macro step, whereas the state observation was recorded before it. The diagnostics exposed the offset: every seed had a gap-2 rate of $0$ and $S_\Gam=0$, yet predicted the first memory-dependent place action with error $0.00$. This mismatch prompted a check of the observation convention used by the theory. A single-frame probe, implemented as a CNN that predicts the hidden bits from individual frames, showed that the place class was readable at $1.00$ from ``place'' frames and at chance from gap frames, identifying a one-step capture offset. After correction, the probe predicts the place class at $0.45$, compared with a majority baseline of $0.52$, at the first place-divergent step; it predicts $\theta$ at $0.03$ ($1/32$-level) and the mass quartile at $0.29$ ($0.33$). This case illustrates the importance of direct leakage probes in visual imitation benchmarks with hidden state: the rate mismatch prompted the probe that identified the timing error, despite apparently correct behavior.

\textbf{Closed-loop evaluation from pixels.} At every step, the environment supplies the pixel policy with the corresponding camera frame at $128^2$, cropped to the table and resized to $64^2$ as during training. We evaluate 16 models, comprising 8 seeds $\times$ \{$-$R, scaffold-R\}, on 128 episodes per model and compute all metrics under the policy's own occupancy. The two seeds that are sufficient under teacher forcing remain sufficient under their own occupancy, with $S_\Gam=1.00$ and $0.95$ in gap$_2$. They carry $1.94$ bits in gap$_1$, choose the side with accuracy $1.00$, choose the slot on $0.92$/$0.94$ of episodes, and achieve success on $0.66$/$0.36$ with mean action error $3.1$/$5.7$\,cm. The fourteen insufficient models choose the slot at chance ($0.41$--$0.66$), achieve success of $0.09$--$0.42$, and have action errors of $7$--$12$\,cm. Symbolic sufficiency therefore predicts closed-loop memory use from pixel observations as it does from state observations. The primary cost of pixels is lower action precision: state policies reach $0.5$\,cm, which explains the lower success ceiling of the pixel policies. One insufficient model exhibits a $5.7$\,cm class-conditional pose separation in gap$_2$ without above-chance slot selection, indicating body memory that is written but not read.

\begin{table}[htbp]\centering\small
\caption{Closed-loop evaluation of the 16 pixel policies (128 episodes each; $S_\Gam$ under the policy's own occupancy).}
\label{tab:pxcl}
\footnotesize\setlength{\tabcolsep}{5pt}\begin{tabular}{@{}lccccc@{}}\toprule
policies & $n$ & $S_\Gam$ (gap$_2$) & slot accuracy & success & action error (cm) \\\midrule
teacher-forced sufficient ($-$R, seeds 0, 6) & 2 & $1.00$, $0.95$ & $0.92$, $0.94$ & $0.66$, $0.36$ & $3.1$, $5.7$ \\
insufficient, $-$R & 6 & $\le0.08$ & $0.42$--$0.66$ & $0.09$--$0.27$ & $8$--$12$ \\
insufficient, scaffold-R & 8 & $\le0.07$ & $0.41$--$0.54$ & $0.10$--$0.42$ & $7$--$10$ \\
\bottomrule\end{tabular}
\end{table}

\subsection{Cue--corridor T-maze}\label{app:tmaze}
We use the cue--corridor--junction T-maze from the RL memory literature \citep{ni2023transformers} to verify that exactness does not depend on an instrumented identification phase. The ordinary first observation displays a cue $c\in[R]$ together with an exogenous texture $z$. The agent then traverses $L$ identical corridor cells and reaches a junction, where the expert turns according to the cue. The task contains no instrumented channel. Assumption (A4) holds automatically because the cue affects only the expert's action, and the solver certifies (A2), (A4), and transitivity for every $L$ up to 80. Throughout the corridor, $H(\Gam_t\mid\Ob_t)=\log_2 R$, while $H(G_{E,t}\mid\Ob_t)=0$ until the junction. We use the corridor configuration: raw VQ, a $\tanh$-bounded state, $\beta=0.01$, $K=16$, and 6k steps, with 8 seeds per cell. A seed is sufficient iff $S_\Gam>0.9$ at every corridor step and its junction action is exact.

\begin{table}[htbp]\centering\small
\caption{T-maze: seeds sufficient / corridor rate of sufficient seeds (bits; theory $\log_2 R$).}
\label{tab:tmaze}
\begin{tabular}{lccccc}\toprule
cue bits, variant & $L=5$ & $L=10$ & $L=20$ & $L=40$ & $L=80$ \\\midrule
1 bit, $-$R & 8/8 / 1.25 & 8/8 / 1.19 & 7/8 / 1.00 & 3/8 / 1.00 & 5/8 / 1.00 \\
1 bit, $-$RF & 8/8 / 1.12 & 8/8 / 1.50 & 8/8 / 1.38 & 5/8 / 1.10 & 6/8 / 1.17 \\
1 bit, scaffold-RF & 8/8 / 1.17 & 8/8 / 1.00 & 8/8 / 1.06 & 8/8 / 1.19 & 6/8 / 1.08 \\
2 bits, $-$R & 6/8 / 2.00 & 4/8 / 2.00 & 1/8 / 2.00 & 0/8 / --- & 0/8 / --- \\
\bottomrule\end{tabular}
\end{table}

Every sufficient seed carries at least the predicted $\log_2 R$ bits through the corridor and $0$ after the junction. Most seeds carry exactly $\log_2 R$, while several retain an additional $0.1$--$0.5$ bit: the range is $1.00$--$1.50$ for a one-bit cue and exactly $2.00$ for a two-bit cue. Thus, the staircase in \S\ref{sec:f2} does not require a probe channel. The learning barrier remains present but is weaker and less regular than on A$'$. A single cue bit is retained reliably through $L=20$ and by $3$--$6$ of 8 seeds at $L=40$--$80$; the pooled slope is $-0.03$/step and is non-monotone in $L$. With two cue bits, the sufficient-seed count decreases monotonically as $6,4,1,0,0$ of 8 at $L=5,10,20,40,80$, with a slope of $-0.21$/step, and every sufficient seed carries exactly $2.00$ bits. In this task, the commitment cost therefore increases with both the number of retained bits and the distance, consistent with the ordering observed between A$'$ and the corridor. We present this as an observation within one task family rather than as an additional distance law, particularly because the one-bit curve is non-monotone. The scaffold again removes most of the one-bit degradation up to $L=40$.

\subsection{Physical weighing and re-observation}\label{app:weigh}
\paragraph{Task.} The A$'$ and readout tasks reveal the hidden class through a probe channel and attach the object kinematically. Here the hidden variable is the object's mass ($4$ classes, $0.1$ to $2.5$\,kg) and nothing displays it. The arm grasps and lifts the object (a physical grasp throughout), holds it for four steps, puts it back and releases it; under load the arm sags, and the tool height differs by $7.5$--$15$\,mm between adjacent classes against $2$\,mm sensor noise. After a gap of $6$ or $20$ steps during which the arm hovers, the expert grasps from the side assigned to the mass class ($2$ bits) and places the object in the slot of the coarser class; the transport sag re-reveals the mass. The symbolic observation shows the sag class at the steps where the class-conditional mean heights are pairwise separated by more than $6$\,mm (the data-driven visibility convention of Appendix~\ref{app:solver}). Each dataset has $1024$ stratified episodes (expert success $1024/1024$), and an independent recording provides the replication set.
\paragraph{Benchmark hygiene.} In the first recording a probe decoded the mass class from the gap observations with accuracy $0.885$: the loaded arm sets the object down $1$--$11$\,mm forward and slightly yawed, so the world remembered the class for the policy. We therefore return the released object to a fixture that re-seats it (held at its spawn pose $3$\,mm above the table during the last identification step; a plain pose write is dragged back by the simulator's cached friction anchors). Afterwards the probe is at chance in the gap ($0.250\pm0.010$ and $0.244\pm0.009$; chance $0.25$) and at $0.98$ during the weighing; the first recording is not used.
\paragraph{Exact requirement.} The solver certifies transitivity at every step, and (A4) fails at the single step at which the second lift begins. $H(\Gam_t\mid\Ob_t)=2$ bits from the put-down to the grasp decision and $0$ afterwards, while $H(\Gam^s_t\mid\Ob_t)$ returns to $2$ bits during the second descent: the imminent sag separates the observation supports, so histories that differ in mass are vacuously compatible. Expiration is here produced by physical re-revelation rather than by the end of use.
\paragraph{Learning.} The gate is $S_\Gam>0.9$ at every gap step and $S_G>0.9$ at the grasp decision. Plain training is sufficient on $1/8$ and $0/8$ seeds (gap 6, 20), the task-informed forecast supervised from the last identification step on $3/8$ and $2/8$, and system identification (class labels at every step) on $8/8$ and $7/8$. The observed failures occur at memory acquisition: $S_\Gam$ is constant from the put-down through the gap to the grasp, typically $0.75$ with $1.5$ bits, with two adjacent mass classes sharing a code. The measured distinction is already incomplete at the gap's start, with no further loss detected during retention. Supervising the same forecast from the first step of the episode, which also drops the reveal-time annotation from the objective, gives $8/8$ and $7/8$ ($N=448$) and $7/8$ and $8/8$ ($N=896$), with $2.02$--$2.05$ bits in the gap and identical counts on the independent recordings. The event-agnostic objective does not repair acquisition here ($2/8$, $0/8$; annealed $0/8$).
\paragraph{Closed loop.} No tested compact-carrier learner solves the task reliably in closed loop (Table~\ref{tab:weigh}), although full-history policies do. A rollout of an offline-sufficient policy illustrates one failure mode: it executes the weighing itself imprecisely (the lift reaches $100$--$165$\,mm instead of about $200$\,mm and the object is half grasped; action error $11$\,cm), so no clean sag is produced and the code is written from an observation distribution it never saw. The auxiliary objectives are associated with worse imitation precision inside the weighing block: the held-out action error there is $0.078$ for the forecast supervised from step~0 and $0.057$ for system identification, against $0.002$--$0.004$ for plain training and $0.001$ for the full-history policies. Annealing the supervision lowers it only to $0.013$--$0.031$, costs sufficiency ($4/8$ at gap~6) and leaves success at $0.08$--$0.13$. Two pre-specified controls examine these limitations.

\emph{Hybrid rollout:} the scripted expert executes the approach and the weighing while the learned policy observes, updates its own memory with the executed actions, and acts alone from the gap onward. For the offline-sufficient policies supervised from step~0 the grasp-side accuracy rises from $0.30$ to $0.99$ at gap~6 (8 seeds) and from $0.32$ to $0.93$ at gap~20 (7 seeds; prediction: above $0.8$): the code is written, kept across the gap and read correctly in closed loop once the observation is produced correctly. Task success rises only to $0.38$ and $0.13$, because the subsequent physical grasp and transport of objects of up to $2.5$\,kg remain imprecise.

\emph{Longer hold:} with an 8-step hold the analog read-out becomes easier (plain training $4/8$ instead of $1/8$; supervised from step~0 $8/8$) while closed-loop success stays low ($0.20$, $0.09$). These controls support retention and use of the measured grasp-side distinction after expert-led weighing. Producing informative observations and executing the subsequent physical manipulation remain limitations; the controls do not establish sufficiency of the learned state for complete continuous control.

\begin{table}[htbp]\centering\footnotesize\setlength{\tabcolsep}{4pt}
\caption{Weighing task, 8 seeds per cell: sufficient seeds on the training recording / on the independent recording, and closed-loop success / grasp-side accuracy over all seeds (chance $0.25$).}
\label{tab:weigh}
\begin{tabular}{@{}lcccc@{}}\toprule
 & \multicolumn{2}{c}{gap 6} & \multicolumn{2}{c}{gap 20} \\
learner & sufficient & closed loop & sufficient & closed loop \\\midrule
plain $-$R & $1/8$ / $1/8$ & $0.37$ / $0.51$ & $0/8$ / $0/8$ & $0.37$ / $0.24$ \\
task-informed forecast & $3/8$ / $3/8$ & $0.21$ / $0.51$ & $2/8$ / $2/8$ & $0.21$ / $0.35$ \\
task-informed forecast from step 0 & $8/8$ / $8/8$ & $0.10$ / $0.30$ & $7/8$ / $7/8$ & $0.07$ / $0.32$ \\
event-agnostic, not annealed & $2/8$ / --- & $0.19$ / $0.51$ & $0/8$ / --- & $0.13$ / $0.35$ \\
system identification & $8/8$ / --- & $0.08$ / $0.38$ & $7/8$ / --- & $0.06$ / $0.40$ \\
GRU bypass (continuous state) & n/a & $0.77$ / $0.90$ & n/a & $0.51$ / $0.50$ \\
full-history Transformer & n/a & $0.82$ / $0.96$ & n/a & $0.84$ / $0.97$ \\
\bottomrule\end{tabular}
\end{table}

\subsection{External Passive T-maze}\label{app:exttmaze}
We take the Passive T-maze of \citet{ni2023transformers} without modification from the MIKASA-Base suite \citep{cherepanov2025mikasa}: the goal cue appears in the first observation only, the position is ambiguous along the corridor, and the decision is taken $L$ steps later. We clone a scripted oracle (move right $L$ times, then turn according to the cue) from $256$ recorded episodes and roll the learned policy out in the external environment ($100$ episodes; success = goal reward). No rate is compared with theory; the question is whether the supervision of Appendix~\ref{app:generic} changes closed-loop success on a task we did not design. The compact carrier is the unchanged $K=16$ model ($\beta=10^{-3}$, $5000$ steps); the event-agnostic head predicts the recorded expert action at a random future offset (the task is deterministic given the cue, so no forecaster is needed and no annealing is applied). A standard GRU policy with a continuous $128$-dimensional state and no bottleneck is the reference.

\begin{table}[htbp]\centering\small
\caption{Passive T-maze: seeds (of 8) that solve the task in closed loop (success $\ge0.99$); mean success in parentheses.}
\label{tab:exttmaze}
\begin{tabular}{lcccc}\toprule
learner & $L=10$ & $L=20$ & $L=50$ & $L=100$ \\\midrule
GRU policy, continuous state & 8 (1.00) & 8 (1.00) & 8 (1.00) & 7 (0.93) \\
compact carrier, plain imitation & 1 (0.53) & 1 (0.53) & 0 (0.46) & 0 (0.46) \\
compact carrier, event-agnostic forecast & 8 (1.00) & 7 (0.94) & 5 (0.80) & 0 (0.48) \\
compact carrier, $-$RF (decoder sees future observations) & --- & 1 (0.53) & 0 (0.46) & 0 (0.46) \\
compact carrier, annealed continuous scaffold & --- & 0 (0.46) & 0 (0.47) & 0 (0.47) \\
GRU policy, 8-dimensional state & --- & 0 (0.47) & 0 (0.41) & 0 (0.35) \\
GRU policy, 4-dimensional state & --- & 0 (0.42) & 0 (0.31) & 0 (0.18) \\
\bottomrule\end{tabular}
\end{table}

The $128$-dimensional GRU is reliable through $L=50$, whereas the compact carrier benefits from event-agnostic supervision but still fails at $L=100$. Doubling the compact learner's training budget does not improve its $L=50$ result ($3/8$): runs either capture the cue early and retain it to the junction or collapse to a cue-free code. Neither the future-conditioned $-$RF decoder nor the annealed scaffold extends the horizon in these tested configurations.

The GRUs with $4$- or $8$-dimensional states fail at all three tested lengths, $L=20,50,100$. These implementations shrink the input network and action head together with the recurrent state, and their information capacity is not matched to the discrete carrier. They show that training difficulties also occur in small continuous RNNs; they do not isolate the effects of state size, network size, optimization, or quantization.

The event-agnostic objective need not identify the cue when candidate objects are rearranged after the delay: a future action can then be uninformative about the cue without the future observation. A future-observation-conditioned decoder, such as $-$RF, can supply that side information. We make no architectural claim against the unconstrained GRU.

\FloatBarrier
\subsection{Certified requirements on community benchmarks}\label{app:cert}
We apply the protocol to two community benchmarks using their environment code unmodified from MIKASA-Base \citep{cherepanov2025mikasa}. In the bsuite memory chain (``MemoryLength''), a $b$-bit context is displayed initially, then hidden until a query index appears; the final action repeats the queried bit. The second task is the Passive T-maze of Appendix~\ref{app:exttmaze}.

\paragraph{Certification and timing.} We record the external environments' own observations and oracle actions, retain a trajectory for every hidden value ($2^b b$ context--query pairs; two T-maze goals), and run the solver on the induced finite POMDP. It certifies (A2), (A4), and transitivity on every instance. With uniform latent weights, the memory-chain requirement is zero in the first two steps, while the cue is visible. It becomes $b$ bits at the first cue-free step (step 3), remains $b$ until the query, and falls to one bit once the query index is observed. Only the queried bit then matters. The requirement is independent of \texttt{memory\_length}; $H(G_{E,t}\mid O_t)=0$ until the final decision. The T-maze requires one bit during its cue-free corridor.

\paragraph{Learning and rate convention.} We clone the oracle with the unchanged $K=16$ sole carrier ($\beta=10^{-3}$, 5000 steps, 512 demonstrations), using plain imitation or the event-agnostic objective (a recorded future action at a random offset; no annealing). A seed is sufficient when $S_\Gam>0.9$ at every positive-requirement step and closed-loop success is at least $0.99$ in the external environment (200 episodes). The pre-specified prediction concerns \emph{mid-delay}: sufficient codes should match the requirement there within $0.1$ bit. Table~\ref{tab:cert} uses the empirical demonstration occupancy for both entropies, giving $1.99$ and $2.99$ bits instead of the uniform-law values $2$ and $3$.

\begin{figure}[htbp]\centering\includegraphics[width=\linewidth]{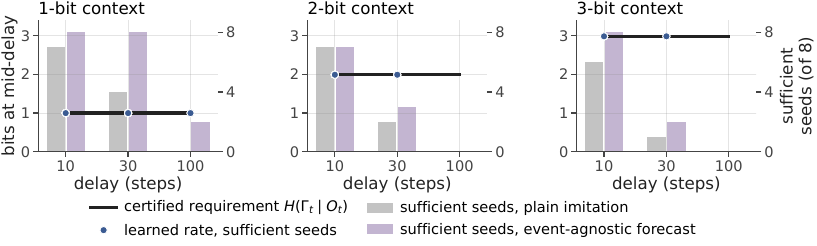}
\caption{\textbf{bsuite memory chain: mid-delay rates and learning reliability.} At mid-delay the certified requirement (line) is independent of delay, and every sufficient seed (dots; both learners) matches it. Bars show the number of sufficient seeds (right axis). This comparison concerns one point in the waiting period; it does not establish minimal rates after the query appears.}
\label{fig:cert}
\end{figure}

\begin{table}[htbp]\centering\footnotesize\setlength{\tabcolsep}{5pt}
\caption{Certified requirement and learned rate on community benchmarks (8 seeds per cell). The certified value is the empirical requirement at mid-delay under the recorded occupancy ($512$ episodes, hence $1.99$ and $2.99$); the learned rate is $\Hc(C_t\mid O_t)$ at the same step over the sufficient seeds of both learners.}
\label{tab:cert}
\begin{tabular}{@{}lccccc@{}}\toprule
task & delay & certified & suff., plain & suff., event-agnostic & learned rate \\\midrule
MemoryChain, 1 bit & 10 & $1.00$ & $7/8$ & $8/8$ & $1.00$ \\
MemoryChain, 1 bit & 30 & $1.00$ & $4/8$ & $8/8$ & $1.00$ \\
MemoryChain, 1 bit & 100 & $1.00$ & $0/8$ & $2/8$ & $1.00$ \\
MemoryChain, 2 bits & 10 & $1.99$ & $7/8$ & $7/8$ & $1.99$ \\
MemoryChain, 2 bits & 30 & $1.99$ & $2/8$ & $3/8$ & $1.99$ \\
MemoryChain, 2 bits & 100 & $1.99$ & $0/8$ & $0/8$ & --- \\
MemoryChain, 3 bits & 10 & $2.99$ & $6/8$ & $8/8$ & $2.99$ \\
MemoryChain, 3 bits & 30 & $2.99$ & $1/8$ & $2/8$ & $2.99$ \\
MemoryChain, 3 bits & 100 & $2.99$ & $0/8$ & $0/8$ & --- \\\midrule
Passive T-maze & 10 & $1.00$ & $2/8$ & $8/8$ & $1.00$ \\
Passive T-maze & 20 & $1.00$ & $0/8$ & $7/8$ & $1.00$ \\
Passive T-maze & 50 & $1.00$ & $0/8$ & $5/8$ & $1.00$ \\
\bottomrule\end{tabular}
\end{table}

\FloatBarrier
All $87$ sufficient runs (of $192$) match the certified rate at mid-delay (largest measured deviation $0.00$ bit). This agreement is strongly constrained by the task structure. For a fixed trained model, the mid-delay code is a deterministic function of the context; no other episode-dependent distinction has been observed, and the preceding oracle actions are fixed. Under the uniform law, $H(C_t\mid O_t)\le b$, while exact sufficiency requires $H(C_t\mid O_t)\ge b$. Under empirical occupancy the same argument uses the empirical context entropy. These benchmarks test transfer of the protocol and acquisition of the required information; Task~A and readout provide the comparisons with additional world information.

\paragraph{Acquisition without immediate forgetting.} Mid-delay agreement does not extend to the query step. Once the query index appears, the requirement is about one bit under empirical occupancy. The $29$ sufficient one-bit memory-chain runs carry $0.999$ bit there. The $19$ sufficient two-bit runs retain $1.514$--$1.990$ bits (mean $1.91$), and the $17$ sufficient three-bit runs retain $2.385$--$2.976$ bits (mean $2.69$). Thus the models can acquire all required context while retaining distinctions that the query has made unnecessary. Like the surplus under un-annealed prediction on A$'$ (\S\ref{sec:f4}), this separates acquisition from subsequent compression; it does not identify a shared optimization mechanism.

Learning reliability also decreases with delay (\S\ref{sec:f4}): no seed passes for two or three bits at delay $100$, and the event-agnostic objective is at least as reliable as plain imitation in every tested cell. The scope remains finite symbolic tasks with enumerable hidden variables and the environment's own observations.

\FloatBarrier
\needspace{8\baselineskip}\section{Representation targets of related work}\label{app:targets}
\paragraph{Scope of the empirical comparison.} We compare representation \emph{targets} under one solver rather than re-implementing the methods that pursue them. Multi-step inverse kinematics and agent-centric state discovery \citep{mhammedi2023representation,lamb2023guaranteed} target control-endogenous state, and decision-centric memory compression \citep{zou2026remember,yamin2026whatmust} targets history distinctions that change a near-optimal decision; neither is defined for a fixed demonstrator with future observations as decoder side information. Their objectives must therefore be re-specified before their rates can be compared with ours, while their exploration, reachability, or language-model mechanisms would introduce additional implementation differences. We implement the fixed-demonstrator forms of the theoretically distinct targets and measure them with the same solver (last column); targets not implemented are marked accordingly.

\begin{table}[htbp]\centering\footnotesize\setlength{\tabcolsep}{2.5pt}\renewcommand{\arraystretch}{1.1}
\caption{Representation targets in related work (\S\ref{sec:related}). Property columns: \textbf{F} the target reproduces a fixed demonstrator's action distribution; \textbf{S} continuations are compared over stochastic reachable supports, so future observations act as decoder side information; \textbf{E} minimality is defined through occupancy-weighted conditional entropy; \textbf{B} the general case yields a non-transitive compatibility bracket. Individual properties appear in prior work; the contribution here is their combination.}\label{tab:targets}
\newcolumntype{C}[1]{>{\centering\arraybackslash}p{#1}}\begin{tabularx}{\linewidth}{@{}R{1.12in}R{1.02in}C{0.42in}C{0.5in}C{0.62in}C{0.22in}>{\raggedright\arraybackslash}X@{}}\toprule
method & target representation & F & S & E & B & rate of the target's fixed-demonstrator form under our solver \\\midrule
bisimulation \citep{zhang2021learning}; $\pi$-bisimulation \citep{castro2020scalable} & value / dynamics equivalence (policy-specific for $\pi$-bisim.) & policy & -- & -- & -- & not implemented \\
$\pi^*$-irrelevance \citep{li2006towards}; support sufficiency \citep{walsh2026supportsufficiency} & optimal-action equivalence (single decision) & optimal act. & -- & rate--regret & -- & single-decision quotient $=G_{E,t}$: 0 bit in the gaps (Fig.~\ref{fig:fig3}) \\
causal states / PSR \citep{shalizi2001computational,littman2002predictive} & sufficiency for predicting observations & -- & -- & stat.\ complexity & -- & observation-predictive form: $1.75$ bit on the re-reveal toy vs.\ $1.00$ for $\Gam$ (App.~\ref{app:e5}) \\
multi-step inverse / ACSD \citep{mhammedi2023representation,lamb2023guaranteed} & control-endogenous latent & -- & -- & -- & -- & $0.3$--$0.7$ bit of drift on the corridor, never sufficient (App.~\ref{app:grid}) \\
approximate information state \citep{subramanian2022ais} & sufficiency for reward and next observation & -- & -- & -- & -- & not implemented \\
stable quotients \citep{zhang2026minimal}; Nerode quotient \citep{nixon2026myhillnerode} & minimal Markov / controller-equivalent state & -- & -- & -- & -- & not implemented \\
incompletely specified machines \citep{paull1959minimizing,pfleeger1973state} & minimal closed cover of a specification & -- & don't-cares & state count & \checkmark & state-count objective; see App.~\ref{app:cover} \\
decision-centric agent memory \citep{zou2026remember,yamin2026whatmust} & history distinctions that change a near-optimal decision & -- & -- & rate--distortion & -- & open-loop future-action form: $1.61$--$1.86$ bit on Task~A, $2.06$ vs.\ $1.00$ on the re-reveal toy (App.~\ref{app:e5}) \\
RMA / system identification \citep{kumar2021rma,liang2024rma} & hidden parameter $\theta$ & -- & -- & -- & -- & $2.25$--$8.71$ bit, slope $0.88$ in $\log_2 M$ (Fig.~\ref{fig:taskA}) \\ \midrule
\textbf{this paper} & \textbf{minimal recurrent behavioral memory} & \checkmark & \checkmark & \checkmark & \checkmark & \textbf{A$'$: $2.00\to1.00$; Task~A: $\le0.10$ vs.\ 0 bit; readout: $2.00$ vs.\ 2 bits} \\
\bottomrule\end{tabularx}
\end{table}

\end{document}